\documentclass{article}
\usepackage{textcomp}
\usepackage{stfloats}
\usepackage{url}
\usepackage{verbatim}
\usepackage{graphicx}
\usepackage{cite}

\usepackage[in]{fullpage}  %
\usepackage{natbib} 
\usepackage{mathpazo} 

\usepackage[utf8]{inputenc} %
\usepackage[T1]{fontenc}    %
\usepackage{hyperref}       %
\usepackage{url}            %
\usepackage{booktabs}       %
\usepackage{amsfonts}       %
\usepackage{nicefrac}       %
\usepackage{microtype}      %
\usepackage{lipsum}         %
\usepackage{todonotes}

\usepackage{tcolorbox}

\usepackage{subfigure}

\AtBeginDocument{\setlength\abovedisplayskip{5pt}}
\AtBeginDocument{\setlength\belowdisplayskip{5pt}}

\hypersetup{
    colorlinks=true,
    citecolor=blue,
    linkcolor=blue,
}

\usepackage{listings} %

\definecolor{codegreen}{rgb}{0,0.6,0}
\definecolor{codegray}{rgb}{0.5,0.5,0.5}
\definecolor{codepurple}{rgb}{0.58,0,0.82}
\definecolor{codeblue}{rgb}{0,0,1}
\definecolor{backcolour}{rgb}{0.95,0.95,0.92}
\definecolor{key-color}{rgb}{0.8, 0.47, 0.196}

\lstdefinestyle{mystyle}{
    backgroundcolor=\color{backcolour},   
    commentstyle=\color{codegreen},
    numberstyle=\tiny\color{codegray},
    stringstyle=\color{codepurple},
    basicstyle=\ttfamily\footnotesize,
    breakatwhitespace=false,         
    breaklines=true,                 
    captionpos=b,                    
    keepspaces=true,                 
    numbers=left,                    
    numbersep=5pt,                  
    showspaces=false,                
    showstringspaces=false,
    showtabs=false,                  
    tabsize=2,
    language=Python,
    emph={lm},
    emphstyle={\color{blue}},
    classoffset=1, %
    otherkeywords={sum},
    morekeywords={rm, mean},
    keywordstyle=\color{codegreen},
    classoffset=0,
}
\usepackage{amsmath,amsthm,amssymb,natbib,graphicx,url} %
\newcommand{\diag}{\operatorname{diag}}
\newcommand{\tX}{\tilde{X}}
\newcommand{\tx}{\tilde{x}}
\newcommand{\tL}{\tilde{\Lambda}}
\newcommand{\R}{\mathbb{R}}  

\newcommand{\N}{\mathbb{N}}

\newcommand{\E}{\mathbf{E}}

\newcommand{\mb}{\mathbb}

\newcommand{\ga}{\gamma}
\newcommand\mf{\mathbf}
\newcommand\ml{\mathcal}
\newcommand\tp{\top}
\newcommand\pl{\partial}
\newcommand{\wt}{\widetilde}
\newcommand{\rw}{\rightarrow}

\renewcommand{\l}{\left} 
\renewcommand{\r}{\right} 
\newtheorem{thm}{Theorem}  

\newtheorem{lem}{Lemma} 
\newtheorem{prop}{Proposition}

\newtheorem{assum}{Assumption}
\newcommand{\be}{\begin{equation}}
\newcommand{\ee}{\end{equation}}
\newcommand{\bea}{\begin{equation}\begin{aligned}}
\newcommand{\eea}{\end{aligned}\end{equation}}

\usepackage{titletoc}

\usepackage[utf8]{inputenc} %
\usepackage[T1]{fontenc}    %
\usepackage{hyperref}       %
\usepackage{url}            %
\usepackage{booktabs}       %
\usepackage{amsfonts}       %
\usepackage{nicefrac}       %
\usepackage{microtype}      %
\usepackage{xcolor}         %

\usepackage{framed}
\colorlet{shadecolor}{orange!15}

\usepackage{lipsum}         %
\usepackage{tcolorbox}
\usepackage{listings} %
\usepackage{algorithm}
\usepackage{algorithmic}
\usepackage{enumitem}
\usepackage{wrapfig}

\definecolor{red}{rgb}{1, 0, 0}

\definecolor{green}{rgb}{0, 1, 0}

\definecolor{blue}{rgb}{0, 0, 1}
\newcommand{\BLUE}[1]{{\color{blue} #1}}

\definecolor{orange}{rgb}{1, 0.4, 0.0}

\newcommand{\yushunrevise}[1]{\textcolor{black}{ #1}}

\newcommand{\zhaoruirevise}[1]{\textcolor{black}{ #1}}

\usepackage{wasysym}
\usepackage{tabulary}

\newcommand\nnfootnote[1]{%
  \begin{NoHyper}
  \renewcommand\thefootnote{}\footnote{#1}%
  \addtocounter{footnote}{-1}%
  \end{NoHyper}
}

\definecolor{ccr}{RGB}{10,110,150}  
\usepackage{hyperref}
\hypersetup{hypertex=true,
    colorlinks=true,
    linkcolor=ccr,
    anchorcolor=ccr,
    citecolor=ccr}

\usepackage{listings}
\usepackage{xcolor}
\definecolor{codegreen}{rgb}{0,0.6,0}
\definecolor{codegray}{rgb}{0.5,0.5,0.5}
\definecolor{codepurple}{rgb}{0.58,0,0.82}
\definecolor{backcolour}{rgb}{0.95,0.95,0.92}
\lstdefinestyle{mystyle}{
    backgroundcolor=\color{backcolour},   
    commentstyle=\color{codegreen},
    keywordstyle=\color{magenta},
    numberstyle=\tiny\color{codegray},
    stringstyle=\color{codepurple},
    basicstyle=\ttfamily\footnotesize,
    breakatwhitespace=false,         
    breaklines=true,                 
    captionpos=b,                    
    keepspaces=true,                 
    numbers=left,                    
    numbersep=5pt,                  
    showspaces=false,                
    showstringspaces=false,
    showtabs=false,                  
    tabsize=2,
}

\definecolor{bgcolor}{rgb}{0.1, 0.1, 0.1}
\definecolor{commentcolor}{rgb}{0.5, 0.5, 0.5}
\definecolor{keywordcolor}{rgb}{0.5, 0.0, 0.5}
\definecolor{stringcolor}{rgb}{0.0, 0.5, 0.0}
\definecolor{numbercolor}{rgb}{0.0, 0.0, 0.7}
\definecolor{functioncolor}{rgb}{0.0, 0.0, 0.5}

\lstdefinestyle{github}{
    backgroundcolor=\color{bgcolor},
    basicstyle=\ttfamily\small\color{white},
    commentstyle=\color{commentcolor},
    keywordstyle=\color{keywordcolor}\bfseries,
    stringstyle=\color{stringcolor},
    numberstyle=\color{numbercolor},
    identifierstyle=\color{white},
    showstringspaces=false,
    numbers=left,
    numbersep=5pt,
    tabsize=4,
    breaklines=true,
}

\lstnewenvironment{python}[1][]
{
    
    \lstset{style=mystyle,#1}
}{}

\begin{document}

\title{\fontsize{16.95pt}{20pt}\selectfont Towards Quantifying the Hessian Structure of Neural Networks}

\author{ Zhaorui Dong$^{*1}$, Yushun Zhang$^{* 12}$, Jianfeng Yao$^{\dagger 1}$, Ruoyu Sun$^{\dagger 12}$ \\
 \text{ }\\
   $^1$ The Chinese University of Hong Kong, Shenzhen, China \\
  $^2$ Shenzhen Research Institute of Big Data 
}

\maketitle
\nnfootnote{$*$: Equal contribution. These authors are listed in alphabetical order. }
\nnfootnote{$\dagger$:Correspondence authors: Ruoyu Sun, Email:  \texttt{sunruoyu@cuhk.edu.cn}; Jianfeng Yao, Email: \texttt{jeffyao@cuhk.edu.cn}.}
\nnfootnote{This work has been submitted to the IEEE for possible publication. Copyright may be transferred without notice, after which this version may no longer be accessible.}

\begin{abstract}

 Empirical studies reported that the Hessian matrix of neural networks (NNs) exhibits a near-block-diagonal structure, yet its theoretical foundation remains unclear. In this work, we reveal that the reported Hessian structure comes from a mixture of two forces: a ``static force'' rooted in the architecture design, and a ``dynamic force'' arisen from training. We then provide a rigorous theoretical analysis of ``static force'' at random initialization. We study linear models and 1-hidden-layer networks for classification tasks with $C$ classes.  By leveraging random matrix theory, we compare the limit distributions of the diagonal and off-diagonal Hessian blocks and find that the block-diagonal structure arises as $C$ becomes large. Our findings reveal that $C$ is one primary driver of the near-block-diagonal structure. These results may shed new light on the Hessian structure of large language models (LLMs), which  typically operate with a large $C$ exceeding $10^4$. 
 \footnote{Our code is available at \url{https://github.com/zyushun/Hessian-structure}.} 
\end{abstract}

\section{Introduction}

The Hessian matrix of neural networks (NNs) is crucial for understanding training dynamics, as well as motivating better algorithm designs.
A classical work \citep{collobert2004large} empirically reported that the Hessian of NNs is highly structured: the Hessian is observed to be {\it near-block-diagonal}.
We reproduce this result in Figure \ref{fig:hessian}. 
Unfortunately, no rigorous theory has been established in the past two decades to explain this phenomenon.  

Very recently, the near-block diagonal structure of Hessian has drawn renewed attention in the machine learning community as it helps understanding the training of large language models (LLMs)
\citep{zhang2024transformers,zhang2024adam,kunstner2024heavy}.
Again, these works primarily focus on empirical observations, 
and there is no rigorous theoretical results on the underlying source of the special structure. 
The following fundamental question remains largely open:

\begin{center}
{\it \yushunrevise{When and} why does the Hessian of NNs exhibit near-block-diagonal structure? }
\end{center}

Before delving into this question, we first list some of its important implications.

\begin{itemize}
[topsep=1pt,parsep=1pt,partopsep=1pt, leftmargin=*]
    \item {\bf I.} Understanding Hessian structure can help understand NN training. For instance, the effectiveness of diagonal preconditioned methods such as Adam \citep{kingma2014adam} is usually strongly related to the Hessian structure; see, e.g.,  \citet{sun2021worst,qu2022optimal,das2024towards}. Recently, the near-block-diagonal Hessian is observed along the training process and such structure is shown to 
 be related to the effectiveness of Adam on LLMs \citep{zhang2024transformers,kunstner2024heavy}. 
 
 Besides Adam, the near-block-diagonal Hessian structure may play a crucial role in the effectiveness of block-diagonal preconditioned methods (e.g.,  \citep{martens2015optimizing,gupta2018shampoo,goldfarb2020practical,vyas2024soap,jordan2024muon}). Among these methods,   Muon optimizer \citep{jordan2024muon} is used for training Moonlight \citep{liu2025muon}, Kimi-K2 \citep{moonshot2025k2}, and  GLM-4.5 \citep{zeng2025glm}  very recently. 
    
    \item {\bf II.} Understanding Hessian structure can help design new training methods for NNs. For instance, Adam-mini \citep{zhang2024adam}, a recently proposed optimizer,  utilizes the near-block-diagonal Hessian structure to cut down 50\% memory consumption in Adam. We believe the special Hessian structure can inspire more new optimizers.

\item {\bf III. } 
The near-block-diagonal Hessian structure can offer a new class of problems for the optimization community
to study. 
For the optimization community, it is rare to analyze (near-) block-diagonal Hessian structure since typical problems do {\it not} have such structure. For instance, in the classical non-linear programming dataset \citep{lavezzi2022nonlinear}, all problems have non-block-diagonal Hessians. 
Understanding the special Hessian structure of NN can draw attention from the optimization community, motivating further study into this specialized class of problems. 
\end{itemize}

\begin{figure}[t]
 \vspace{-0.3cm} 
    \centering
    \subfigure[\small 1-hidden-layer network with MSE loss]{\includegraphics[width=0.49\textwidth]{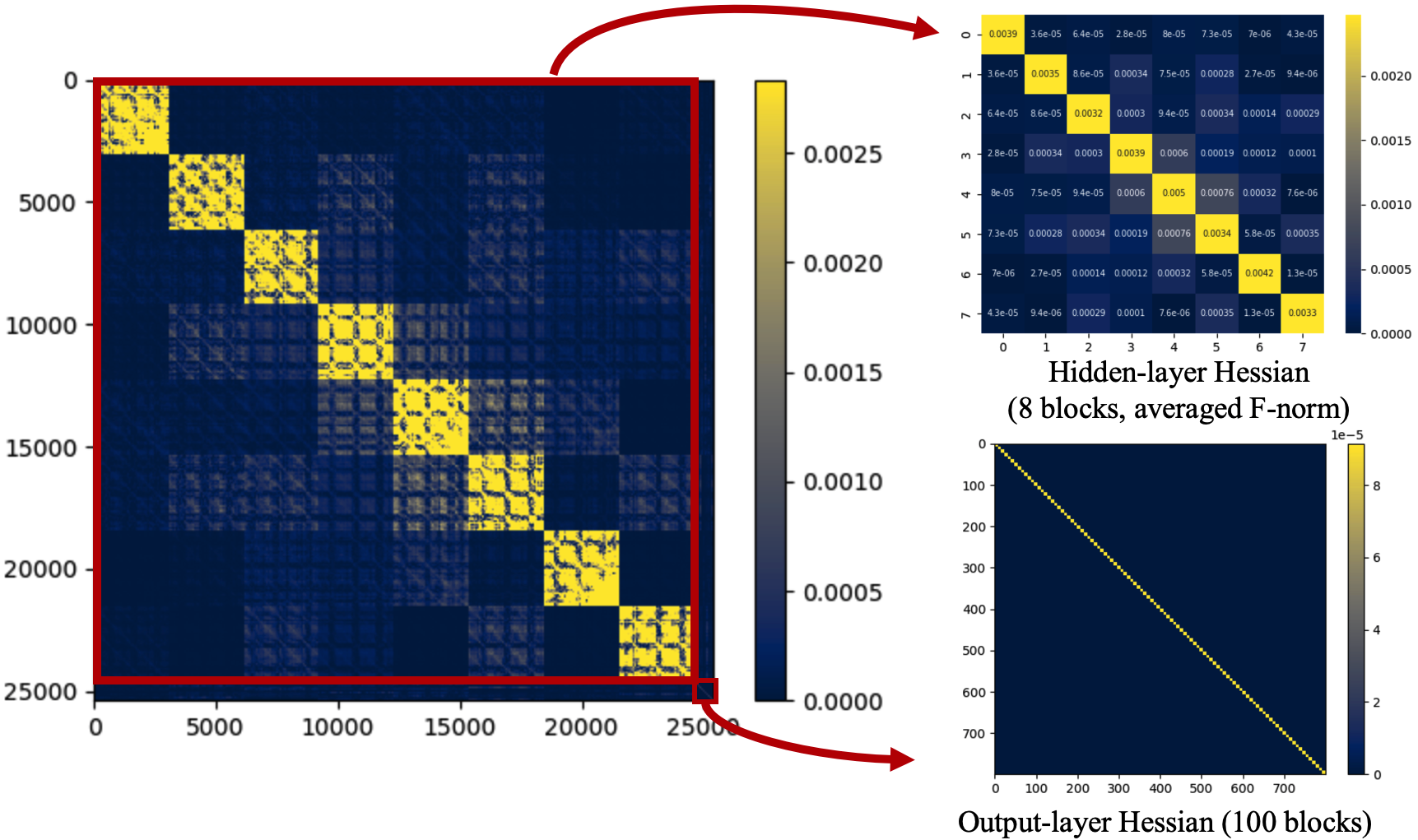}}
    \subfigure[\small 1-hidden-layer network with CE loss]{\includegraphics[width=0.49\textwidth]{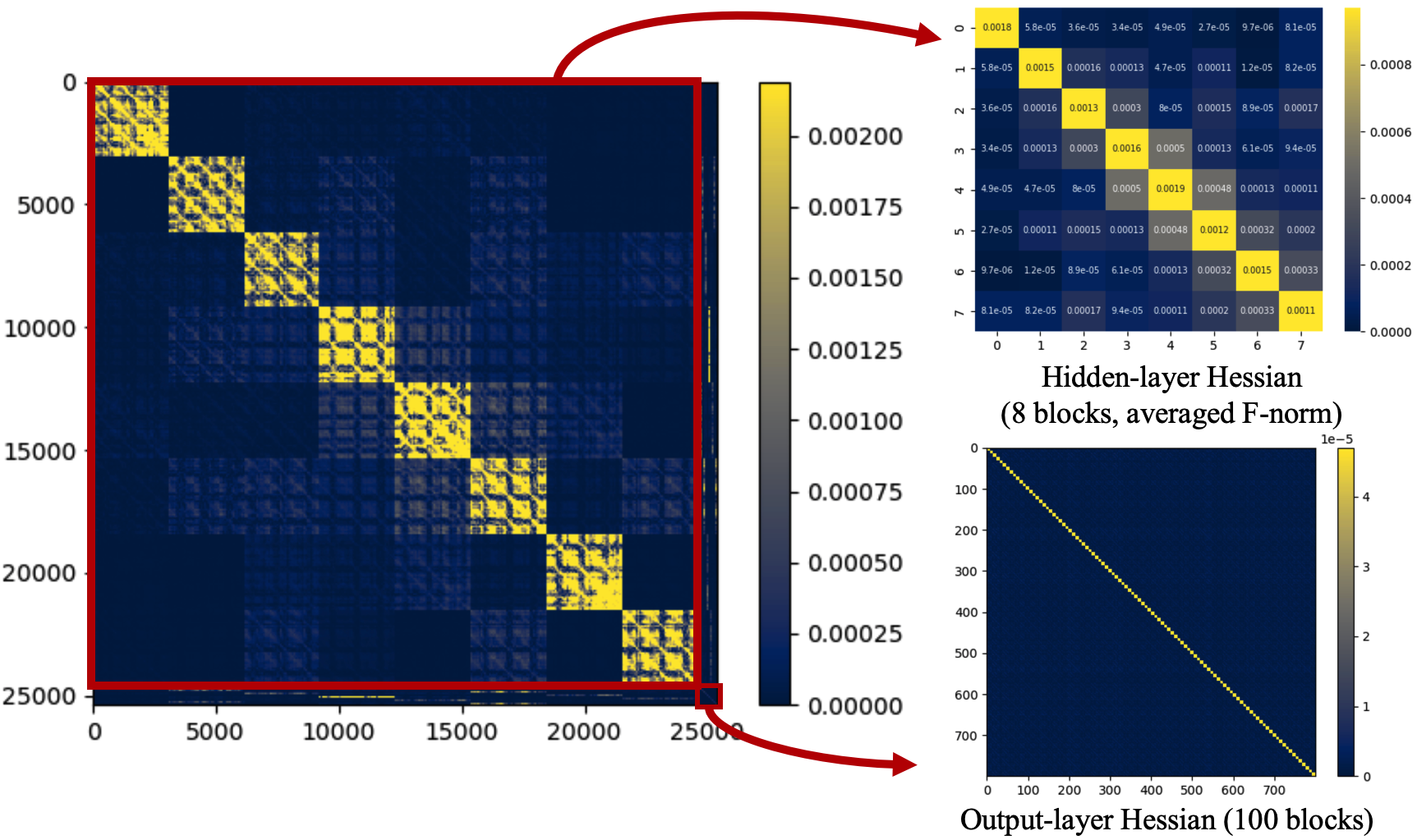}}
    \caption{\small {\bf(a, b):}  The Hessian matrix of a  1-hidden-layer network with $8$ hidden neurons at random initialization on CIFAR-100 dataset (\# hidden neuron $m = 8$ and \# classes or output neuron $C = 100$). For clearer visualization, we report the absolute value of each Hessian entry, and this applies to all Hessian matrices reported in this work. We observe near-block-diagonal structures under both MSE and CE loss with $m + C = 108$ blocks in total.}
  \label{fig:hessian}
\vspace{-0.3cm}
\end{figure}

In this work, we explore the Hessian structure of NNs both numerically and theoretically. First, we report more fine-grained numerical findings on Hessian: we observe {\it ``block-circulant-block-diagonal''} structure at the random initialization and {\it ``block-diagonal''} structure after training starts (presented later in Section \ref{sec_closer_look}).  In particular, the ``dynamic force'' compresses cross-layer Hessian components  during training; and the ``static force'' compresses the cross-neuron component in each layer for both initialization stage and training stage. Our findings suggest that the previously reported block-diagonal structure actually comes from a mixture of two  forces: a ``static force'' rooted in the architecture design, and a ``dynamic force'' arisen from training.

Then, we provide a rigorous theoretical analysis of ``static force'' at random initialization. We focus on linear models and 1-hidden-layer networks for standard classification tasks with $C$ classes.
Leveraging tools from random matrix theory, we characterize the limiting behavior of diagonal and off-diagonal Hessian blocks as the sample size $N$ and input dimension $d$ grow proportionally to infinity.  Our theory shows that the off-diagonal blocks will be pushed to 0 as the number of classes $C$ increases, suggesting that 
$C$  is a  primary driver of the near-block-diagonal Hessian structure. 
Our theory may shed new light on the Hessian structures of  LLMs since they usually have large $C$ (more than $10^4$ or $10^5$) \footnote{$C =32k$  in Llama 2 \citep{touvron2023llama} and $C =128k$ in many recent models such as DeepSeek-V3 \citep{liu2024deepseek}.}.

Our findings challenge the conventional wisdom of
\citep{collobert2004large}, which argues that Cross-Entropy (CE) loss is the primary determining factor. Our findings suggest that CE loss is not crucial.
Later, we will analyze why \citep{collobert2004large} misattributed the emergence of the near-block diagonal structure to CE loss.

Our main contributions are summarized as follows.

\begin{itemize}
[topsep=1pt,parsep=1pt,partopsep=1pt, leftmargin=*]
    \item \yushunrevise{We numerically investigate the source of the near-block-diagonal Hessian structure. We reveal two forces that shape such structure:  a ``static force'' rooted in the architecture, and a ``dynamic force'' arisen from training. In particular, the ``dynamic force'' compresses cross-layer Hessian components  during training; and the ``static force'' compresses the cross-neuron component in each layer for both initialization stage and training stage.}
    \item We provide rigorous theory on the Hessian of linear models at random initialization. 
 As the sample size $N$ and input dimension $d$ grow proportionally to infinity, we calculate the limit of Frobenius norm of the diagonal and off-diagonal blocks of the Hessian. Specifically, the diagonal blocks correspond to the Hessian of 
 weights associated with the same class, 
 while the off-diagonal blocks represent the Hessian of weights from different classes.   
 We find that: the ratio between the off-diagonal and diagonal blocks decays to zero at the rate of $O(1/C)$, where $C$ is the number of classes. This demonstrates that the Hessian becomes block-diagonal as $C \rightarrow \infty$.
    \item We extend the above analysis to 1-hidden-layer networks. We primarily focus on two sub-matrices in Hessian:  the hidden-layer Hessian and the output-layer Hessian, which are highlighted with red boxes in Figure \ref{fig:hessian}. For the hidden-layer Hessian, the ratios between their off-diagonal and diagonal blocks decay to zero at the rate of $O(1/\sqrt{C})$. For the output-layer Hessian, the decay rate is $O(1/C)$. This demonstrates that these sub-matrices will become block-diagonal as $C \rightarrow \infty$. In this case, the total number of blocks in these sub-matrices equals  $(m + C)$, where $m$ denotes the number of hidden neurons.

   \item We highlight some key technical contributions in our proof. The major challenge lies in characterizing limiting eigenvalue distribution of product of {\it dependent} random matrices, which is a non-standard problem in classical random matrix theory. 
   For the Hessian of NNs, we find that such dependency arises from ReLU activation and CE loss, and diminishes as $d \rightarrow \infty$. Subsequently, inspired by \citet{pastur2020random}, we use the {\it Lindeberg interpolation principle} to address such dependency.

\end{itemize}

\paragraph{Notations.}  For a  matrix $X \in \mathbb{R}^{m\times n}$, $X^\top$ denotes the transpose of $X$, $\|X\|_{\operatorname{F}}$ the Frobenius norm of $X$, $I_{n\times n}$ and $0_{n \times n}$ the identity matrix and the zero matrix of size $n \times n$, respectively. We denote $[n]$ as the index set $\{1,\cdots, n\}$. We say $ x\overset{d}{=} y$ if the random variables (r.v.) $x$ and $y$ share a same distribution. We denote the Dirac measure at $x$ by $\delta_x$, the support of measure $\mu$ by $\operatorname{supp}(\mu)$  and the expectation of $x$  by $\E[x]$ .
 We use $\mathcal{N}(\mu, \sigma^2)$ to denote Gaussian distribution with mean $\mu$ and variance $\sigma^2$. 
 We use $\Im(z)$ to denote the image part of a complex number $z \in \mb{C}$, and that $\mb{C}^+ = \{z \in \mb{C}| \Im(z) >0\}$.  In this paper, we will intermittently employ the notations $H_{ww}$, $H_{vv}$, and $H_{wv}$ to denote the hidden-layer, output-layer, and cross-layer Hessian, respectively, of a 1-hidden-layer network.

\section{Related works}
\label{appendix_related_works}
\paragraph{Hessian spectrum analysis} Most studies on Hessian of NNs focus on Hessian eigenvalue distribution, a.k.a., the spectrum. \citet{lecun2002efficient,dauphin2014identifying,sagun2016eigenvalues,sagun2017empirical, chaudhari2019entropy,ghorbani2019investigation, granziol2019towards,yao2020pyhessian} reported that the Hessian spectra of NNs consist of a ``bulk" together with a few ``outliers". \citet{pennington2017geometry,fort2019emergent,papyan2020traces,wu2020dissecting,liao2021hessian,singh2021analytic}  studied the shape of the Hessian  spectrum and Hessian rank in theory.  \citet{papyan2018full, papyan2019measurements,sankar2021deeper} numerically studied the relation between the spectrum of Hessian and that of Gauss-Newton matrix.   \citet{keskar2016large,yao2018hessian,zhang2019algorithmic,granziol2022learning} studied the connection between the Hessian spectrum of NNs and some training phenomena such as the effect of batch sizes.   \citet{ghorbani2019investigation, yao2020pyhessian} explained the effectiveness of training techniques such as BatchNorm via the shape of Hessian spectrum. 
\citet{zhang2024transformers} numerically studied the blockwise Hessian spectrum of CNNs and Transformers. They further connect the blockwise spectra to the effectiveness of Adam.  
 Another line of works studied the interplay between  Hessian extreme eigenvalues and  the trajectories of gradient methods (e.g., \citep{wu2017towards,draxler2018essentially,gur2018gradient,jastrzkebski2018relation,chaudhari2019entropy,jiang2019fantastic,he2019asymmetric,alain2019negative,wei2019noise,li2020hessian,jastrzebski2020break,cohen2021gradient,cohen2022adaptive,arora2022understanding,wang2022analyzing,lyu2022understanding,park2022vision}).

 Different from all these works, we study the macroscopic structure of the Hessian rather than its spectrum. Note that these two topics are rather orthogonal: it is possible to change the matrix structure without changing its eigenvalues, and vice versa.
 Specifically, we focus on the ratio between diagonal Hessian blocks and off-diagonal ones, 
 which is not covered in the spectrum analysis. 

\paragraph{Hessian structure analysis} 
 \citet{collobert2004large} empirically observed the following phenomenon: 
when using a neural network to solve a binary classification problem under CE loss, the Hessian is near-block-diagonal. 
They also reported that the near-block-diagonal structure disappears when changing to Mean-Square (MSE) loss. 
\citet{collobert2004large} thereby conjectured that the near-block-diagonal Hessian stems from CE loss, and they provided an one-line informal explanation (re-stated later in Section~\ref{sec_collobert}). The near-block-diagonal structure was also reported recently under CE loss for various models including linear models  \citep{kunstner2024heavy}, 1-hidden-layer network \citep{zhang2024transformers}, and 1-hidden-layer Transformers \citep{zhang2024adam}. Similar Hessian structure is later numerically reported on more practical models including  GPT-2 \citep{maes2024understanding}, and OPT-125M \citep{malinovskii2024pushing}. 
These results show that the near-block-diagonal structure
appeared in a wide range of architectures. 
We point out that these works primarily focus on empirical observations, and the rigorous theoretical analysis is still missing.

Very recently, \citet{ormaniec2024does} employed matrix calculus to derive the  Hessian expression of a 1-hidden-layer Transformer to understand the difficulties in training Transformers. It is valuable and non-trivial to derive the Hessian expression of Transformers due to their complicated design. However, the subsequent analysis of Hessian structure is relatively simplified: e.g.,  they view the weights and data as constant matrices and did not incorporate their random distributions.  Consequently, the exact behavior of each Hessian block has not been characterized yet, and the origin of the near-block-diagonal structure remains unexplored.

Different from the aforementioned work, we establish the first rigorous theory on the Hessian structure of linear and 1-hidden-layer network via random matrix theory.  Our theory reveals that the number of classes $C$ is one major cause of the near-block-diagonal or block-circulant-block-diagonal structure.

\paragraph{Algorithm design} Multiple algorithm designs are proposed by approximating Hessian (or other curvature matrices) by block-diagonal matrices (e.g., \citep{roux2007topmoumoute,martens2015optimizing,desjardins2015natural,zhang2017block,gupta2018shampoo,george2018fast,goldfarb2020practical,dangel2020modular,vyas2024soap,jordan2024muon,an2025asgo}). Our theory can explain why these methods work. The special Hessian structure also has strong connections to diagonal preconditioned methods (e.g., \citep{kingma2014adam,liu2023sophia}).

\section{Empirical Observations and Existing Wisdom}
\subsection{\yushunrevise{Two Forces Shaping the Hessian Structure}}
\label{sec_closer_look}
Now we conduct more fine-grained experiments on Hessian structures of NNs. In particular, we explore the 1-hidden-layer network on a Gaussian synthetic dataset. We consider both Mean-Square (MSE) and Cross-Entropy (CE) loss. The detailed experimental setups are presented in Appendix \ref{appendix_experimental_details}. 

We emphasize that these experiments can reveal more Hessian properties not shown in the CIFAR-100 experiments in Figure \ref{fig:hessian}. We highlight the new changes as follows.

\begin{itemize}[topsep=1pt,parsep=1pt,partopsep=1pt, leftmargin=*]
    \item  We change the input dimension $d$ and \# classes $C$ to amplify the effect of cross-layer Hessian components $H_{wv}$. In the CIFAR-100 example in Figure \ref{fig:hessian}, the proportions of the Hessian of hidden layer and output layer, which we abbreviate as $H_{ww}$ and $H_{vv}$, are largely imbalanced. In particular,  $H_{ww}$ occupied the vast majority of the entries, while $H_{vv}$ only accounted for a small portion. This imbalanced distribution would cause the cross-layer components, which we abbreviate as $H_{wv}$, to occupy only a minor fraction of the whole Hessian matrix, making them easily overlooked. To better illustrate the pattern of the entire Hessian, we change $(d, C) = (3072, 100)$ to $(d, C) = (500, 500)$   so that $H_{ww}$, $H_{vv}$, and $H_{wv}$ are proportionally balanced within the Hessian.
    \item We change the dataset from CIFAR-100 to Gaussian synthetic dataset. Such change suggests that the Hessian structure might be inherently general and is not  overfitted to one specific dataset like CIFAR-100. 
 
    \item In addition to the Hessian at random initialization, we present the Hessian along training until convergence. We present the loss curves in Appendix \ref{appendix_experimental_details}.
\end{itemize}

The results are shown in Figure  \ref{fig:closer_look_mse} and \ref{fig:closer_look_ce}. We summarize two findings.

\begin{itemize}[topsep=1pt,parsep=1pt,partopsep=1pt, leftmargin=*]
    \item {\bf Finding 1:} For both MSE loss and CE loss, we observe near-block-diagonal structures in $H_{ww}$ and $H_{vv}$ and such structures maintains along training.
    \item {\bf Finding 2:} For CE loss, we observe new special structures in $H_{wv}$ at random initialization: $H_{wv}$ exhibits a ``block-circulant'' pattern with periodic stripes. 
When using CE loss, the full Hessian matrix appears to be a combination of ``block-circulant'' matrix (for $H_{wv}$) and block-diagonal matrix (for $H_{ww}$ and $H_{vv}$). We refer to it as ``{\it block-circulant-block-diagonal} matrix''.  We observe that the ``block-circulant'' pattern in $H_{wv}$ vanishes as training goes on, while the near-block-diagonal patterns in  $H_{ww}$ and  $H_{vv}$ remain obvious.
\end{itemize}

\begin{figure}[t]
 \vspace{-0.3cm} 
    \centering
    \subfigure[Hessian at initialization]{\includegraphics[width=0.32\textwidth]{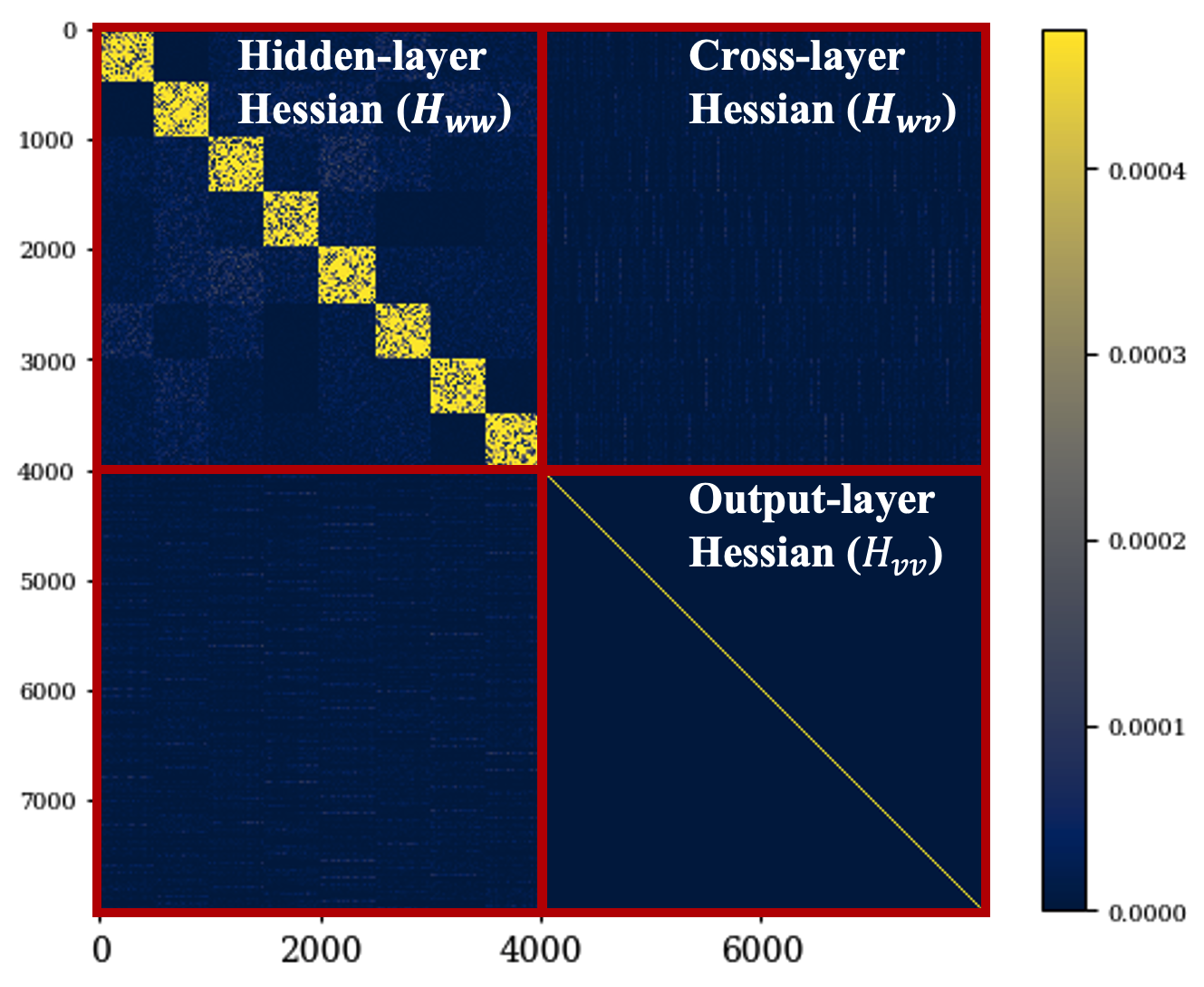}}
    \subfigure[Hessian at 10\% steps]{\includegraphics[width=0.32\textwidth]{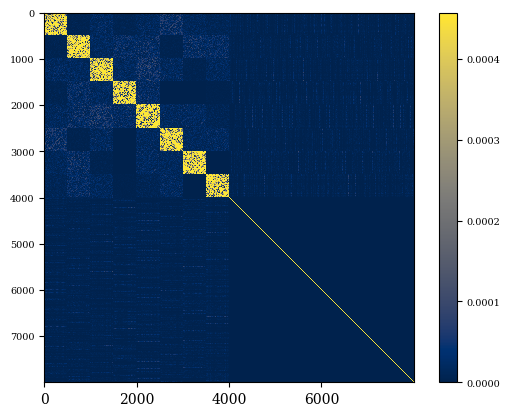}}
    \subfigure[Hessian at 25\% steps]{\includegraphics[width=0.32\textwidth]{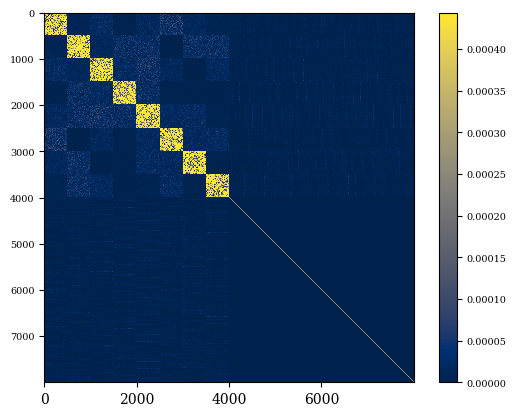}}
    \subfigure[Hessian at 50\% steps]{\includegraphics[width=0.32\textwidth]{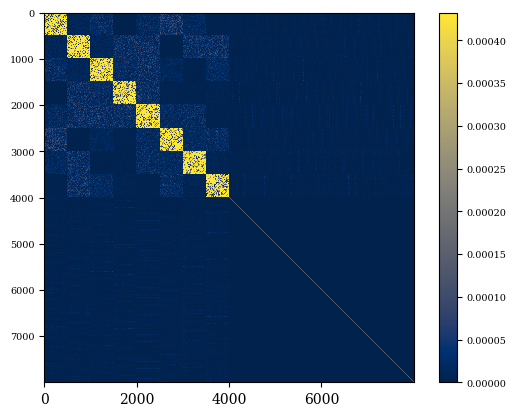}}
    \subfigure[Hessian at 75\% steps]{\includegraphics[width=0.32\textwidth]{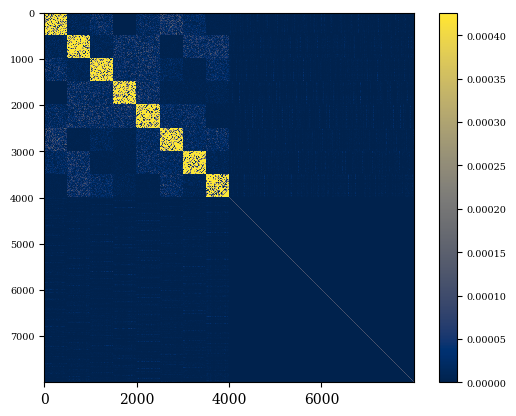}}
    \subfigure[Hessian at 100\% steps]{\includegraphics[width=0.32\textwidth]{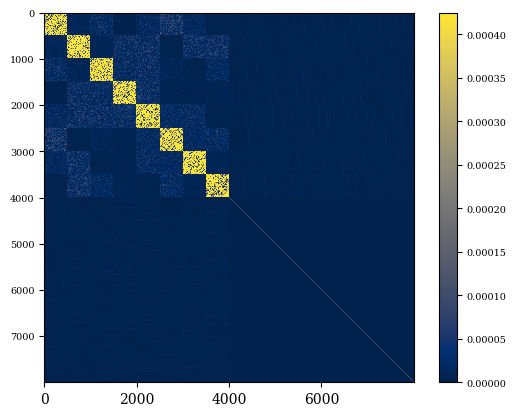}}
    \caption{\small  {\bf(a-f):}  The Hessian of a 1-hidden-layer network on Gaussian synthetic data under MSE loss. 
    We notice the near-block-diagonal patterns in  $H_{ww}$ and $H_{vv}$ with $m + C = 508$ blocks in total, and they maintain along training.}
  \label{fig:closer_look_mse}
\vspace{-0.3cm}
\end{figure}

\begin{figure}[t]
 \vspace{-0.3cm} 
    \centering
    \subfigure[Hessian at initialization]{\includegraphics[width=0.32\textwidth]{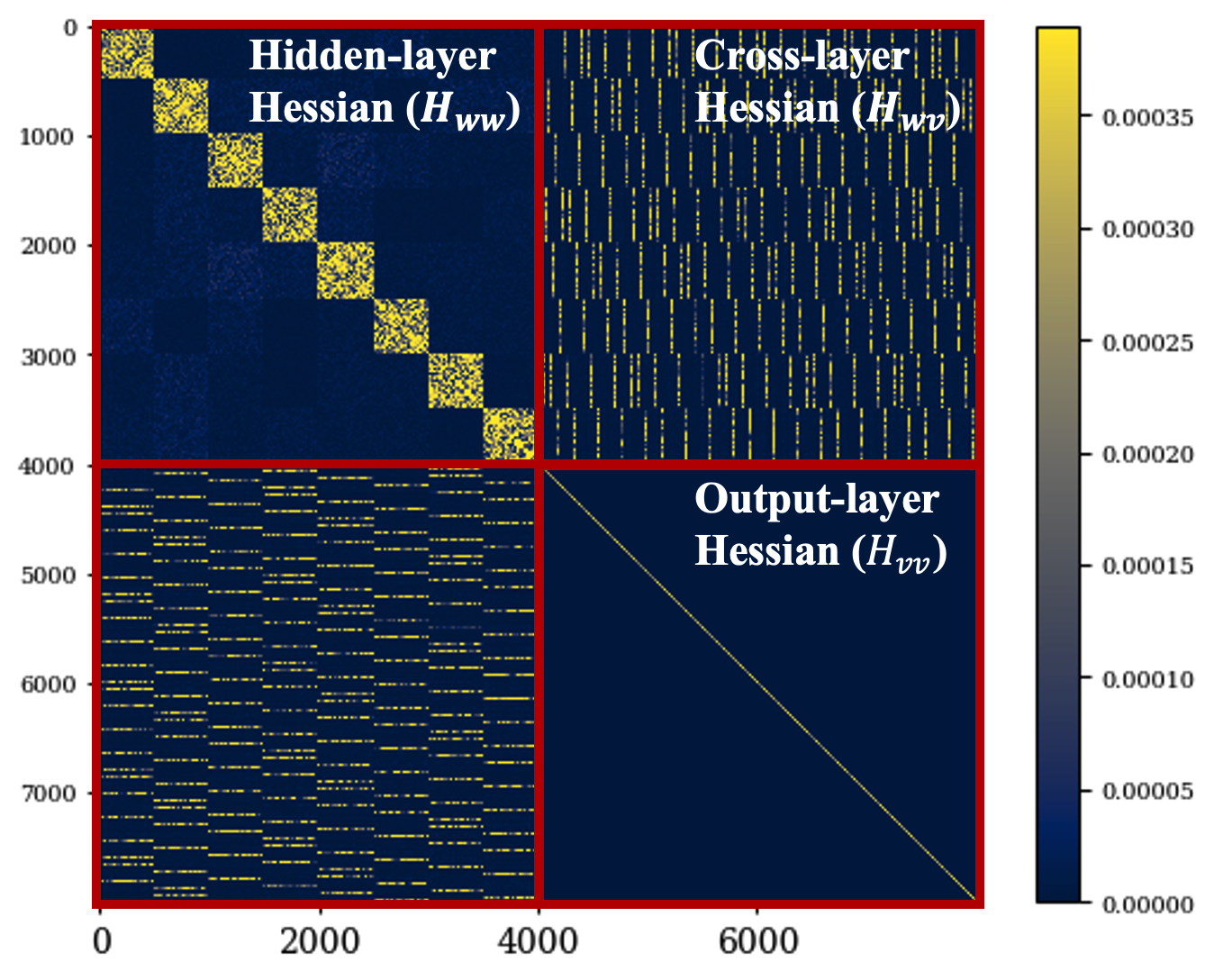}}
    \subfigure[Hessian at 10\% steps]{\includegraphics[width=0.32\textwidth]{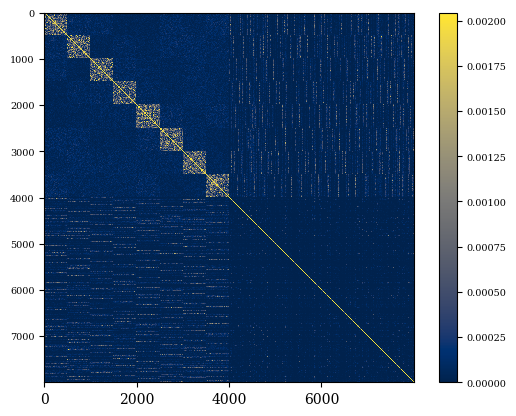}}
    \subfigure[Hessian at 25\% steps]{\includegraphics[width=0.32\textwidth]{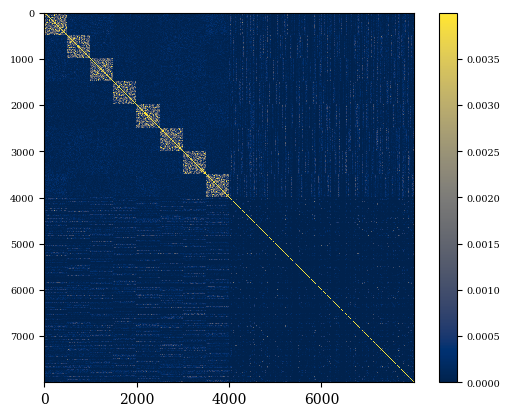}}
    \subfigure[Hessian at 50\% steps]{\includegraphics[width=0.32\textwidth]{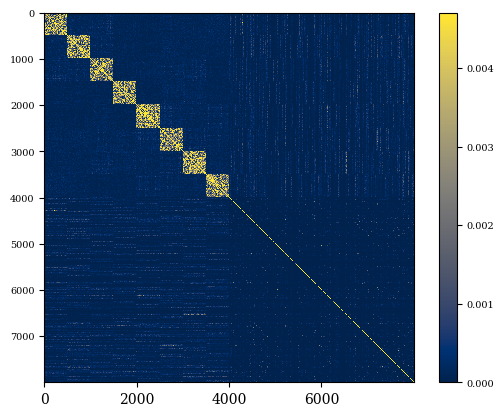}}
    \subfigure[Hessian at 75\% steps]{\includegraphics[width=0.32\textwidth]{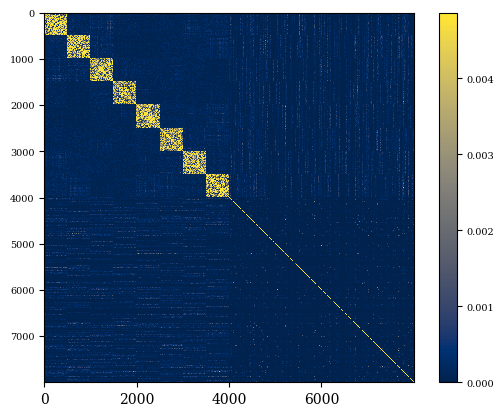}}
    \subfigure[Hessian at 100\% steps]{\includegraphics[width=0.32\textwidth]{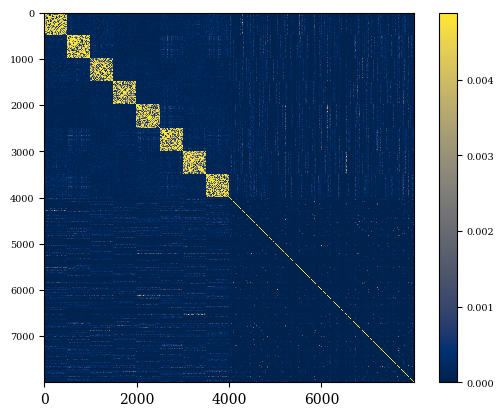}}
    \caption{\small  {\bf(a-f):}  The Hessian of a 1-hidden-layer network on Gaussian synthetic data under CE loss. At initialization, we observe the “block-circulant” pattern in $H_{wv}$, and the near-block-diagonal structure in $H_{ww}$ and $H_{vv}$ (with $m + C = 508$ blocks in total). We refer to it as {\it “block-circulant-block-diagonal matrix”}.  We notice that the “block-circulant” pattern in $H_{wv}$ vanishes along training, while the near-block-diagonal patterns in  $H_{ww}$ and $H_{vv}$ are preserved. }
  \label{fig:closer_look_ce}
\vspace{-0.3cm}
\end{figure}

\begin{snugshade}
\paragraph{Main Takeaways and insights from the experiments.} \yushunrevise{Based on the findings from Figures  \ref{fig:closer_look_mse} and \ref{fig:closer_look_ce}, we find that there are at least two forces shaping the Hessian structure. }

\begin{itemize}[topsep=1pt,parsep=1pt,partopsep=1pt, leftmargin=*]
      \item {\bf (1)  A ``static force'' rooted in the architecture design.}  For both MSE and CE loss, this force compresses the cross-neuron components in $H_{ww}$ and $H_{vv}$. This force is effective in both the initialization and the training stages. 
    \item {\bf (2)  A ``dynamic force'' arisen from training.} When using CE loss, this force gradually erases the initial ``block-circulant'' pattern in the cross-layer component $H_{wv}$ along training.
\end{itemize}
 \end{snugshade}

\yushunrevise{In the sequel, we study how both forces shape the Hessian structure. We will primarily focus on the effect of ``static force'' at random initialization, particularly, how the architecture shapes the Hessian structure for $H_{ww}$ and $H_{vv}$ for both CE loss and MSE loss. As for ``how the `dynamic force' eliminates the block-circulant pattern in $H_{wv}$ along training'', we find that it can be explained directly from Hessian expressions. We provide an initial analysis in Section \ref{sec_preliminaries} and leave more fine-grained analysis as future direction.}

\vspace{-0.2cm}
\subsection{Existing Wisdom} 
\label{sec_collobert}

Here, we revisit the results in \citep{collobert2004large}, which has remained for two-decades the dominating
understanding of the near-block-diagonal Hessian structure. 
The author attributes the near-block-diagonal structure to CE loss.  We 
will point out that this perspective might not be accurate.

\citet{collobert2004large} considered  the binary classification problem: minimizing $\ell_{\operatorname{CE}}  (f(\theta; x), y)$ where $\ell_{\operatorname{CE}} (\cdot,\cdot)$ is  CE  loss, $f(\theta; x) = \sum_{i = 1}^{n} v_i \sigma(w_i^{\top}x) \in \mathbb{R}$ is  an single-output-1-hidden-layer neural network with input $x \in \mathbb{R}^{d}$, weight $w_i \in \mathbb{R}^{d}$,  $v_i \in \mathbb{R}$, and label $y  \in \{0,1\}$. The author focused on the hidden-layer Hessian $H_{ww}$, and they point out that  off-diagonal-blocks in $H_{ww}$ would contain
\begin{small}\begin{equation}
\label{eq_hessian_calculation_collobert}
    \frac{\partial^2 \ell_{\operatorname{CE}} (f(\theta; x), y) }{\partial w_i \partial w_j^\top}=\BLUE{p \left(1 - p\right)}v_i v_j \sigma^{\prime}\left(w_i^\top x \right) \sigma^{\prime}\left(w_j^\top x \right) x x^\top \quad \text{for } i \neq j, 
\end{equation}
\end{small}
where $p = 1 /(1 + \exp (-y f(\theta,x)))$ denotes the probability of correct prediction, and $\sigma^{\prime}(\cdot)$ is the derivative of $\sigma(\cdot)$.
\citet{collobert2004large} argued that CE loss is the key factor for the near-block-diagonal structure.  The author provided a one-line intuitive explanation:   since the training objective is to maximize $p$,  the term \BLUE{$p \left(1 - p\right)$}
will decay to zero, which pushes the off-diagonal blocks to zero. Numerically, \citet{collobert2004large} reported that CE loss brings the near-block-diagonal structure in $H_{ww}$, while  MSE loss does not (their Figure 7.3 \& 7.5, also restated in Figure \ref{fig:collobert} in Appendix \ref{appendix_collobert}).
The author argued that this is because MSE loss does not produce the term \BLUE{$p \left(1 - p\right)$}. 

We find that the arguments in \citep{collobert2004large} might not be accurate. In particular:

\begin{itemize}
[topsep=1pt,parsep=1pt,partopsep=1pt, leftmargin=*]
    \item[1.] For binary classification, the term \BLUE{$p(1-p)$} occurs for both diagonal and off-diagonal blocks. This can be easily inferred in our latter analysis in {\bf Case 1} with $C =2$. Therefore, it cannot serve as a distinguishing factor between the diagonal and off-diagonal blocks. For the CE loss, then the observed block-diagonal structure in $H_{ww}$ might be due to other properties.
    \item[2.] Our numerical results in Figure \ref{fig:hessian} show that the near-block-diagonal structure in $H_{ww}$ occurs not only during the training, but also at initialization. As such, the special structure does not result from  ``maximizing $p$'' or ``minimizing $(1-p)$''.
\end{itemize}

In the subsequent analysis, we will show that the number of classes $C$, instead of \zhaoruirevise{the}  CE losses, is one key factor. Specifically, the near-block-diagonal structure in $H_{ww}$ arises as $C \rightarrow \infty$ for {\it both} \zhaoruirevise{the} MSE and \zhaoruirevise{the} CE loss.  \citep{collobert2004large} did not observe \zhaoruirevise{the} special structure under \zhaoruirevise{the} MSE loss because binary classification with $C = 2$ \zhaoruirevise{was considered}. 

We emphasize that we do not claim ``large $C$'' as the {\it only} cause for the near-block-diagonal structure in $H_{ww}$, but just that it is a sufficient condition. It is also possible that the special structure arises with small $C$ (Figure 7.3 in \citep{collobert2004large}) in some different situations, which we have not explored yet.

\section{Preliminaries and Intuitive Understanding}
\label{sec_preliminaries}

\paragraph{Our settings} We consider the multi-class classification problems with a given classification dataset $\{(x_n,y_n)\}_{n = 1}^N$, where $x_n \in \mathbb{R}^{d}$ is the input data, $y_n\in \{1,\cdots, C\}$ is the label,  and $C$ is the number of classes. This setting is quite general: it covers simple logistic regression\zhaoruirevise{s}, as well as the most advanced  LLMs.  We consider the following four cases.

\paragraph{Case 1: linear models with MSE loss}  
 Consider the linear model $f(V;x) = Vx \in \mathbb{R}^{C}$, where  $V  = (v_1^\top;\cdots;v_C^\top)\in \mathbb{R}^{C\times d}$ is the weight matrix, and $v_i \in\mathbb{R}^d$ is the weight associated with the $i$-th class (or output neuron).  Consider minimizing the MSE loss as follows: 
\begin{equation}
    \label{eq_linear_model_mse}
    \min_V \ell_{\text{MSE}}(V) := \frac{1}{N}\sum_{n= 1}^N\ \|Vx_n- \mathcal{Y}_n\|_2^2,
\end{equation}
where $\mathcal{Y}_n \in \{0,1\}^C$ is a $C$-dimensional one-hot vector with  1 at the index for the class of $y_n$  and 0 elsewhere. \zhaoruirevise{The Hessian matrix is, } for $i,j \in [C]$:
\begin{equation}
\label{eq_linear_model_mse_Hessian}
\left\{\begin{array}{l}
    \frac{\partial^2\ell_{\text{MSE}}(V)}{\partial v_i  \partial v_i^\top} = \frac{1}{N}\sum_{n= 1}^N x_n x_n^\top, \\[2mm]
    \frac{\partial^2\ell_{\text{MSE}}(V)}{\partial v_i  \partial v_j^\top} = 0_{d\times d},\quad i\ne j.
\end{array}\right.
\end{equation}
In subsequent formulas in this paper, $i$ and $j$ are different indices.
\zhaoruirevise{ Here, the Hessian is always} block-diagonal with $C$ blocks. Note that the expression in \eqref{eq_linear_model_mse_Hessian} holds for general real-valued vector $\mathcal{Y}_n \in \mathbb{R}^C$, so the same Hessian structure also arises in regression tasks.

\paragraph{Case 2: linear models with CE loss.}
We now change the loss function in {\bf Case 1} to \zhaoruirevise{the} CE loss.
\begin{equation}
    \label{eq_linear_model_ce}
    \min_V  \ell_{\text{CE}}(V) := -  \frac{1}{N}\sum_{n= 1}^N \log\left(\frac{\exp({v_{y_n}^\top x_n})}{\sum_{c=1}^C\exp({v_c^\top x_n})}\right).
\end{equation}
Define \zhaoruirevise{$p_{n,i} := \exp({v_i^\top x_n})/\l({\sum_{c=1}^C\exp({v_c^\top x_n})}\r)$}. \zhaoruirevise{The Hessian matrix is, }for $i,j \in [C]$.
\begin{equation}
\label{eq_linear_model_ce_Hessian}
\left\{\begin{array}{l}
    \frac{\partial^2\ell_{\text{CE}}(V)}{\partial v_i  \partial v_i^\top} = \frac{1}{N}\sum_{n= 1}^N  p_{n,i} (1-p_{n,i}) x_n x_n^\top,   \\[2mm]
    \frac{\partial^2\ell_{\text{CE}}(V)}{\partial v_i  \partial v_j^\top} =-  \frac{1}{N}\sum_{n= 1}^N p_{n,i} p_{n,j} x_n x_n^\top .
\end{array}\right.
\end{equation}
\begin{snugshade}
{\bf Intuitive understanding:} at random initialization, suppose each entry in $V$ follows i.i.d. zero-mean Gaussian distribution, we have $p_{n,i} \approx \frac{1}{C}$ for all $n \in [N], i \in [C]$. As such: 
\begin{equation}
\frac{\bigg\|  \frac{\partial^2\ell_{\text{CE}}(V)}{\partial v_i  \partial v_j^\top} \bigg\|_{\operatorname{F}}}{\bigg\|    \frac{\partial^2\ell_{\text{CE}}(V)}{\partial v_i  \partial v_i^\top} 
 \bigg\|_{\operatorname{F}}}
  \approx  \frac{ \sum_{n=1}^N p_{n,i} p_{n,j}} {\sum_{n=1}^N  p_{n,i} (1-p_{n,i})}\approx \frac{  \frac{1}{C^2}}{ \frac{1}{C}\left(1-\frac{1}{C}\right)} = \frac{1}{C-1},
\end{equation} 
which pushes the Hessian to become block-diagonal as $C \rightarrow \infty$. 
\vspace{-0.2cm}
\end{snugshade}
\paragraph{Case 3:  1-hidden-layer networks with MSE loss} We now consider the 1-hidden-layer network with $m$ hidden neurons: $f(W,V;x) =  V \sigma (Wx) \in \mathbb{R}^C$, where $W = (w_1^\top;\cdots, w_m^\top) \in\mathbb{R}^{m\times d}$; $\sigma(z) = \max\{0,z\}$ is the ReLU activation and is applied elementwise to $Wx$; $V = (v_1^\top;\cdots;v_C^\top)\in \mathbb{R}^{C\times m}$.  Consider the MSE loss as follows. 
\begin{equation}
    \label{eq_nn_mse}
    \min_{W,V} \ell_{\text{MSE}}(W,V) := \frac{1}{N}\sum_{n= 1}^N\ \|V \sigma (Wx_n)- \mathcal{Y}_n\|_2^2.
\end{equation}
The hidden-layer Hessian $H_{ww}$ is: for $i, j \in [m]$,
\begin{equation}
\label{eq_nn_mse_Hessian_w}
\left\{\begin{array}{l}
    \frac{\partial^2\ell_{\text{MSE}}(W,V)}{\partial w_i  \partial w_i^\top} = \frac{1}{N} \left(\sum_{c= 1}^C v_{c,i}^2 \right) \left(\sum_{n= 1}^N    \mathbf{1}(w_i^\top x_n >0) x_n x_n^\top\right),  \\
        \frac{\partial^2\ell_{\text{MSE}}(W,V)}{\partial w_i  \partial w_j^\top} = \frac{1}{N} \left(\sum_{c= 1}^C v_{c,i} v_{c,j} \right) \left(\sum_{n= 1}^N  \mathbf{1}(w_i^\top x_n >0) \mathbf{1}(w_j^\top x_n >0) x_n x_n^\top \right).  
\end{array}\right.
\end{equation}
The output-layer Hessian $H_{vv}$  is: for $i, j \in [C]$,
\begin{equation}
\label{eq_nn_mse_Hessian_v}
\left\{\begin{array}{l}
    \frac{\partial^2\ell_{\text{MSE}}(W,V)}{\partial v_i  \partial v_i^\top} = \frac{1}{N}\sum_{n= 1}^N \sigma(Wx_n) \sigma(Wx_n)^\top,  \\
    \frac{\partial^2\ell_{\text{MSE}}(W,V)}{\partial v_i  \partial v_j^\top} = 0_{d\times d},
\end{array}\right.
\end{equation}
The output-layer Hessian is block-diagonal. We now discuss the hidden-layer Hessian.
\begin{snugshade}
{\bf \zhaoruirevise{Intuitive understanding}:} at random initialization, suppose \zhaoruirevise{entries} in $v_i \in \mathbb{R}^d$ \zhaoruirevise{follow an} i.i.d. zero-mean Gaussian distribution, then
\begin{equation}
    \frac{\bigg\|    \frac{\partial^2\ell_{\text{MSE}}(W,V)}{\partial w_i  \partial w_j^\top}
 \bigg\|_{\operatorname{F}}}{\bigg\|  \frac{\partial^2\ell_{\text{MSE}}(W,V)}{\partial w_i  \partial w_i^\top} \bigg\|_{\operatorname{F}}
}
  \approx \frac{\left(\sum_{c= 1}^C v_{c,i} v_{c,j} \right)  } {   \left(\sum_{c= 1}^C v_{c,i}^2 \right) } \overset{C \rightarrow \infty}{=} \frac{\operatorname{Cov}(v_{i,i},v_{i,j}) }{ \operatorname{Var}(v_{i,i})}.
\end{equation}
As $v_{i,i},v_{i,j}$ are independent, $\operatorname{Cov}(v_{i,i},v_{i,j}) = 0$ and thus the block-diagonal structure in $H_{ww}$ occurs as $C \rightarrow \infty$.
\vspace{-0.2cm}
\end{snugshade}

We now discuss the cross-layer component  $H_{wv}$ under MSE loss:
\bea
\frac{\partial^2 \ell_{\text{MSE}}}{\partial w_i\partial v_j^\top }=\frac{2}{N}\sum_{n=1}^N\l[    \left(  \sigma(Wx_n)^\top v_j - \mathcal{Y}_{n,j}  \right)  \mf{1}
(w_i^\top x_n>0)x_n e_i^\top + v_{j,i} \mf{1}
(w_i^\top x_n>0)x_n \sigma(Wx_n)^\top   \r],
\eea
where $e_i \in \mathbb{R}^m$ is an one-hot vector with the $i$-th component equals to 1. Note that the 2nd term has expectation 0 when $v_{j,i}$ is initialized as a zero-mean Gaussian distribution. As for the 1st term, it is a matrix of the form
\begin{equation}
\label{eq_hwv}
    \left[\begin{array}{cccccc}
0 & \cdots & a_{1, i} & 0 & \cdots & 0 \\
\vdots & \ddots & \vdots & \vdots & \ddots & \vdots \\
0 & \cdots & a_{d, i} & 0 & \cdots & 0
\end{array}\right] \in \mathbb{R}^{d \times m},
\end{equation}

which is a matrix with one non-zero column at position $i$ with {\small $$a_{d^\prime, i}=\frac{1}{N} \sum_{n=1}^N  \left(  \sigma(Wx_n)^\top v_j - \mathcal{Y}_{n,j}  \right)  \mf{1}
(w_i^\top x_n>0)x_{n,d^\prime},\quad   d^\prime \in [d].$$}   
Note that $v_{j}$ is initialized as zero-mean Gaussian distribution, so the inner product  $ \sigma(Wx_n)^\top v_j$ has expectation 0. Further,  as the training goes,   {\small $ \left(  \sigma(Wx_n)^\top v_j - \mathcal{Y}_{n,j}  \right) \rightarrow 0$} and thus the 1st  term in $H_{wv}$ shall approach 0 along training.  Numerically, we observe that $H_{wv}$ under MSE loss is indeed negligible compared to $H_{ww}$ and $H_{vv}$. This is observed throughout the training, including at the initialization (see Figure \ref{fig:closer_look_mse}).

\paragraph{Case 4: 1-hidden-layer networks with CE loss} We now consider 1-hidden-layer networks with CE loss.
\begin{equation}
    \label{eq_nn_ce}
    \min_{W,V}  \ell_{\text{CE}}(W,V) := - \frac{1}{N} \sum_{n= 1}^N \log\left(\frac{\exp({v_{y_n}^\top \sigma(Wx_n)})}{\sum_{c=1}^C\exp({v_c^\top  \sigma(Wx_n)})}\right).
\end{equation}
The hidden-layer Hessian  $H_{ww}$  is: for $i, j \in [m]$,

{\small
\begin{equation}
\label{eq_nn_ce_Hessian_w}
\left\{\begin{array}{l}
    \frac{\partial^2\ell_{\text{CE}}(W,V)}{\partial w_i  \partial w_i^\top} = \frac{1}{N}\sum_{n= 1}^N  \left(\sum_{c=  1}^C p_{n,c} v_{c,i}^2 - \left(\sum_{c=1}^C p_{n,c} v_{c,i}\right)^2\right)   \mathbf{1}(w_i^\top x_n >0) x_n x_n^\top,  \\
        \frac{\partial^2\ell_{\text{CE}}(W,V)}{\partial w_i  \partial w_j^\top} =  \frac{1}{N}\sum_{n= 1}^N   \left(\sum_{c= 1}^C p_{n,c} v_{c,i} v_{c,j} - \left(\sum_{c=1}^C p_{n,c} v_{c,i}\right)  \left(\sum_{c=1}^C p_{n,c} v_{c,j}\right)\right) \mathbf{1}(w_i^\top x_n >0) \mathbf{1}(w_j^\top x_n >0) x_n x_n^\top.   
\end{array}\right.
\end{equation}
}
The output-layer Hessian $H_{vv}$  is: for $i, j \in [C]$, 
\begin{equation}
\label{eq_nn_ce_Hessian_v}
\left\{\begin{array}{l}
    \frac{\partial^2\ell_{\text{CE}}(W,V)}{\partial v_i  \partial v_i^\top} =  \frac{1}{N}\sum_{n= 1}^N  p_{n,i} (1-p_{n,i}) \sigma(Wx_n) \sigma(Wx_n)^\top,   \\
    \frac{\partial^2\ell_{\text{CE}}(W,V)}{\partial v_i  \partial v_j^\top} =- \frac{1}{N}\sum_{n= 1}^N p_{n,i} p_{n,j} \sigma(Wx_n) \sigma(Wx_n)^\top.
\end{array}\right.
\end{equation}
\begin{snugshade}
{\bf Intuitive understanding:} at random initialization, suppose \zhaoruirevise{entries} in $W, V$ follows i.i.d. zero-mean Gaussian distribution, we have $p_{n,i} \approx \frac{1}{C}$ for all $n \in [N], i \in [C]$. As such:
\begin{equation}
    \frac{\bigg\|    \frac{\partial^2\ell_{\text{CE}}(W,V)}{\partial w_i  \partial w_j^\top}
 \bigg\|_{\operatorname{F}}}{\bigg\|  \frac{\partial^2\ell_{\text{CE}}(W,V)}{\partial w_i  \partial w_i^\top} \bigg\|_{\operatorname{F}}}
  \approx  \frac{ \left(\sum_{c= 1}^C v_{c,i} v_{c,j}  - \left(\sum_{c=1}^C v_{c,i}\right)\left(\sum_{c=1}^C v_{c,j}\right) \right) /C  } {   \left(\sum_{c= 1}^C v_{c,i}^2 
 - \left(\sum_{c=1}^C v_{c,i}\right)^2\right) /C } \overset{C \rightarrow \infty}{=} \frac{\operatorname{Cov}(v_{i,i},v_{i,j})}{  \operatorname{Var}(v_{i,i})}.
\end{equation}
Since $v_{i,i},v_{i,j}$ are independent, $\operatorname{Cov}(v_{i,i},v_{i,j}) = 0$ and thus 
the block-diagonal structure in $H_{ww}$ occurs as $C\rightarrow \infty$. Similarly, we have
\begin{equation}
    \frac{\bigg\|    \frac{\partial^2\ell_{\text{CE}}(W,V)}{\partial v_i  \partial v_j^\top}
 \bigg\|_{\operatorname{F}}}{\bigg\|  \frac{\partial^2\ell_{\text{CE}}(W,V)}{\partial v_i  \partial v_i^\top} \bigg\|_{\operatorname{F}}
}
    \approx  \frac{\sum_{n=1}^N  p_{n,i} p_{n,j}} {\sum_{n=1}^N  p_{n,i} (1-p_{n,i})}\approx \frac{  \frac{1}{C^2}}{ \frac{1}{C}\left(1-\frac{1}{C}\right)} = \frac{1}{C-1},
\end{equation}

and thus 
the block-diagonal structure  in $H_{vv}$ arises as $C\rightarrow \infty$. 

\vspace{-0.2cm}
\end{snugshade}

We now discuss the cross-layer component  $H_{wv}$ under CE loss. We will explain the block-circulant structure at initialization (i.e., Figure \ref{fig:closer_look_ce} (a)) and why it vanishes along training (i.e., Figure \ref{fig:closer_look_ce} (b-f)). We find that this phenomenon can be seen by a direct Hessian calculation. The cross-layer Hessian
is
\bea
\frac{\partial^2 \ell_{\text{CE}}}{\partial w_i\partial v_j^\top }=\frac{1}{N}\sum_{n=1}^N\l[(p_{n,j}-\delta_{y_n,j})\mf{1}
(w_i^\top x_n>0)x_n e_i^\top +\sum_{c=1}^C (\delta_{j,c}-p_{n,c})p_{n,j}v_{c,i}\mf{1}(w_i^\top x_n>0)x_n\sigma(Wx_n)^\top\r].\eea
When $C$ is large, \zhaoruirevise{by the law of large number and approximating $p_{n,j}\approx 1/C$, for the 2nd-term we have} 
\bea
&\sum_{c=1}^C (\delta_{j,c}-p_{n,c})p_{n,j}v_{c,i}\mf{1}(w_i^\top x_n>0)x_n\sigma(Wx_n)^\top\\
\approx & p_{n,j} v_{j,i}\mf{1}(w_i^\tp x_n>0) x_n \sigma(Wx_n)^\tp-p_{n,j} x_n \sigma(Wx_n)^\tp \l(\frac{1}{C}\sum_{c=1}^C  v_{c,i}\mf{1}(w_i^\tp x>0)\r)\\
\approx& \frac{1}{C}\cdot  v_{j,i}\mf{1}(w_i^\tp x_n>0) x_n \sigma(Wx_n)^\tp.
\eea
Thus 
\bea
\frac{\partial^2 \ell_{\text{CE}}}{\partial w_i\partial v_j^\top } \approx \frac{1}{N}\sum_{n=1}^N(p_{n,j}-\delta_{y_n,j})\mf{1}
(w_i^\top x_n>0)x_n e_i^\top + \mathcal{O}\left(\frac{1}{C}\right),
\eea
As such, the leading term of $H_{wv}$ under CE loss  is a matrix of the form
\begin{equation}
\label{eq_hwv}
    \left[\begin{array}{cccccc}
0 & \cdots & a_{1, i} & 0 & \cdots & 0 \\
\vdots & \ddots & \vdots & \vdots & \ddots & \vdots \\
0 & \cdots & a_{d, i} & 0 & \cdots & 0
\end{array}\right] \in \mathbb{R}^{d \times m},
\end{equation}
which is a matrix with one non-zero column at position $i$ with {\small $$a_{d^\prime, i}=\frac{1}{N} \sum_{n=1}^N \left(p_{n, c} - \delta_{y_n, c}\right) \mathbf{1}\left(w_i^\top x_n \geq 0\right) x_{n, d^{\prime}},\quad   d^\prime \in [d].$$}   This explains the initial block-circulant structure in $H_{wv}$ under CE loss. As the training goes,   $\left(p_{n, c} - \delta_{y_n, c}\right)\rightarrow 0$ and the block-circulant structure disappears.   This shows the ``dynamic force'' arisen from training.

\section{Main Results}

We now present our rigorous statements. 
We first state some standard assumptions. 
\begin{assum}
\label{assum_1}
   \zhaoruirevise{The entries} of the data matrix $X_N=(x_1,\cdots,x_N) \in \mathbb{R}^{d\times N}$ \zhaoruirevise{are} i.i.d. $\mathcal{N}(0,1)$. 
\end{assum}
\begin{assum}
\label{assum_2}
    \zhaoruirevise{The} model weights \zhaoruirevise{in} $W$ and $V$ are initialized by LeCun initialization. That is: for the linear model, $V_{i,j}\overset{\operatorname{i.i.d.}}{\sim}\mathcal{N}(0, \frac{1}{d})$, $i\in[C],j\in[d]$; for 1-hidden-layer network, $W_{i,j}\overset{\operatorname{i.i.d.}}{\sim}\mathcal{N}(0, \frac{1}{d})$, $i\in[m],j\in[d]$,  $V_{i,j}\overset{\operatorname{i.i.d.}}{\sim}\mathcal{N}(0, \frac{1}{m})$, $i\in[C],j\in[m]$. 
\end{assum}
\zhaoruirevise{Note} that Assumption \ref{assum_2} is widely adopted in NNs \citep{sun2019optimization}. Assumption \ref{assum_1} on data distribution is standard in random matrix theory \citep{pastur2020random}. It is possible to extend the  Gaussian distribution to, 
 e.g., Gaussian orthogonal ensembles and \zhaoruirevise{more} general i.i.d. distribution. However, such generalization is non-trivial and each case may require an independent paper (e.g. \citet{pastur2022haarorthogonal,pastur2023iid}).
\begin{thm}
\label{thm_linear_model}
    {\bf (Linear models.)}  Consider the Hessian expressions in \eqref{eq_linear_model_ce_Hessian} and assume Assumption\zhaoruirevise{s} \ref{assum_1} and \ref{assum_2} hold. Suppose  $d,N\rw\infty, \frac
    {d}{N}\rw \gamma\in (0,+\infty)$, then for fixed $C\geq 2$, it holds almost surely that
        \bea\label{eq1thm3}
        \lim_{d,N\rw \infty}\frac{1}{d}\bigg\|\frac{\partial^2\ell_{\text{CE}}(V)}{\partial v_i  \partial v_i^\top}\bigg\|_{\operatorname{F}} = g_{ii}(\gamma,C),\quad \forall i \in [C],
        \eea
        \bea\label{eq2thm3}
        \lim_{d,N\rw \infty}\frac{1}{d}\bigg\|\frac{\partial^2\ell_{\text{CE}}(V)}{\partial v_i  \partial v_j^\top}\bigg\|_{\operatorname{F}} = g_{ij}(\gamma,C),\quad \forall i,j \in [C], i \ne j,
        \eea
    where functions $g_{ii},g_{ij}$ are given in Section \ref{appendix_thm1}. Furthermore,
\bea\label{eq3thm3}
\lim_{C\rw\infty}C^2 g_{ii}(\gamma,C)=\gamma e+1,  
\eea
\bea\label{eq4thm3}
\lim_{C\rw\infty}C^4 g_{ij}(\gamma,C)=\gamma e^2+1.  
\eea
\end{thm}
Theorem \ref{thm_linear_model} implies that we have the following relation between the diagonal and off-diagonal blocks:

        \bea\label{eq_thm_linear_model_ratio}
        \lim_{d,N\rw\infty} \frac{\bigg\|\frac{\partial^2\ell_{\text{CE}}(V)}{\partial v_i  \partial v_j^\top}\bigg\|_{\operatorname{F}}^2} {\bigg\|\frac{\partial^2\ell_{\text{CE}}(V)}{\partial v_i  \partial v_i^\top}\bigg\|_{\operatorname{F}}^2}  =  \frac{g_{ij}(\gamma,C)}{g_{ii}(\gamma,C)}, \quad 
        \lim_{C\rw\infty}\frac{C^2g_{ij}(\gamma,C)}{g_{ii}(\gamma,C)}=\frac{\gamma e^2+1}{\gamma e+1}.
        \eea
\zhaoruirevise{When $C\rw\infty$, the ratio vanishes at the rate $\mathcal{O}(1/C^2)$, and the block-diagonal structure emerges.}

\zhaoruirevise{The next theorem presents} a similar result for 1-hidden-layer networks.

\begin{thm}
\label{thm_nn}
        {\bf (1-hidden-layer networks.)} Consider the Hessian expressions \zhaoruirevise{in} \eqref{eq_nn_mse_Hessian_w} to \eqref{eq_nn_ce_Hessian_v}, and assume Assumptions \ref{assum_1} and \ref{assum_2} hold. Then for any fixed $m\geq 3$, suppose $d,N\rw\infty, \frac
    {d}{N}\rw \gamma\in (0,+\infty)$, it holds that
{\small   \bea\label{thm2_NN_CE_Hww}
\lim_{d,N\rw \infty}\frac{1}{d}\E\l[\bigg\|\frac{\partial^2\ell_{\text{CE}}(W,V)}{\partial w_i  \partial w_i^\top}\bigg\|_{\operatorname{F}}^2\r]=h_{ii}(\gamma,C),\quad
\lim_{d,N\rw \infty}\frac{1}{d}\E\l[\bigg\|\frac{\partial^2\ell_{\text{CE}}(W,V)}{\partial w_i  \partial w_j^\top}\bigg\|_{\operatorname{F}}^2\r]=h_{ij}(\gamma,C),
\eea
\bea\label{thm2_NN_MSE_Hww}
\lim_{d,N\rw\infty}\frac{1}{d}\E\l[\bigg\|\frac{\partial^2\ell_{\text{MSE}}(W,V)}{\partial w_i  \partial w_i^\top}\bigg\|_{\operatorname{F}}^2\r]=u_{ii}(\gamma,C),\quad
\lim_{d,N\rw\infty}\frac{1}{d}\E\l[\bigg\|\frac{\partial^2\ell_{\text{MSE}}(W,V)}{\partial w_i  \partial w_j^\top}\bigg\|_{\operatorname{F}}^2\r]=u_{ij}(\gamma,C),
\eea }
\bea\label{NN_CV_Hvv}
\lim_{d,N\rw\infty}\E\l[\bigg\|\frac{\partial^2\ell_{\text{CE}}(W,V)}{\partial v_i  \partial v_i^\top}\bigg\|_{\operatorname{F}}^2\r]=q_{ii}(\gamma,C),\quad
\lim_{d,N\rw\infty}\E\l[\bigg\|\frac{\partial^2\ell_{\text{CE}}(W,V)}{\partial v_i  \partial v_j^\top}\bigg\|_{\operatorname{F}}^2\r]=q_{ij}(\gamma,C),
\eea
where functions $h_{ii},h_{ij},u_{ii},u_{ij},q_{ii},q_{ij}$ are given in Section \ref{appendix_thm2}. Furthermore, we have 
\bea
\lim_{C\rw\infty} h_{ii}(\gamma,C)=\frac{1+2\gamma}{4m^2},\quad
\lim_{C\rw\infty} Ch_{ij}(\gamma,C)=\frac{\gamma(m-1)^2}{2^m(m-2)^3m}\l(\sqrt{\frac{m}{m-2}}+1\r)^{m-2},
\eea
\bea
\lim_{C\rw\infty}\frac{u_{ii}(\gamma,C)}{C^2}=\frac{1+2\gamma}{4m^2},\quad
\lim_{C\rw\infty}\frac{u_{ij}(\gamma,C)}{C}=\frac{1+4\gamma}{16m^2},
\eea
\bea
\lim_{C\rw\infty} C^2 q_{ii}(\gamma,C)=ma_{12}b_1^{m-1}+m(m-1)a_{11}^2b_1^{m-2},
\eea
\bea
\lim_{C\rw\infty} C^4 q_{ij}(\gamma,C)=ma_{22}b_2^{m-1}+m(m-1)a_{21}^2b_2^{m-1},
\eea
where the constant terms $a_{11}, a_{12}, a_{21}, a_{22}, b_1, b_2$ are presented in \eqref{eq_constant_terms_Hvv} in Section \ref{appendix_thm2}.

\end{thm}
Similar to the implication of Theorem \ref{thm_linear_model} in \eqref{eq_thm_linear_model_ratio}, Theorem 2 implies that the ratios
{\small \begin{equation}
\label{eq_thm_nn_ratio}
    \lim_{d,N\rw\infty} \frac{\E\l[\bigg\|\frac{\partial^2\ell_{\text{CE}}(W,V)}{\partial w_i  \partial w_j^\top}\bigg\|_{\operatorname{F}}^2\r]} {\E\l[\bigg\|\frac{\partial^2\ell_{\text{CE}}(W,V)}{\partial w_i  \partial w_i^\top}\bigg\|_{\operatorname{F}}^2\r]},\quad \lim_{d,N\rw\infty} \frac{\E\l[\bigg\|\frac{\partial^2\ell_{\text{MSE}}(W,V)}{\partial w_i  \partial w_j^\top}\bigg\|_{\operatorname{F}}^2\r]} {\E\l[\bigg\|\frac{\partial^2\ell_{\text{MSE}}(W,V)}{\partial w_i  \partial w_i^\top}\bigg\|_{\operatorname{F}}^2\r]},\quad \lim_{d,N\rw\infty} \frac{\E\l[\bigg\|\frac{\partial^2\ell_{\text{CE}}(W,V)}{\partial v_i  \partial v_j^\top}\bigg\|_{\operatorname{F}}^2\r]} {\E\l[\bigg\|\frac{\partial^2\ell_{\text{CE}}(W,V)}{\partial v_i  \partial v_i^\top}\bigg\|_{\operatorname{F}}^2\r]}
\end{equation}}
vanish at the rate $\mathcal{O}(1/C),\ \mathcal{O}(1/C),\ \mathcal{O}(1/C^2)$, respectively, and the block-diagonal structure in $H_{ww}$ and $H_{vv}$ also emerges as $C$ increases.

The above results rigorously quantify the structure in $H_{ww}$ and $H_{wv}$. As for the ``block-circulant'' structure in the cross-layer component $H_{wv}$  and how it vanishes along training (particularly for CE loss),  we find
that it can be explained directly from Hessian expressions. We have provided an initial analysis in Section \ref{sec_preliminaries} and leave more rigorous analysis as a future direction.

\section{Proof Sketch and Technical Challenges}

Now we \zhaoruirevise{explain} the major technical challenges and \zhaoruirevise{the main ideas in our proofs}. We primarily introduce the proof procedure for Theorem \ref{thm_linear_model}, i.e., linear models with CE loss ({\bf Case 2}).  Despite the simple form of linear models,  we find that it is rather non-trivial to characterize its Hessian structure, and the classical random matrix approaches {\it cannot} be directly applied. After introducing the proof for Theorem \ref{thm_linear_model},  we will discuss how to extend our analysis to Theorem \ref{thm_nn}  ({\bf Case 3} and {\bf 4}).
{\bf Since the complete proof of Theorem \ref{thm_linear_model} and \ref{thm_nn} are  rather  long and technical, we  present them in Section \ref{appendix_thm1} and \ref{appendix_thm2}, \zhaoruirevise{respectively}.}

\paragraph{Challenges for proving Theorem \ref{thm_linear_model}}

We first rewrite the Hessian expression in {\bf Case 2} as follows, and then we discuss why the classical random matrix approaches cannot be directly applied here. Due to the limited space, we only discuss the diagonal blocks $\frac{\partial^2\ell_{\text{CE}}(V)}{\partial v_i  \partial v_i^\top}$.  The same challenges and solutions also apply to the off-diagonal blocks $\frac{\partial^2\ell_{\text{CE}}(V)}{\partial v_i  \partial v_j^\top}$, which we omit here.

\begin{equation}
\label{eq_linear_model_ce_Hessian_rewrite}
    \frac{\partial^2\ell_{\text{CE}}(V)}{\partial v_i  \partial v_i^\top} \overset{\eqref{eq_linear_model_ce_Hessian}}{=} \frac{1}{N}\sum_{n= 1}^N  p_{n,i} (1-p_{n,i}) x_n x_n^\top   :=\frac{1}{N} X_N \Lambda_N X_N^\top \in \mathbb{R}^{d\times d}, 
\end{equation}
where $X_N  = (x_1, \cdots, x_N) \in \mathbb{R}^{d\times N}$,  $\Lambda_N = \operatorname{diag}( p_{1,i} (1-p_{1,i}), \cdots,  p_{N,i} (1-p_{N,i})) \in \mathbb{R}^{N\times N}$, and $p_{n,i} := \frac{\exp({v_i^\top x_n})}{\l({\sum_{c=1}^C\exp({v_c^\top x_n})}\r)}$. Note that both $X_N$ and $\Lambda_N$ are random matrices.  How to characterize $\|\frac{1}{N} X_N \Lambda_N X_N^\top\|_{\operatorname{F}}$? We first recall some classical results in random matrix theory.

\paragraph{Classical results from random matrix theory and why they cannot be directly applied}

We first introduce some basic concepts in random matrix theory.
\begin{itemize}
[topsep=1pt,parsep=1pt,partopsep=1pt, leftmargin=*]
    \item {\bf Eigenvalue distribution.} 
    For a symmetric matrix $A \in \mathbb{R}^{d\times d}$, we define  $\mu_A=\frac{1}{d}\sum_{i=1}^d \delta_{\lambda_i(A)}$ as the empirical eigenvalue distribution of $A$. $\mu_A$ is a probability measure on $\R$ that assigns equal probability $\frac{1}{d}$ to each eigenvalue. Note that when $A$ is a random matrix, $\mu_A$ is a random measure. For a sequence of random matrices $(A_n)_{n=1}^\infty$, we will consider the weak convergence of its eigenvalue distribution $(\mu_{A_n})_{n=1}^{\infty}$.
    \item {\bf Stieltjes transform.} For a probability measure $\nu$ on $\R$,
    The Stieltjes transform of $\nu$ is defined as 
    $$s_\nu(z)=\int_\mathbb{R} \frac{1}{x-z}d\nu(x),\quad z\in \mathbb{C}^+  \textbackslash \operatorname{supp}(\nu).$$
    A probability measure is uniquely characterized by its Stieltjes transform. For a symmetric matrix $A$, we write $s_{\mu_A}(z)$ as $s_{A}(z)$ for short. A sequence of eigenvalue distributions $(\mu_{A_n})_{n=1}^\infty$ converges weakly to a probability measure $\mu$ if and only if 
    $
    s_{A_n}(z)\rw s_{\mu}(z),\ \forall z\in \mb{C}^+.
    $
    \\
    
    Now we notice that $\|A\|_{\operatorname{F}}^2$ is nothing but the 2nd-order moment of $\mu_A$. \zhaoruirevise{Moreover}, we can retrieve the moments of $\mu_A$ from
    $s_A(z)$ by
    
\begin{equation}
\label{eq_stieltjes_power_series}
    s_{A}(z)= -\frac{1}{z} -\frac{m_1}{z^2} -\frac{m_2}{z^3} - \cdots,\quad z\rw\infty,
\end{equation}

    where $m_k = \int_{\mathbb{R}} t^k d\mu_A(t)$ denotes the $k$-th order moment. Therefore, the calculation of $\|A\|_\text{F}^2$ can be achieved by finding the limiting eigenvalue distribution of $A$ as the matrix size tends to infinity, which is a classical topic in random matrix theory. Typically, the random matrices of the type $X_N \Lambda_N X_N^\top$ are closely related to the sample covariance matrices. The case that $\Lambda_N$ is deterministic or independent of $X_N$ has already been deeply studied (e.g., \citet{YaoBook,Baibook,BaiZhou,Pastur67}). For the limiting eigenvalue distribution, we have the following classical result, the generalized Marcenko-Pastur theorem.

\end{itemize}

\begin{prop}
    \label{GMP}\citep{Pastur67}
    Consider random matrices $X_N \in \mathbb{R}^{d\times N}$ with entries i.i.d. with mean 0 and variance 1; and $\Lambda_N \in \mathbb{R}^{N\times N}$ \zhaoruirevise{which is}  either deterministic or independent of $X_N$. Suppose that the eigenvalue distribution of $\Lambda_N$ converges weakly almost surely to a deterministic probability measure $\nu$. Let $A_N=\frac{1}{d}X_N \Lambda_N X_N^\top$, then as $N,d\rw\infty, d/N\rw \gamma\in(0,+\infty)$, $\mu_{A_N}$ converges weakly almost surely to a deterministic probability measure $\mu$. Here $\mu$ is uniquely specified by a functional equation of its Stietjes transform $s(z)$:
    \bea\label{GMPeq}
    s_{\mu}(z)=\frac{1}{\frac{1}{\gamma}\int_{\mathbb{R}}\frac
    {t d\nu(t)}{1+ts_{\mu}(z)}-z},\quad \forall z\in\mathbb{C}^+.
    \eea
\end{prop}

Unfortunately, Proposition \ref{GMP} can {\it not} be directly applied to our case. This is because Proposition \ref{GMP} requires  $\Lambda_N$ and $X_N$ to be independent, while  $\Lambda_N$ and $X_N$ 
 in \eqref{eq_linear_model_ce_Hessian_rewrite} are clearly dependent.

\paragraph{Key observation: asymptotic independence} For our matrix of interests \eqref{eq_linear_model_ce_Hessian_rewrite}, although $\Lambda_N \in \mathbb{R}^{N\times N}$ and $X_N \in  \mathbb{R}^{d\times N}$ are not  independent for any fixed $d$, we observe that they are {\it asymptotically independent} as $d\rightarrow \infty$. This is because: 
\begin{itemize}
[topsep=1pt,parsep=1pt,partopsep=1pt, leftmargin=*]
    \item Recall $v_i \sim \mathcal{N}(0,\frac{1}{d})$ and denote  $z_n = v_i^\top x_n$, then $z_n |x_n \sim \mathcal{N}(0,\frac{\|x_n\|^2_2}{d})$. Further, since $x\sim  \mathcal{N}(0,1)$, by the law of large number, $\|x\|_2^2/d\rw1$ almost surely.
    \item As such, $z_n$ converges in distribution to $\mathcal{N}(0,1)$. Thus $z_n$ and $x_n$ are asymptotically independent. This suggests that the dependence between $\Lambda_N$ and $X_N$ may be weak as $d,N\rw\infty$. 
\end{itemize}

Therefore, as $d,N \rightarrow \infty$, it seems possible to obtain the same limiting eigenvalue distribution as in Proposition \ref{GMP} for our matrix \eqref{eq_linear_model_ce_Hessian_rewrite}.
To establish it in mathematics, one possible path is from free probability theory \citep{Mingobook}, proving the asymptotic freeness between $\Lambda_N$ and $X_N$ \citep{CollinsFreeness}. We will instead take a different path in this work.

\paragraph{Our solutions.} We use a rather classical decoupling technique, motivated by {\it the Lindeberg \zhaoruirevise{ interpolation} principle}. The  Lindeberg principle is originally an elegant proof for the Central Limit Theorem (CLT) \citep{Lindeberg1922}, by replacing the random variables with Gaussian ones incrementally and proving the impact is negligible under certain conditions. The Lindeberg principle is also applicable for random matrices \citep{Chatterjee2006,gotze2015asymptotic,pastur2020random}. We find that such methods \zhaoruirevise{are} useful for handling asymptotic independence in our case.

We now illustrate our proof strategy. The  idea is to first decouple and then apply Proposition \ref{GMP}.

\begin{itemize}
[topsep=1pt,parsep=1pt,partopsep=1pt, leftmargin=*]
    \item {\bf Step 1.} For our matrix \eqref{eq_linear_model_ce_Hessian_rewrite} (denoted as ${H}_{ii}^{\text{CE}}$),
we introduce the decoupled matrix
\be
\wt{H}_{ii}^{\text{CE}} = \frac{1}{N}\sum_{n= 1}^N  \wt{p}_{n,i} (1-\wt{p}_{n,i}) x_n x_n^\top,\quad \wt{p}_{n,i}:=\frac{\exp({v_i^\top \wt{x}_n})}{\sum_{c=1}^C\exp({v_c^\top \wt{x}_n})},
\ee
where $\wt{X}_N=(\wt{x}_1,\cdots,\wt{x}_N)\in\R^{d\times n}$ is an independent copy of $X_N$. The goal is to prove that 
$$
\lim_{N\rw\infty} \l(s_{{H}_{ii}^{\text{CE}}}(z)-s_{\wt{H}_{ii}^{\text{CE}}}(z)\r)=0,\quad \text{a.s.}\quad \forall z\in \mb{C}^+.
$$
From standard measure concentration results, $s_{{H}_{ii}^{\text{CE}}}(z),s_{\wt{H}_{ii}^{\text{CE}}}(z)$ concentrates around their means as $N\rw\infty$. Therefore, it suffices to prove that
$\lim_{N\rw\infty} \l(\E[s_{{H}_{ii}^{\text{CE}}}(z)]-\E[s_{\wt{H}_{ii}^{\text{CE}}}(z)]\r)=0$.
\item {\bf Step 2.} Following  the Lindeberg principle, we 
define the \zhaoruirevise{matrix interpolation process} 
$$
X_N(t)=\sqrt{t} X_N + \sqrt{1-t}\wt{X}_N,\quad t\in [0,1].
$$
\zhaoruirevise{Note that $X_N(1)=X_N$ and $X_N(0)=\wt{X}_N$.} We then define
$$
{H}_{ii}^{\text{CE}}(t) = \frac{1}{N}\sum_{n= 1}^N  {p}_{n,i}(t) (1-{p}_{n,i}(t)) x_n x_n^\top,\quad \text{ where }   {p}_{n,i}(t):=\frac{\exp({v_i^\top {x}_n(t)})}{\sum_{c=1}^C\exp({v_c^\top {x}_n(t)})}.
$$
Then

\be \label{eq_integrand}
\E[s_{{H}_{ii}^{\text{CE}}}(z)]-\E[s_{\wt{H}_{ii}^{\text{CE}}}(z)]=\int_0^1\E\l[\frac{d}{dt}s_{{H}_{ii}^{\text{CE}}(t)}\r]dt.
\ee
\item {\bf Step 3.} Then we calculate the integrand in \eqref{eq_integrand} using Stein's  Lemma:   for  $Z\sim \ml{N}(0,1)$ and differentiable function $f:\mb{R}\rw\mb{C}$ with sub-exponential decay at infinity, we have: 
\bea\label{SteinLemma_proof_sketch}
    \E[Zf(Z)]= \E[f^\prime(Z)].
    \eea
    We then prove that the r.h.s. of  \eqref{SteinLemma_proof_sketch} decays to zero at rate  $\mathcal{O}(1/\sqrt{N})$, and thus ${H}_{ii}^{\text{CE}}$ shares the same limiting eigenvalue distribution as $\wt{H}_{ii}^{\text{CE}}$. 
 Note that Stein's Lemma requires Gaussian conditions in Assumption \ref{assum_1}. We refer to Appendix A.6 of \citep{Talagrandbook} for the proof of Stein's Lemma.
\item {\bf Step 4.} Apply Proposition \ref{GMP}. The decoupled matrix $\wt{H}_{ii}^{\text{CE}}$ has the type $X_N\wt{\Lambda}_NX_N^\top$ where $\wt{\Lambda}_N$ is independent of $X_N$. Then Proposition \ref{GMP} is applicable to obtain the limiting eigenvalue distribution of $\wt{H}_{ii}^{\text{CE}}$. Then we apply the expansion \eqref{eq_stieltjes_power_series} in the functional equation \eqref{GMPeq} to get the limiting second moment of $\mu_{\wt{H}_{ii}^{\text{CE}}}$, which is also the limit of $\|{H}_{ii}^{\text{CE}}\|_\text{F}^2$. This concludes  the proof.
\end{itemize}

\paragraph{Challenges for the hidden-layer Hessian in 1-hidden-layer networks.} We extend our analysis of linear models to 1-hidden-layer networks. Similar as before, We primarily discuss the diagonal blocks of the Hessian.  The same challenges and solutions also apply to the off-diagonal blocks, which we omit here. We first discuss the hidden-layer Hessian. 
\begin{equation}
\label{eq_nn_ce_Hessian_rewrite_w}
  \frac{\partial^2\ell_{\text{CE}}(W,V)}{\partial w_i  \partial w_i^\top} = \frac{1}{N}\sum_{n= 1}^N  \underbrace{\left(\sum_{c=  1}^C p_{n,c} v_{c,i}^2 - \left(\sum_{c=1}^C p_{n,c} v_{c,i}\right)^2\right)}_{(a)}   \underbrace{\mathbf{1}(w_i^\top x_n >0)}_{(b)} x_n x_n^\top := \frac{1}{N} X_N \Theta_N X_N^\top, 
\end{equation}
Similar to Part c), we provide observations that may indicate the weak dependence between $X_N$ and $\Theta_N$ as $d,N\rw\infty$. 
We start with $(b)$ first. Recall $w_i \sim \mathcal{N}(0,\frac{1}{d})$ and denote  $z_{n,i} = w_i^\top x_n$. Following the same argument as in the linear model case, $z_{n,i}$ is asymptotically independent of $x_n$, hence $\mathbf{1} (z_{n,i}>0)$ is also \zhaoruirevise{asymptotically} independent of $x_n$ for all $n \in [N]$. Therefore,  $(b)$ and $x_n$ are asymptotically independent. For $(a)$, denote $y_{n,c} = (z_{n,1}, \cdots, z_{n,m})^\top v_c$. Since $(z_{n,1}, \cdots, z_{n,m})^\top$ is independent of $x_n$ as $d\rightarrow \infty$, so are $y_{n,c}$ and $p_{n,c}$. As such, $(a)$ and $x_n$ are asymptotically independent.  

Therefore, $(a)$ and $(b)$ are asymptotically independent of $x_N$, then the dependence between $X_N$ and $\Theta_N$ may be asymptotically weak.  Thus, it is reasonable to apply the same decoupling technique here. A similar procedure also applies to the hidden weights with MSE loss.

\paragraph{Challenges for the output-layer Hessian in 1-hidden-layer networks.} 
Now we discuss the different challenges for the output-layer Hessian. We rewrite  $\frac{\partial^2\ell_{\text{CE}}(W,V)}{\partial v_i  \partial v_i^\top}$ as follows. 
\begin{equation}
\label{eq_nn_ce_Hessian_rewrite_v}
  \frac{\partial^2\ell_{\text{CE}}(W,V)}{\partial v_i  \partial v_i^\top} = \frac{1}{N} \sum_{n= 1}^N  p_{n,i} (1-p_{n,i}) \sigma(Wx_n) \sigma(Wx_n)^\top  := \frac{1}{N} F_N \Lambda_N F_N^\top, 
\end{equation}
where $F_N = (\sigma(Wx_1), \cdots, \sigma(Wx_N)) \in \mathbb{R}^{m\times N}$. 
We highlight two major difference with the $X_N \Lambda_N X_N^\top$ in linear models \eqref{eq_linear_model_ce_Hessian_rewrite} and the hidden weights in 1-hidden-layer networks in \eqref{eq_nn_ce_Hessian_rewrite_w}.

\begin{itemize}
[topsep=1pt,parsep=1pt,partopsep=1pt, leftmargin=*]
\item {\bf First,} in \zhaoruirevise{the}  previous cases 
\eqref{eq_linear_model_ce_Hessian_rewrite} and \eqref{eq_nn_ce_Hessian_rewrite_w}, the matrices have growing dimension. Now the matrix in \eqref{eq_nn_ce_Hessian_rewrite_v} has fixed dimension $m$, which is away from the standard setting of random matrix theory.
\item {\bf Second,} the matrices $F_N$ and $\Lambda_N$ have more complicated dependence structure. The dependence between $Z_N$ and $\Lambda_N$ is caused by both $W$ and $X_N$, which means that we need to decouple $\Lambda_N$ with $\{W,X_N\}$ at the same time. Meanwhile, in  previous cases 
\eqref{eq_linear_model_ce_Hessian_rewrite} and \eqref{eq_nn_ce_Hessian_rewrite_w}, we only need to decouple $\Lambda_N$ with $X_N$.
\end{itemize}

To tackle the above challenges, we choose to handle \eqref{eq_nn_ce_Hessian_rewrite_v} using a largely different approach from above. We consider fixed $m$ and conduct the following steps.

\begin{itemize}
[topsep=1pt,parsep=1pt,partopsep=1pt, leftmargin=*]
\item {\bf Step 1.} Replace $WX_N$ with $Z_N$, where $Z_N\in \R^{m\times N}$ has i.i.d. $N(0,1)$ entries. This can be done by letting $d\rightarrow \infty$ and applying the Lindeberg principle. This step decouples $\Lambda_N$ with $H_N$.  

\item {\bf Step 2.} Calculate the expectation of the entry-wise second moment of the Hessian matrices. Note that the decoupling in Step 1 is essential, otherwise, the calculation in Step 2 would be complicated.  
\end{itemize}

\section{Proofs of the main theorems}

\subsection{Proof of Theorem \ref{thm_linear_model}}
\label{appendix_thm1}

\zhaoruirevise{Before delving into the proof, we first present the functions $g_{ii},\ g_{ij}$ in Theorem \ref{thm_linear_model}. Let $\mf{z}_C=(Z_1,\cdots,Z_C)\sim \ml{N}_C(0,I_C)$, define$$h_1(\mf{z}_C)=\frac{e^{Z_1}}{\sum_{l=1}^Ce^{Z_l}}\l(1-\frac{e^{Z_1}}{\sum_{l=1}^Ce^{Z_l}}\r),\quad h_2(\mf{z}_C)=\frac{e^{Z_1+Z_2}}{\l(\sum_{l=1}^Ce^{Z_l}\r)^2}, $$
then $$g_{ii}(\gamma,C):=\gamma\E[h_1(\mf{z}_C)^2]+(\E[h_1(\mf{z}_C)])^2,\quad g_{ij}(\gamma,C):=\gamma\E[h_2(\mf{z}_C)^2]+(\E[h_2(\mf{z}_C)])^2.$$}

We now introduce some notations that will be used in the proof. Let $\wt{X}_N=(\wt{x}_1,\cdots,\wt{x}_N) \in \mathbb{R}^{d\times N}$ be an independent copy of $X_N  \in \mathbb{R}^{d\times N}$. Denote
$$p_{i}(x) := \frac{\exp({v_i^\top x})}{\sum_{c=1}^C\exp({v_c^\top x})},\quad x\in \R^d,$$
$$
\alpha_n=p_i(x_n)(1-p_i(x_n)),\quad \Lambda_n=\diag(\alpha_1,\cdots,\alpha_N),
$$
$$
\wt{\alpha}_n=p_i(\wt{x}_n)(1-p_i(\wt{x}_n)),\quad \wt{\Lambda}_n=\diag(\wt{\alpha}_1,\cdots,\wt{\alpha}_N),
$$
$$
\beta_n=-p_i(x_n)p_j(x_n),\quad \Gamma_n=\diag(\beta_1,\cdots,\beta_N),
$$
$$
\wt{\beta}_n=-p_i(\wt{x}_n)p_j(\wt{x}_n),\quad \wt{\Gamma}_n=\diag(\wt{\beta}_1,\cdots,\wt{\beta}_N).
$$
To ease notations, we write $H_{ii}=\frac{\partial^2\ell_{\text{CE}}(V)}{\partial v_i  \partial v_i^\top} \in \mathbb{R}^{d\times d},\ H_{ij}=\frac{\partial^2\ell_{\text{CE}}(V)}{\partial v_i  \partial v_j^\top} \in \mathbb{R}^{d\times d}$, then
$$H_{ii}=\frac{1}{N}X_N\Lambda_NX_N^\top,\quad  H_{ij}=\frac{1}{N}X_N\Gamma_NX_N^\top.$$
Similarly, we define
$$
\wt{H}_{ii}=\frac{1}{N}X_N\wt{\Lambda}_NX_N^\top, \quad \wt{H}_{ij}=\frac{1}{N}X_N\wt{\Gamma}_NX_N^\top.
$$

Now we prove Theorem \ref{thm_linear_model}. The proof consists of two steps. First, we ``decouple" $X_N$, $\Lambda_N$ and $\Gamma_N$.  That is, we prove that ${H}_{ii}$ and $\wt{H}_{ii}$ share the same Stieltjes transform as $N$ and $d$ grow proportionally to infinity. Similarly for ${H}_{ij}$ and $\wt{H}_{ij}$. Second, with the help of Proposition \ref{GMP}, we find the second moments of limit eigenvalue distribution of  $\wt{H}_{ii}$ and $\wt{H}_{ij}$. Recall that the second moment is the Frobenius norm, so the proof is concluded. 

We now``decouple" $X_N$, $\Lambda_N$ and $\Gamma_N$ using the following Lemma \ref{decoupling}.

\begin{lem}\label{decoupling}
For any $z\in \mb{C}^+$, as $d,N\rw \infty,\ d/N\rw\gamma\in (0,+\infty)$, it holds almost surely that
\bea
s_{H_{ii}}(z)-s_{\wt{H}_{ii}}(z)=O\l(N^{-\frac{1}{2}}\r),
\eea
\bea
s_{H_{ij}}(z)-s_{\wt{H}_{ij}}(z)=O\l(N^{-\frac{1}{2}}\r).
\eea

\end{lem}

\begin{proof}
Here, we only present the proof for $s_{H_{ii}}$. The proof for $s_{H_{ij}}$ is done following the same procedure.
    
    For $t\in [0,1]$, let 
    $$
    X_N(t)=\sqrt{t}X_N+\sqrt{1-t}\wt{X}_N.
    $$
    Then $X_N(t)=(x_1(t),\cdots,x_N(t)) \in\mathbb{R}^{d\times N}$, where
    $$
    x_n(t)=\sqrt{t}x_n+\sqrt{1-t}\wt{x}_n,\quad n\in [N].
    $$
    Denote
    $$
    {\alpha}_n(t)=p_i\big(x_n(t)\big)\l[1-p_i\big(x_n(t)\big)\r],
    $$
    $$\Lambda_N(t)=\operatorname{diag}(\alpha_1(t),\cdots,\alpha_N(t)),$$
    $$
    H_{ii}(t)=\frac{1}{N}X_N\Lambda_N(t)X_N^\top,$$
    $$\ml{G}_N(z,t)=(H_{ii}(t)-zI_{d\times d})^{-1} \in \mathbb{R}^{d\times d},
    $$

    By the definition of $s_{H_{ii}}$, it is easy to see that $s_{H_{ii}} = \frac{1}{d}\operatorname{tr}\big(\ml{G}_N(z,t)\big)$.

       We first prove that $s_{H_{ii}}(z), s_{\wt{H}_{ii}}(z)$ concentrate around their mean. 
    By treating $s_{H_{ii}}(z), s_{\wt{H}_{ii}}(z)$ as Lipchitz functions of the Gaussian vectors $v_1,\cdots,v_C,x_1,\cdots,x_N$, we have the following results from Talagrand's inequality,
    $$
    \mf{P}\l(|s_{H_{ii}}(z)-\E[s_{{H}_{ii}}(z)]|\geq t\r)\leq c_1 e^{-pc_2t^2},
    $$
    $$
    \mf{P}\l(|s_{\wt{H}_{ii}}(z)-\E[s_{\wt{H}_{ii}}(z)]|\geq t\r)\leq \wt{c}_1 e^{-d\wt{c}_2t^2},
    $$
    where $t>0$ and constants $c_1,c_2,\wt{c}_1,\wt{c}_2>0$. Then from the Borel-Cantelli Lemma,
    \bea\label{tg13}
    s_{H_{ii}}(z)-\E[s_{{H}_{ii}}(z)]&\stackrel{a.s.}{\rw}0,\\
    s_{\wt{H}_{ii}}(z)-\E[s_{\wt{H}_{ii}}(z)]&\stackrel{a.s.}{\rw}0.
    \eea
    Now we prove that
    \bea\label{tg12}
    \delta_N(z)=\E[s_{H_{ii}}(z)]-\E[s_{\wt{H}_{ii}}(z)]=O\l(N^{-\frac{1}{2}}\r).
    \eea
    Recall for any function $A(t)$ valued in invertible matrices, we have
    \begin{equation}
    \label{eq_inverse_deriviative}
            \frac{d}{dt}A^{-1}(t)=-A^{-1}(t)\frac{d}{dt}A(t)A^{-1}(t).
    \end{equation}

    Then we have
    \bea\label{derivativeIntegral}
    \delta_N(z)& =\frac{1}{d}\E\l[\operatorname{tr}\big(\ml{G}_N(z,1)-\ml{G}_N(z,0)\big)\r]\\
    &=\frac{1}{d}\int_0^1 \frac{d}{dt}\E\l[\operatorname{tr}\left(\ml{G}_N(z,t)\right)\r]dt\\
    &=\frac{1}{d}\int_0^1 \frac{d}{dt}\E\l[\operatorname{tr}\left(\left(H_{ii}(t) -z I_{d\times d}\right)^{-1}\right)\r]dt\\
    &\overset{\eqref{eq_inverse_deriviative}}{=}-\frac{1}{d}\int_0^1 \E\l[\operatorname{tr}\l(\ml{G}_N(z,t)^2\frac{d}{dt}H_{ii}(t)\r)\r] dt\\
    &=-\frac{1}{dN}\int_0^1\E\l[\operatorname{tr}\l(X_N^\top\ml{G}_N(z,t)^2X_N\frac{d}{dt}\Lambda_N(t)\r)\r]dt, 
    \eea
    Define 
    $$
    \Delta_N(z)=-\frac{1}{dN}\sum_{n=1}^N \int_0^1 \E\l[\big(X_N^\top\ml{G}_N(z,t)X_N\big)_{nn}\frac{d}{dt}\alpha_n(t)\r]dt,
    $$
    then we have 
    
    $$\delta_N(z) \overset{\eqref{eq_inverse_deriviative}}{=}\frac{d}{dz}\Delta_N(z).$$

   \yushunrevise{We now bound $\delta_N(z)$ by bounding $\Delta_N(z)$. We 
    first define $D_{\zeta} = \left\{z | z\in \mb{C}^+, \Im z \geq \zeta >0 \right\}$, where $\Im z$ denotes the image part of $z$. 
    Since all eigenvalues of $H_{ii}(t) \in \mb{R}$, $\Delta_N(z)$ is an analytic function in  $D_{\zeta} $.  Based on  Cauchy's integral formula, for any $z  \in D_{\zeta} $ and arbitrary circle $\gamma \in D_{\zeta} $ containing $z$. we  have}
    $$
    \delta_N(z) = \frac{d}{dz}\Delta_N(z) = \frac{1}{2 \pi i} \oint_\gamma \frac{\Delta_N(s) }{(s-z)^{2}} d s.
    $$
    
    Then we have:
    
    $$\delta_N(z)  \leq \frac{\operatorname{length}(\gamma)}{2 \pi}  \max_{s\in \gamma} \frac{1}{(s-z)^2}  \max_{s\in \gamma} \left|\Delta_N(s)\right| \leq \operatorname{Const.} \max_{z\in D_\zeta } \left| \Delta_N(z) \right|,$$
    
    where $\operatorname{Const.} $ is some positive constant. Now we aim to prove the following equation:
    
    \be \label{tg1}
    \max_{z\in D_\zeta } \left| \Delta_N(z) \right| = O\l(N^{-\frac{1}{2}}\r)
    \ee    
    
    If this is true, then $|\delta_N(z)| = O(N^{-\frac{1}{2}})$ for all $z \in D_\zeta$. Let $\zeta \rightarrow 0$, we will get $|\delta_N(z)|  = O(N^{-\frac{1}{2}})$, which converges to 0 (pointwise) as $N \rightarrow \infty$.      In the following analysis, we aim to prove \eqref{tg1}. We first rewrite $\Delta_N(z)$ as follows.

    \be \label{eq_delta_z}
    \Delta_N(z)=-\frac{1}{2dN}\sum_{n=1}^N\sum_{l=1}^C\sum_{s=1}^d\int_0^1\E\l[x_n^\top \ml{G}_N(z,t)x_n B_{ln}(t)V_{ls}\l(\frac{X_{sn}}{\sqrt{t}}-\frac{\wt{X}_{sn}}{\sqrt{1-t}}\r)\r]dt,
    \ee
    where
    $$
    B_{ln}(t)=\Big(1-2p_i(x_n(t))\Big)p_i(x_n(t))\Big(\delta_{il}-p_l(x_n(t))\Big),
    $$
    and $X_{sn}$ is the abbreviation for the $(s,n)$-th entry in $X_N$. Similar abbreviation also applies to $\wt{X}_{sn}$. We define $\delta_{il} = 1$ if $i =l$ and $\delta_{il} = 0$ if otherwise. Note that trivial upper bound of \eqref{eq_delta_z} does not vanish with $N$.
    To better evaluate the expectation in \eqref{eq_delta_z}, we will use Stein's Lemma
    \bea\label{SteinLemma}
    \E[Zf(Z)]=\E[f'(Z)]
    \eea
    for $Z\sim \ml{N}(0,1)$ and differentiable function $f:\mb{R}\rw\mb{C}$ with sub-exponential decay at infinity.
    Then we have
    \bea
    \Delta_N(z)&=-\frac{1}{2dN}\sum_{n=1}^N\sum_{l=1}^C\sum_{s=1}^d\int_0^1\l(\frac{1}{\sqrt{t}}\E\l[\frac{\pl F_N(n,l,z,t)}{\pl X_{sn}}V_{ls}\r]-\frac{1}{\sqrt{1-t}}\E\l[\frac{\pl F_N(n,l,z,t)}{\pl \wt{X}_{sn}}V_{ls}\r]\r)dt,
    \eea
    $$
    F_N(n,l,z,t)=x_n^\top \ml{G}_N(z,t)x_n B_{ln}(t),
    $$
    Write $\ml{G}_N=\ml{G}_N(z,t)$ for short, then
        \bea
    &\frac{1}{\sqrt{1-t}}\E\l[\frac{\pl F_N(n,l,z,t)}{\pl \wt{X}_{sn}}V_{ls}\r] - \frac{1}{\sqrt{t}}\E\l[\frac{\pl F_N(n,l,z,t)}{\pl X_{sn}}V_{ls}\r]\\
    =&\frac{2}{\sqrt{t}}\E\l[\frac{\delta_{ns}}{N}B_{ln}(t)V_{ls}\alpha_n(t)(\ml{G}_Nx_n)_n(\ml{G}_Nx_n)^\top x_n - B_{ln}(t)V_{ls}(\ml{G}_Nx_n)_s\r].
    \eea
    Hence $\Delta_N(z)=\Delta_{N,1}(z)-\Delta_{N,2}(z)$, where
    \bea
    \Delta_{N,1}(z)&=\frac{1}{dN^2}\sum_{l=1}^C\sum_{n=1}^{\min(d,N)}\int_0^1\E\l[B_{ln}(t)V_{ln}\alpha_n(t)(\ml{G}_Nx_n)_n(\ml{G}_Nx_n)^\top x_n\r]\frac{dt}{\sqrt{t}},\\
    \Delta_{N,2}(z)&=\frac{1}{dN}\sum_{n=1}^N\sum_{l=1}^C\sum_{s=1}^d\int_0^1\E\l[B_{ln}(t)V_{ls}(\ml{G}_Nx_n)_s\r]\frac{dt}{\sqrt{t}}.
    \eea
    From Hölder's inequality,
    $$
    \E[|V_{ln}(\ml{G}_Nx_n)_n(\ml{G}_Nx_n)^\top x_n|]\leq \l(\E[V_{ln}^4]\r)^{\frac{1}{4}}\l(\E[(\ml{G}_Nx_n)_n^2]\r)^\frac{1}{2}\l(\E[((\ml{G}_Nx_n)^\top x_n)^4]\r)^\frac{1}{4},
    $$
    $$
    \E[|V_{ls}(\ml{G}_Nx_n)_s|]\leq \l(\E[V_{ls}^2]\r)^{\frac{1}{2}}\l(\E[(\ml{G}_Nx_n)_s^2]\r)^{\frac{1}{2}}.
    $$
   Recall that
$$|B_{ln}(t)|\leq \frac{1}{9},\ |\alpha_n(t)|\leq \frac{1}{4},\ V_{ls}\sim\ml{N}(0,\frac{1}{d}),$$
$$\|\ml{G}_N\|\leq \min_{z \in D_\zeta} \left|\frac{1}{\lambda_{H_{ii}} - z}\right| \leq \min_{z \in D_\zeta}  \frac{1}{\Im (z)}\leq \frac{1}{\zeta},$$

$$\E[(\ml{G}_Nx_n)_s^2]=\frac{1}{d}\E[\|\ml{G}_Nx_n\|^2] \leq \E[\|\ml{G}_N\|^2]  \leq  \frac{1}{\zeta},$$
we have
    $$
    |\Delta_{N,1}(z)|\leq \frac{3^{1/4}C\l(\E[\|x_n\|^8]\r)^{1/4}}{18N^2d^{1/2}\zeta^2},\quad |\Delta_{N,2}(z)|\leq \frac{2C}{9d^{1/2}\zeta}.
    $$
    As $d,N\rw\infty,\ d/N\rw \ga>0$, we have the following equations for any $z \in D_\zeta$
    $$
    |\Delta_{N,1}(z)|=O\l(N^{-3/2}\r),\quad |\Delta_{N,2}(z)|=O\l(N^{-1/2}\r).
    $$
    Set $\zeta \rightarrow 0$, then we have \eqref{tg1} and hence \eqref{tg12}, together with \eqref{tg13} implying
    $$
    s_{H_{ii}}(z)-s_{\wt{H}_{ii}}(z)=O\l(N^{-1/2}\r)\quad a.s.,\quad \forall z \in \mb{C}^+.
    $$
 The proof for $s_{H_{ij}}$ is done following the same procedure.
    
\end{proof}
Now we can apply Proposition \ref{GMP} to characterize the limiting eigenvalue distribution of $\widetilde{H}_{ii},\widetilde{H}_{ij}$, which are identical to the distributions of $H_{ii},H_{ij}$.
\begin{prop}\label{Linear_CE_MP}
    Fix $C\geq 2$, as $d,N\rw\infty, \frac
    {d}{N}\rw \gamma\in (0,+\infty)$, we have
    \begin{enumerate}
        \item $\mu_{H_{ii}}$  converges almost surely to a deterministic measure  $\mu^H_{11}$, and its Stieltjes transform  $s_{\mu^H_{11}}(z)$ is uniquely specified by the functional equation
        \bea\label{eqhii}
        s(z)=\frac{1}{\int_{\mb{R}^C} \frac{h_1(\mf{t})\varphi_C(\mf{t})}{1+\gamma s(z)h_1(\mf{t})}
 dt_1\cdots dt_C-z},\quad \forall z\in \mb{C}^+.
        \eea
        Here $\mf{t}=(t_1,\cdots,t_C)$, and 
        $$
        h_1(\mf{t})=\frac{e^{t_1}}{\sum_{l=1}^C e^{t_l}}\l(1-\frac{e^{t_1}}{\sum_{l=1}^C e^{t_l}}\r),
        $$
        $$
        \varphi_C(\mf{t})=(2\pi)^{-\frac{C}{2}}e^{-\frac{1}{2}\sum_{l=1}^Ct_l^2}.
        $$
        \item For $i\ne j$, $\mu_{H_{ij}}$ converges weakly almost surely to a deterministic measure $\mu_{12}$, and its Stieltjes transform  $s_{\mu_{12}}(z)$ is uniquely specified by the functional equation
        \bea\label{eqhij}
        s(z)=\frac{1}{\int_{\mb{R}^C} \frac{h_2(\mf{t})\varphi_C(\mf{t})}{1+\gamma s(z)h_2(\mf{t})}
 dt_1\cdots dt_C-z},\quad \forall z\in \mb{C}^+.
        \eea
        Here 
        $$
        h_2(\mf{t})=\frac{e^{t_1+t_2}}{\l(\sum_{l=1}^C e^{t_l}\r)^2}.
        $$
    \end{enumerate}
\end{prop}

\begin{proof}
    From Lemma \ref{decoupling}, it suffices to prove the convergence of  $\mu_{\wt{H}_{ii}},\mu_{\wt{H}_{ij}}$ to the limiting measure specified by \eqref{eqhii}, \eqref{eqhij} respectively. For the case of $\wt{H}_{ii}$, let 
$$
Y_{n,i}^{(N)}=p_i(\wt{x}_n)=\frac{\exp(v_{i}^\top \wt{x}_n)}{\sum_{l=1}^C\exp(v_l^\top \wt{x}_n)},
$$
then $\wt{\alpha}_n=Y_{n,i}^{(N)}(1-Y_{n,i}^{(N)})$. Recall that $\wt{H}_{ii}=\frac{1}{N}X_N\wt{\Lambda}_NX_N^\top,\ \wt{\Lambda}_N=\diag(\wt{\alpha}_1,\cdots,\wt{\alpha}_N)$. Then $\wt{\Lambda}_N$ is independent of $X_N$, and the eigenvalue distribution of $\wt{\Lambda}_N$ is the counting measure
$\nu_N=\frac{1}{N}\sum_{n=1}^N\delta_{\wt{\alpha}_n}$. Note that $(\wt{\alpha}_n)_{n =1}^N$ are identically distributed but dependent, therefore we need to prove that $\nu_N$ converges almost surely to a deterministic measure before applying Proposition \ref{GMP}.

From the strong law of large number, it holds almost surely that
\bea\label{prop1SLLN}
\lim_{d\rw\infty}\|w_c\|^2&=1,\quad\forall c\in[C],\\
\lim_{d\rw\infty}w_c^\top w_{c'}&=0,\quad\forall c,c'\in[C],\ c\ne c'.
\eea
Now we restrict the probability space to a subspace that \eqref{prop1SLLN} holds. Note that this restriction does not change validity of the proof since \eqref{prop1SLLN} holds almost surely. Let $\mathcal{F}_V$ be the $\sigma$-algebra generated by $\{V_{ij}\}_{i\in [C],j\in\N^+}$, then from CLT, the conditional distribution of 
$$
(v_1^\top \wt{x}_1,\cdots,v_C^\top\wt{x}_1)
$$
given $\mathcal{F}_V$ converges weakly to $\mathcal{N}_C(0,I_{C\times C})$.
Let $\nu$ be the deterministic probability measure such that for any interval $\Delta\subset\mb{R}$,
\bea\label{nudel}
\nu(\Delta)=\mf{P}\l(\frac{e^{Z_1}}{\sum_{l=1}^C e^{Z_l}}\l(1-\frac{e^{Z_1}}{\sum_{l=1}^C e^{Z_l}}\r)\in \Delta\r),
\eea
where $(Z_l)_{l=1}^C$ are i.i.d. $\ml{N}(0,1)$.
Let $f:\ \R\rw\R$ be a bounded continuous function, from the conditional independence of $(\wt{\alpha}_n)_{n=1}^N$ given $\ml{F}_V$, as $N\rw\infty$,
\bea
&\int_{\R}f(x)\nu_N(dx)-\int_\R f(x)\nu(dx)\\
=&\l(\frac{1}{N}\sum_{n=1}^N f(\wt{\alpha}_n)-\E\l[f(\wt{\alpha}_1)|\ml{F}_V\r]\r)+\l(\E\l[f(\wt{\alpha}_1)|\ml{F}_V\r]-\int_\R f(x)\nu(dx)\r)\\
\rw&0,\quad a.s..
\eea

Here the first term converges a.s to $0$ because of the strong law of large number. And the second term converges to $0$ from Portmanteau theorem. Then $\nu_N$ converges weakly almost surely to $\nu$.

Now let $\wt{T}_N=\frac{N}{d}\wt{H}_{ii}$. Then $\wt{T}_N=\frac{1}{d}X_N\wt{\Lambda}_NX_N^\top$, and the eigenvalue distribution of $\wt{\Lambda}_N$ converges weakly almost surely to $\nu$. From Proposition \ref{GMP}, $s_{\wt{T}_N}$ converges weakly almost surely to the unique solution of 
\bea\label{eqstn}
s(z)=\frac{1}{\frac{1}{\gamma}\int_{\mb{R}}\frac{t\nu(dt)}{1+ts(z)}-z},\quad \forall z\in \mb{C}^+.
\eea
From \eqref{nudel}, 
$$
\int_\mb{R} \frac{t\nu(dt)}{1+ts(z)}=\int_{\mb{R}^C} \frac{h_1(\mf{t})\varphi_C(\mf{t})}{1+ s(z)h_1(\mf{t})}
 dt_1\cdots dt_C.
$$
Note that $s_{\wt{T}_N}(z)=\frac{d}{N}s_{\wt{H}_{ii}}\Big(\frac{d}{N}z\Big)$ and $d/N\rw \gamma$. Then with a change of variable $z'=\gamma z$ in \eqref{eqstn}, it implies that $s_{\wt{H}_{ii}}$ converges weakly almost surely to the unique solution of \eqref{eqhii}. Then we finish the proof for the $H_{ii}$ case. 

The proof for $H_{ij}$ case is in the same procedure. \zhaoruirevise{The variables} $(\wt{\beta}_n)_{n=1}^N$ are identically distributed, and
$\wt{\beta}_n=-Y_{n,i}^{(N)}Y_{n,j}^{(N)}$. Define the counting measure $\eta_N=\frac{1}{N}\sum_{n=1}^N\delta_{\wt{\beta}_n}$, then $\eta_N$ converges weakly almost surely to a deterministic probability measure $\nu$, where for any interval $\Delta\subset\mb{R}$,
$$
\eta(\Delta)=\mf{P}\l(\frac{e^{Z_1}}{\sum_{l=1}^C e^{Z_l}}\frac{e^{Z_2}}{\sum_{l=1}^C e^{Z_l}}\in \Delta\r).
$$
Let $\wt{S}_N=\frac{N}{d}\wt{H}_{ij}$, then from Proposition \ref{GMP}, $s_{\wt{S}_N}$ converges weakly almost surely to the unique solution of 
$$
s(z)=\frac{1}{\frac{1}{\gamma}\int_{\mb{R}}\frac{t\eta(dt)}{1+ts(z)}-z},\quad \forall z\in \mb{C}^+.
$$
We have
\bea\label{eqtsn}
\int_\mb{R} \frac{t\eta(dt)}{1+ts(z)}=\int_{\mb{R}^C} \frac{h_2(\mf{t})\varphi_C(\mf{t})}{1+s(z)h_2(\mf{t})}
 dt_1\cdots dt_C.
\eea
Then from $s_{\wt{S}_N}(z)=\frac{d}{N}s_{\wt{H}_{ij}}\Big(\frac{d}{N}z\Big)$, $d/N\rw \gamma$, and a change of variable $z'=\gamma z$ in \eqref{eqtsn}, it implies that $s_{\wt{H}_{ij}}$ converges weakly almost surely to the unique solution of \eqref{eqhij}. This concludes the whole proof.
\end{proof}

The next proposition is to extract the second moment of the limiting eigenvalue distribution from the implicit equations \eqref{eqhii} and \eqref{eqhij}. \zhaoruirevise{This leads to Theorem \ref{thm_linear_model}.}

\begin{prop}
\label{thm3appendix}
    Suppose that as $d,N\rw\infty, \frac
    {d}{N}\rw \gamma\in (0,+\infty)$. Then for $i\ne j$, it holds almost surely that
        \bea\label{eq1thm3appendix}
        \lim_{d,N\rw \infty}\frac{\|H_{ii}\|_{\operatorname{F}}^2}{d} =\gamma\int_{\mb{R}^C}h_1(\mf{t})^2\varphi_C(\mf{t})dt_1\cdots dt_C + \l(\int_{\mb{R}^C}h_1(\mf{t})\varphi_C(\mf{t})dt_1\cdots dt_C\r)^2,
        \eea
        \bea\label{eq2thm3appendix}
        \lim_{d,N\rw \infty}\frac{\|H_{ij}\|_{\operatorname{F}}^2}{d} =\gamma\int_{\mb{R}^C}h_2(\mf{t})^2\varphi_C(\mf{t})dt_1\cdots dt_C + \l(\int_{\mb{R}^C}h_2(\mf{t})\varphi_C(\mf{t})dt_1\cdots dt_C\r)^2,
        \eea
        \bea\label{eq3thm3appendix}
\lim_{C\rw\infty}\lim_{d,N\rw\infty}\frac{C^2\|H_{ii}\|_{\operatorname{F}}^2}{d}=\gamma e+1,
        \eea
        \bea\label{eq4thm3appendix}
\lim_{C\rw\infty}\lim_{d,N\rw\infty}\frac{C^4\|H_{ij}\|_{\operatorname{F}}^2}{d}=\gamma e^2+1.
        \eea
\end{prop}
\begin{proof}
   Recall that as $z\rw \infty$ in $\mb{C}^+$, 
   $$
   s_{\mu}(z)=-\frac{1}{z}-\frac{1}{z^2}\int_{\mb{R}}x\mu(dx)-\frac{1}{z^3}\int_{\mb{R}}x^2\mu(dx)+O\l(\frac{1}{z^4}\r).
   $$
   Then in \eqref{eqhii} as $z\rw\infty$,
   \bea
   s_{\mu^H_{11}}(z)&=\frac{1}{\int_{\mb{R}^C} \frac{h_1(\mf{t})\varphi_C(\mf{t})}{1+\gamma s_{\mu^H_{11}}(z)h_1(\mf{t})}dt_1\cdots dt_C-z}\\
   &=\frac{1}{\int_{\mb{R}^C} \frac{h_1(\mf{t})\varphi_C(\mf{t})}{1-\gamma h_1(\mf{t})z^{-1}+O(z^{-2})}dt_1\cdots dt_C-z}\\
   &=\frac{1}{\int_{\mb{R}^C}h_1(\mf{t})\varphi_C(\mf{t})\Big(1+\gamma h_1(\mf{t})z^{-1}+O(z^{-2})\Big)dt_1\cdots dt_C-z}\\
   &=-\frac{1}{z}-\frac{1}{z^2}\int_{\mb{R}^C}h_1(\mf{t})\varphi_C(\mf{t})dt_1\cdots dt_C\\
   &\quad -\frac{1}{z^3}\l[\gamma\int_{\mb{R}^C}h_1(\mf{t})^2\varphi_C(\mf{t})dt_1\cdots dt_C + \l(\int_{\mb{R}^C}h_1(\mf{t})\varphi_C(\mf{t})dt_1\cdots dt_C\r)^2\r]+O\l(\frac{1}{z^4}\r).
   \eea
   Hence
   $$
   \int_{\mb{R}}x^2\mu^H_{11}(dx)=\gamma\int_{\mb{R}^C}h_1(\mf{t})^2\varphi_C(\mf{t})dt_1\cdots dt_C + \l(\int_{\mb{R}^C}h_1(\mf{t})\varphi_C(\mf{t})dt_1\cdots dt_C\r)^2.
   $$
   Then we obtain \eqref{eq1thm3} since from Proposition \ref{Linear_CE_MP}, 
   $$
   \lim_{d,N\rw\infty}\frac{\|H_{ii}\|_{\operatorname{F}}^2}{d}=\lim_{d,N\rw\infty}\int_{\mb{R}}x^2\mu_{H_{ii}}(dx)=\int_{\mb{R}}x^2\mu^H_{11}(dx)\quad a.s.
   $$
   The proof of \eqref{eq2thm3} \zhaoruirevise{follows the} same procedure, i.e., expanding $s(z)$ as $z\rw\infty$ in \eqref{eqhij}.
   Now let $\mf{q}_C=(q_1,\cdots,q_C)\sim\ml N_C(0,I_{C\times C})$. From the strong law of large number, as $C\rw\infty$,
   $$
\frac{e^{q_1}+\cdots+e^{q_C}}{C}\stackrel{a.s.}{\rw}\E[e^{q_1}]=\sqrt{e}.
   $$
   Then from Slutsky's theorem,
   $$
C\sqrt{e}\cdot h_1(\mf{q}_C)=\frac{C\sqrt{e}\cdot e^{q_1}}{e^{q_1}+\cdots+e^{q_C}}\l(1-\frac{e^{q_1}}{e^{q_1}+\cdots+e^{q_C}}\r)\stackrel{d}{\rw}\operatorname{Lognormal(0,1)},
   $$
   $$
   C^2e\cdot h_2(\mf{q}_C)=\frac{C\sqrt{e}\cdot e^{q_1}}{e^{q_1}+\cdots+e^{q_C}}\cdot \frac{C\sqrt{e}\cdot e^{q_2}}{e^{q_1}+\cdots+e^{q_C}}\stackrel{d}{\rw}\operatorname{Lognormal(0,1)}\otimes\operatorname{Lognormal(0,1)}.
   $$
 Here $\otimes$ denotes the multiplicative convolution. That is, $\operatorname{Lognormal(0,1)}\otimes\operatorname{Lognormal(0,1)}$ is the distribution of $\xi_1\xi_2$ where $\xi_1,\xi_2$ are iid $\operatorname{Lognormal(0,1)}$. Then from $\E[\xi]=\sqrt{e},\ \E[\xi^2]=e^2$, we have
 $$
\lim_{C\rw\infty}C\int_{\mb{R}^C}h_1(\mf{t})\varphi_C(\mf{t})dt_1\cdots dt_C=\lim_{C\rw\infty}\E\l[Ch_1(\mf{q}_C)\r]=1,
 $$
 $$
\lim_{C\rw\infty}C^2\int_{\mb{R}^C}h_1(\mf{t})^2\varphi_C(\mf{t})dt_1\cdots dt_C=\lim_{C\rw\infty}\E\l[(Ch_1(\mf{q}_C))^2\r]=e,
 $$
 $$
\lim_{C\rw\infty}C^2\int_{\mb{R}^C}h_2(\mf{t})\varphi_C(\mf{t})dt_1\cdots dt_C=\lim_{C\rw\infty}\E\l[C^2h_2(\mf{q}_C)\r]=1,
 $$
 $$
\lim_{C\rw\infty}C^4\int_{\mb{R}^C}h_2(\mf{t})^2\varphi_C(\mf{t})dt_1\cdots dt_C=\lim_{C\rw\infty}\E\l[(C^2h_2(\mf{q}_C))^2\r]=e^2.
 $$
 Then, \eqref{eq3thm3} and \eqref{eq4thm3} follows from \eqref{eq1thm3} and \eqref{eq2thm3} by taking $C\rw\infty$, \zhaoruirevise{respectively}.
\end{proof}

\subsection{Proof of Theorem \ref{thm_nn}}
\label{appendix_thm2}

Before delving into the proof, we first present the constant terms in Theorem \ref{thm_nn}. We remark that the integrals are finite when $m \geq 2$.

\bea
\label{eq_constant_terms_Hvv}
a_{11}&=\frac{1}{2\pi}\int_0^{+\infty}\int_0^{+\infty}xy\exp\l(\frac{xy}{m}-\frac{x^2}{2}-\frac{y^2}{2}\r)dxdy,\\
a_{12}&=\frac{1}{2\pi}\int_0^{+\infty}\int_0^{+\infty}x^2y^2\exp\l(\frac{xy}{m}-\frac{x^2}{2}-\frac{y^2}{2}\r)dxdy,\\
a_{21}&=\frac{1}{2\pi}\int_0^{+\infty}\int_0^{+\infty}xy\exp\l(\frac{2xy}{m}-\frac{x^2}{2}-\frac{y^2}{2}\r)dxdy,\\
a_{22}&=\frac{1}{2\pi}\int_0^{+\infty}\int_0^{+\infty}x^2y^2\exp\l(\frac{2xy}{m}-\frac{x^2}{2}-\frac{y^2}{2}\r)dxdy,\\
b_1&=\frac{3}{4}+\frac{1}{8\pi}\int_0^{+\infty}\int_0^{+\infty}\exp\l(\frac{xy}{m}-\frac{x^2}{2}-\frac{y^2}{2}\r)dxdy,\\
b_2&=\frac{3}{4}+\frac{1}{8\pi}\int_0^{+\infty}\int_0^{+\infty}\exp\l(\frac{2xy}{m}-\frac{x^2}{2}-\frac{y^2}{2}\r)dxdy.
\eea

\zhaoruirevise{The functions $h_{ii},\ h_{ij},\ u_{ii},\ u_{ij},\ q_{ii},\ q_{ij}$ are given by the right hand side of \eqref{eqRHShii},  \eqref{eqRHShij}, \eqref{eqRHSuii},  \eqref{eqRHSuij}, \eqref{eqRHSqii}, \eqref{eqRHSqij}, repectively.}

\subsubsection{Proof for the Hidden-layer Hessian with CE Loss}
\label{appendix_hidden_weights_ce}

Our proof uses the following strategy. Firstly we consider the case that $V\in\R^{C\times m}$ is deterministic. In this case, the techniques in the proof of Theorem \ref{thm_linear_model} is available. Then for the target case that entries of $V$ are i.i.d. Gaussian, we use
$$
\E\l[\bigg\|\frac{\partial^2\ell_{\text{CE}}(W,V)}{\partial w_i  \partial w_j^\top}\bigg\|_\text{F}^2\r]=\E\l[\E\l[\bigg\|\frac{\partial^2\ell_{\text{CE}}(W,V)}{\partial w_i  \partial w_j^\top}\bigg\|_\text{F}^2\ \bigg|\ V\r]\r]
$$
and apply the deterministic case results in the conditional expectation.

Now for short of notations we write
$G_{ii}=\frac{\partial^2\ell_{\text{CE}}(W,V)}{\partial w_i  \partial w_i^\top},\ G_{ij}=\frac{\partial^2\ell_{\text{CE}}(W,V)}{\partial w_i  \partial w_j^\top}$. Then
\bea
G_{ii}&=\frac{1}{N}\sum_{n=1}^N \mf{1}(w_i^\top x_n>0)\l[\sum_{k=1}^C q_k(x_n)V_{ki}^2-\l(\sum_{k=1}^Cq_k(x_n)V_{ki}\r)^2\r]x_nx_n^\top,\\
G_{ij}&=\frac{1}{N}\sum_{n=1}^N \mf{1}(w_i^\top x_n>0)\mf{1}(w_j^\top x_n>0)\l[\sum_{k=1}^C q_k(x_n)V_{ki}V_{kj}-\l(\sum_{k=1}^Cq_k(x_n)V_{ki}\r)\l(\sum_{k=1}^Cq_k(x_n)V_{kj}\r)\r]x_nx_n^\top,
\eea
where
$$
q_k(x_n)=\frac{\exp(\sigma(Wx_n)^\top v_k)}{\sum_{l=1}^C\exp(\sigma(Wx_n)^\top v_l)}.
$$

Let $\wt{X}_N=(\wt{x}_1,\cdots,\wt{x}_N) \in \mathbb{R}^{d\times N}$ be an independent copy of $X_N$. Define

\bea
\wt{G}_{ii}&=\frac{1}{N}\sum_{n=1}^N \mf{1}(w_i^\top \wt{x}_n>0)\l[\sum_{k=1}^C q_k(\wt{x}_n)V_{ki}^2-\l(\sum_{k=1}^Cq_k(\wt{x}_n)V_{ki}\r)^2\r]x_nx_n^\top,\\
\wt{G}_{ij}&=\frac{1}{N}\sum_{n=1}^N \mf{1}(w_i^\top \wt{x}_n>0)\mf{1}(w_j^\top \wt{x}_n>0)\l[\sum_{k=1}^C q_k(\wt{x}_n)V_{ki}V_{kj}-\l(\sum_{k=1}^Cq_k(\wt{x}_n)V_{ki}\r)\l(\sum_{k=1}^Cq_k(\wt{x}_n)V_{kj}\r)\r]x_nx_n^\top.
\eea

\begin{lem}\label{decoupling_NN_CE_Hww}
Suppose that $C\geq 2,\ m\geq 3$ are fixed, $V\in\R^{C\times m}$ is deterministic. For any $z\in \mb{C}^+$, as $d,N\rw \infty,\ d/N\rw\gamma\in (0,+\infty)$, it holds almost surely that
\bea\label{decoupling_NN_CE_Hww_eq1}
s_{G_{ii}}(z)-s_{\wt{G}_{ii}}(z)=O\l(N^{-\frac{1}{2}}\r),
\eea
\bea
s_{G_{ij}}(z)-s_{\wt{G}_{ij}}(z)=O\l(N^{-\frac{1}{2}}\r).
\eea
\end{lem}

\begin{proof}
    Here, we only present the proof for $s_{G_{ii}}$. The proof for $s_{G_{ij}}$ is done following the same procedure. 
    
    For $t\in [0,1]$, let 
    $$
    X_N(t)=\sqrt{t}X_N+\sqrt{1-t}\wt{X}_N.
    $$
    Then $X_N(t)=(x_1(t),\cdots,x_N(t)) \in \mathbb{R}^{d\times N}$, where
    $$
    x_n(t)=\sqrt{t}x_n+\sqrt{1-t}\wt{x}_n,\quad n\in [N].
    $$
        Denote
    $$
    {\theta}_n(t)=\mf{1}(w_i^\top x_n(t)>0)\l[\sum_{k=1}^C q_k(x_n(t))V_{ki}^2-\l(\sum_{k=1}^Cq_k(x_n(t))V_{ki}\r)^2\r],
    $$

    $$
    \Theta_N(t)=\operatorname{diag}(\theta_1(t),\cdots,\theta_N(t)),
    $$

    $$
    G_{ii}(t)=\frac{1}{N}X_N\Theta_N(t)X_N^\top,
    $$
    
    $$
    \ml{G}_N(z,t)=(G_{ii}(t)-z)^{-1}.
    $$
    
    By treating $s_{H_{ii}}(z),s_{\wt{H}_{ii}}(z)$ as Lipchitz functions of the Gaussian vectors $w_1,\cdots,w_m,x_1,\cdots,x_N$, from Talagrand's inequality,
    $$
    \mf{P}\l(|s_{G_{ii}}(z)-\E[s_{{G}_{ii}}(z)]|\geq t\r)\leq c_1 e^{-pc_2t^2},
    $$
    $$
    \mf{P}\l(|s_{\wt{G}_{ii}}(z)-\E[s_{\wt{G}_{ii}}(z)]|\geq t\r)\leq \wt{c}_1 e^{-d\wt{c}_2t^2},
    $$
    for $t>0$ and constants $c_1,c_2,\wt{c}_1,\wt{c}_2>0$. Then from the Borel-Cantelli Lemma,
    \bea
    s_{G_{ii}}(z)-\E[s_{{G}_{ii}}(z)]&\stackrel{a.s.}{\rw}0,\\
    s_{\wt{G}_{ii}}(z)-\E[s_{\wt{G}_{ii}}(z)]&\stackrel{a.s.}{\rw}0.
    \eea
    Now we prove that
    \bea
    \delta_N(z)=\E[s_{H_{ii}}(z)]-\E[s_{\wt{H}_{ii}}(z)]=O\l(N^{-\frac{1}{2}}\r).
    \eea
    Following \eqref{derivativeIntegral}, we have
    \bea
    \Delta_N(z)=-\frac{1}{dN}\int_0^1\E\l[\operatorname{tr}\l(X_N^\top\ml{G}_N(z,t)^2X_N\frac{d}{dt}\Theta_N(t)\r)\r]dt.
    \eea
    Then $\delta_N(z)=\frac{d}{dz}\Delta_N(z)$, where
    \bea
    \Delta_N(z)&=-\frac{1}{dN}\sum_{n=1}^N \int_0^1 \E\l[\big(X_N^\top\ml{G}_N(z,t)X_N\big)_{nn}\frac{d}{dt}\theta_n(t)\r]dt.
    \eea
    Similarly to the proof of Theorem \ref{thm_linear_model}, it suffices to prove that for any open set $O\subset\mb{C}^+$ such that $\zeta=\inf_{z\in O} |\Im z|>0$,
    \be\label{tg2prop3}
    \max_{z\in O} |\Delta_N^{(1)}(z)|=O\l(N^{-\frac{1}{2}}\r).
    \ee
    We have
    $$
    \frac{d}{dt}q_k(x_n(t))=\frac{1}{2}\sum_{l=1}^C\sum_{h=1}^m q_k(x_n(t))[\delta_{kl}-q_l(x_n(t))]\mf{1}(w_h^\top x_n(t)>0)V_{lh}w_h^\top \l(\frac{x_n}{\sqrt{t}}-\frac{\wt{x}_n}{\sqrt{1-t}}\r),
    $$
    $$
    \frac{d}{dt}\theta_n(t)=\mf{1}(w_i^\top x_n(t)>0)\sum_{k=1}^C\l(V_{ki}^2-2V_{ki}\sum_{k'=1}^Cq_{k'}(x_n(t))V_{k'i}\r)\frac{d}{dt}q_k(x_n(t)).
    $$
    Therefore
    \bea
    \Delta_N(z)=-\frac{1}{2dN}\sum_{n=1}^N\sum_{k=1}^C\sum_{l=1}^C\sum_{h=1}^m\sum_{s=1}^d\int_0^1\E\l[x_n^\top \ml{G}_N(z,t)x_n Q_{klhn}(t)W_{hs}\l(\frac{X_{sn}}{\sqrt{t}}-\frac{\wt{X}_{sn}}{\sqrt{1-t}}\r)\r]dt,
    \eea
    where
    $$
    Q_{klhn}(t)=\l(V_{ki}^2-2V_{ki}\sum_{k'=1}^Cq_{k'}(x_n(t))V_{k'i}\r)q_k(x_n(t))[\delta_{kl}-q_l(x_n(t))]\mf{1}(w_h^\top x_n(t)>0)V_{lh}.
    $$
    From Stein's Lemma \eqref{SteinLemma},
    \bea
    \Delta_N(z)&=-\frac{1}{2dN}\sum_{n=1}^N\sum_{k=1}^C\sum_{l=1}^C\sum_{h=1}^m\sum_{s=1}^d\int_0^1\l(\frac{1}{\sqrt{t}}\E\l[\frac{\pl A^{(k,l,h,n)}_N(z,t)}{\pl X_{sn}}W_{hs}\r]-\frac{1}{\sqrt{1-t}}\E\l[\frac{\pl A^{(k,l,h,n)}_N(z,t)}{\pl \wt{X}_{sn}}W_{hs}\r]\r)dt,
    \eea
    where
    $$
    A^{(k,l,h,n)}_N(z,t)=x_n^\top \ml{G}_N(z,t)x_n Q_{klhn}(t).
    $$
    Then
    \bea
    &\frac{1}{\sqrt{t}}\E\l[\frac{\pl A^{(k,l,h,n)}_N(z,t)}{\pl X_{sn}}W_{hs}\r]-\frac{1}{\sqrt{1-t}}\E\l[\frac{\pl A^{(k,l,h,n)}_N(z,t)}{\pl \wt{X}_{sn}}W_{hs}\r]\\
    =&\frac{2}{\sqrt{t}}\E\l[Q_{klhn}(t)W_{hs}\l((\ml{G}_Nx_n)_s-\frac{\delta_{ns}}{N}\theta_n(t)(\ml{G}_Nx_n)_n(\ml{G}_Nx_n)^\top x_n\r)\r].
    \eea
    Hence $\Delta_N(z)=\Delta_{N,1}(z)-\Delta_{N,2}(z)$, where
    \bea
    \Delta_{N,1}(z)&=\frac{1}{dN^2}\sum_{k=1}^C\sum_{l=1}^C\sum_{h=1}^m\sum_{n=1}^{\min(d,N)}\int_0^1\E\l[Q_{klhn}(t)\theta_n(t)W_{hn}(\ml{G}_Nx_n)_n(\ml{G}_Nx_n)^\top x_n\r]\frac{dt}{\sqrt{t}},\\
    \Delta_{N,2}(z)&=\frac{1}{dN}\sum_{n=1}^N\sum_{k=1}^C\sum_{l=1}^C\sum_{h=1}^m\sum_{s=1}^d\int_0^1\E\l[Q_{klhn}(t)W_{hs}(\ml{G}_Nx_n)_s\r]\frac{dt}{\sqrt{t}}.
    \eea
    Let $M_V=\max_{k\in [C],i\in [m]}|V_{ki}|$. From Hölder's inequality,
    $$
    \E[|W_{hn}(\ml{G}_Nx_n)_n(\ml{G}_Nx_n)^\top x_n|]\leq \l(\E[W_{hn}^4]\r)^{\frac{1}{4}}\l(\E[(\ml{G}_Nx_n)_n^2]\r)^\frac{1}{2}\l(\E[((\ml{G}_Nx_n)^\top x_n)^4]\r)^\frac{1}{4},
    $$
    $$
    \E[|W_{hs}(\ml{G}_Nx_n)_s|]\leq \l(\E[W_{hs}^2]\r)^{\frac{1}{2}}\l(\E[(\ml{G}_Nx_n)_s^2]\r)^{\frac{1}{2}},
    $$
    together with 
    $$|Q_{klhn}(t)|\leq 3M_v^4,\ |\theta_n(t)|\leq 2M_v^2,\ \|\ml{G}_N\|\leq \frac{1}{\zeta},\ W_{hs}\sim N(0,\frac{1}{d}),\ \E[(\ml{G}_Nx_n)_s^2]=\frac{1}{d}\E[\|\ml{G}_Nx_n\|^2]\leq \frac{1}{\zeta},$$
    we have
    $$
    |\Delta_{n,1}(z)|\leq \frac{3^{\frac{1}{4}}12C^2mM_V^5\l(\E[\|x_n\|^8]\r)^\frac{1}{4}}{\zeta^2d^\frac{3}{2}N},\quad |\Delta_{n,2}(z)|\leq \frac{6mC^2M_V^3}{\zeta d^{\frac{1}{2}}}.
    $$
    Then as $d,N\rw\infty,\ d/N\rw \gamma\in (0,+\infty)$,
    $$
    |\Delta_{N,1}(z)|=O\l(N^{-\frac{3}{2}}\r),\quad |\Delta_{N,1}(z)|=O\l(N^{-\frac{1}{2}}\r).
    $$
    Then we have \eqref{tg2prop3} and then \eqref{decoupling_NN_CE_Hww_eq1}. The proof in the case of $G_{ij}$ is similar and is omitted.
\end{proof}

\begin{prop}\label{MP_NN_CE_Hww}
    Suppose that $C\geq 2,\ m\geq 3$ are fixed, $V\in\R^{C\times m}$ is deterministic. As $d,N\rw\infty, \frac
    {d}{N}\rw \gamma\in (0,+\infty)$, we have
    \begin{enumerate}
        \item $\mu_{G_{ii}}$ converges weakly almost surely to a deterministic measure $\mu^G_{11}$, where $s_{\mu^G_{11}}(z)$ is uniquely specified by the functional equation
        \bea\label{MP_NN_CE_Hwiwi}
        s(z)=\frac{1}{\int_{\mb{R}} \frac{t\nu_{1}(dt)}{1+\gamma s(z)t}
 dt -z},\quad \forall z\in \mb{C}^+. 
\eea
Here $\nu_1$ is defined as follows. Let 
$\mf{z}=(z_1,\cdots,z_m)\sim \mathcal{N}_m(0,I_m)$, define random variables
$$
r_k=\frac{\exp(\sigma(\mf{z})^\top v_k)}{\sum_{l=1}^C\exp(\sigma(\mf{z})^\top v_l)},\quad k\in [C],
$$
$$
\xi_C(V)=\mf{1}(z_1>0)\l[\sum_{k=1}^C r_kV_{k1}^2-\l(\sum_{k=1}^C r_k V_{k1}\r)^2\r].
$$
Then $\nu_1$ is given by that for all intervals $\Delta\subset \mb{R}$, $\nu_1(\Delta)=\mf{P}(\xi_C(V)\in \Delta)$.

        \item For $i\ne j$, $\mu_{G_{ij}}$ converges weakly almost surely to a deterministic measure $\mu^G_{12}$, where $s_{\mu^G_{12}}(z)$ is uniquely specified by the functional equation
\bea\label{MP_NN_CE_Hwiwj}
        s(z)=\frac{1}{\int_{\mb{R}} \frac{t\nu_{2}(dt)}{1+\gamma s(z)t}
 dt -z},\quad \forall z\in \mb{C}^+. 
\eea
Here $\nu_2$ is defined as follows. Let 
$\mf{z}=(z_1,\cdots,z_m)\sim \mathcal{N}_m(0,I_m)$, define random variables
$$
\eta_C(V)=\mf{1}(z_1>0)\mf{1}(z_2>0)\l[\sum_{k=1}^C r_kV_{k1}V_{k2}-\l(\sum_{k=1}^C r_k V_{k1}\r)\l(\sum_{k=1}^C r_k V_{k2}\r)\r].
$$
Then $\nu_2$ is given by that for all intervals $\Delta\subset \mb{R}$, $\nu_2(\Delta)=\mf{P}(\eta_C(V)\in \Delta)$.
    \end{enumerate}
\end{prop}

\begin{proof}
     We give a proof for the $G_{ii}$ case, the $G_{ij}$ case is in the same procedure. From Lemma \ref{decoupling_NN_CE_Hww}, it suffices to prove the convergence of  $\mu_{\wt{H}_{ii}}$ to the limiting measure specified by \eqref{MP_NN_CE_Hww}. Let 
$$
\wt{\theta}_n=\mf{1}(w_i^\top \wt{x}_n>0)\l[\sum_{k=1}^C q_k(\wt{x}_n)V_{ki}^2-\l(\sum_{k=1}^Cq_k(\wt{x}_n)V_{ki}\r)^2\r],
$$
$$
\wt{\Theta}_N=\diag\{\theta_1,\cdots,\theta_N\},
$$
then $\wt{H}_{ii}=\frac{1}{N}X_N\wt{\Theta}_NX_N^\top$, and $\wt{\Theta}_N$ is independent of $X_N$. The eigenvalue distribution of $\wt{\Theta}_N$ is the counting measure
$\tau_N=\frac{1}{N}\sum_{n=1}^N\delta_{\wt{\theta}_n}$. Since $m$ is fixed, from the strong law of large number, it holds almost surely that
\bea\label{prop1SLLN1}
\lim_{d\rw\infty}\|w_h\|^2&=1,\quad\forall h\in[m],\\
\lim_{d\rw\infty}w_h^\top w_{h'}&=0,\quad\forall h,h'\in[m],\ h\ne h'.
\eea
Without loss of generality we can restrict the probability space to a subspace that \eqref{prop1SLLN1} holds. Let $\mathcal{F}_W$ be the $\sigma$-algebra generated by $\{W_{hs}\}_{h\in [m],s\in\N^+}$, then from CLT, the conditional distribution of 
$$
(w_1^\top \wt{x}_1,\cdots,w_m^\top\wt{x}_1)
$$
given $\mathcal{F}_W$ converges weakly to $\mathcal{N}_m(0,I_m)$.
Since $V$ is deterministic, $(\wt{\theta}_n)_{n=1}^N$ are independent conditioning on $\ml{F}_W$.
Then for any continuous bounded function $f:\R\rw\R$,
\bea
&\int_{\R}f(x)\tau_N(dx)-\int_\R f(x)\nu_1(dx)\\
=&\l(\frac{1}{N}\sum_{n=1}^N f(\wt{\theta}_n)-\E\l[f(\wt{\theta}_1)|\ml{F}_W\r]\r)+\l(\E\l[f(\wt{\theta}_1)|\ml{F}_W\r]-\int_\R f(x)\nu_1(dx)\r)\\
\rw&0,\quad a.s.
\eea

Therefore $\tau_N$ converges weakly almost surely to $\nu_1$.

Now let $\wt{T}_N=\frac{N}{d}\wt{G}_{ii}$. Then $\wt{T}_N=\frac{1}{d}X_N\wt{\Theta}_NX_N^\top$, and the eigenvalue distribution of $\wt{\Theta}_N$ converges weakly almost surely to $\nu_1$. From Proposition \ref{GMP}, $s_{\wt{T}_N}$ converges weakly almost surely to the unique solution of 
\bea\label{eqstn1}
s(z)=\frac{1}{\frac{1}{\gamma}\int_{\mb{R}}\frac{t\nu_1(dt)}{1+ts(z)}-z},\quad \forall z\in \mb{C}^+.
\eea
Then with a change of variable $z'=\gamma z$ in \eqref{eqstn1}, it implies that $s_{\wt{G}_{ii}}$ converges weakly almost surely to the unique solution of \eqref{MP_NN_CE_Hwiwi}.
\end{proof}

The next proposition is exactly \eqref{thm2_NN_CE_Hww} in Theorem \ref{thm_nn}.
\begin{prop}
    Suppose that $m\geq 3$ is fixed, and $V\in\R^{C\times m}$ has i.i.d. $\ml{N}(0,\frac{1}{m})$ entries. If as $d,N\rw\infty, \frac
    {d}{N}\rw \gamma\in (0,+\infty)$, then for $i\ne j$, 
    \bea
\lim_{C\rw\infty}\lim_{d,N\rw\infty}\frac{\E\l[\|G_{ii}\|_{\operatorname{F}}^2\r]}{d}&=\frac{2\gamma+1}{4m^2},\\
\lim_{C\rw\infty}\lim_{d,N\rw\infty}\frac{C\E\l[\|G_{ij}\|_{\operatorname{F}}^2\r]}{d}&=\frac{\gamma(m-1)^2}{2^m(m-2)^3m}\l(\sqrt{\frac{m}{m-2}}+1\r)^{m-2}.
        \eea
\end{prop}
\begin{proof}
Let $\overline{V}$ be a deterministic realization of $V$. By expanding \eqref{MP_NN_CE_Hwiwi} and \eqref{MP_NN_CE_Hwiwj} at $z=\infty$, we have almost surely 
\bea
\lim_{d,N\rw\infty}\frac{\|G_{ii}\|_{\operatorname{F}}^2}{d}=\int_\R x^2\mu_{11}^G(dx)=\gamma\E[\xi_C(\overline{V})^2]+(\E[\xi_C(\overline{V})])^2,
\eea
\bea
\lim_{d,N\rw\infty}\frac{\|G_{ij}\|_{\operatorname{F}}^2}{d}=\int_\R x^2\mu_{12}^G(dx)=\gamma\E[\eta_C(\overline{V})^2]+(\E[\eta_C(\overline{V})])^2.
\eea

Then we have
\bea
\lim_{d,N\rw\infty}\frac{\E\l[\|G_{ii}\|_{\operatorname{F}}^2\ |\ V\r]}{d}=\gamma\E[\xi_C(V)^2|V]+(\E[\xi_C(V)|V])^2,
\eea
\bea
\lim_{d,N\rw\infty}\frac{\E\l[\|G_{ij}\|_{\operatorname{F}}^2\ |\ V\r]}{d}=\gamma\E[\eta_C(V)^2|V]+(\E[\eta_C(V)|V])^2.
\eea
Taking expectation both sides we have
\bea\label{eqRHShii}
\lim_{d,N\rw\infty}\frac{\E\l[\|G_{ii}\|_{\operatorname{F}}^2\r]}{d}=\gamma\E[\xi_C(V)^2]+(\E[\xi_C(V)])^2,
\eea
\bea\label{eqRHShij}
\lim_{d,N\rw\infty}\frac{\E\l[\|G_{ij}\|_{\operatorname{F}}^2\r]}{d}=\gamma\E[\eta_C(V)^2]+(\E[\eta_C(V)])^2.
\eea

Write for short that $\xi_C=\xi_C(V),\ \eta_C=\eta_C(V)$. By the strong law of large number, as $C\rw \infty$,
$$
\E\l[\sum_{k=1}^C r_kV_{k1}^2\ \bigg|\ \mf{z}\r]=\E\l[\frac{\frac{1}{C}\sum_{k=1}^C\exp({\sigma(\mf{z})^\top v_k})V_{k1}^2}{\frac{1}{C}\sum_{k=1}^C\exp({\sigma(\mf{z})^\top v_k})}\ \bigg|\ \mf{z}\r]\rw \frac{\E[\exp({\sigma(\mf{z})^\top v_k})V_{k1}^2\ |\ \mf{z}]}{\E[\exp({\sigma(\mf{z})^\top v_k})\ |\ \mf{z}]}=\frac{\sigma(z_1)^2}{m^2}+\frac{1}{m},
$$
$$
\E\l[\l(\sum_{k=1}^C r_kV_{k1}\r)^2\ \bigg|\ \mf{z}\r]=\E\l[\l(\frac{\frac{1}{C}\sum_{k=1}^C\exp({\sigma(\mf{z})^\top v_k})V_{k1}}{\frac{1}{C}\sum_{k=1}^C\exp({\sigma(\mf{z})^\top v_k})}\r)^2\ \bigg|\ \mf{z}\r]\rw \l(\frac{\E[\exp({\sigma(\mf{z})^\top v_k})V_{k1}\ |\ \mf{z}]}{\E[\exp({\sigma(\mf{z})^\top v_k})\ |\ \mf{z}]}\r)^2=\frac{\sigma(z_1)^2}{m^2}.
$$
Therefore 
$$
\E[\xi_C]=\E[\E[\xi_C|\mf{z}]]=\E\l[\frac{\mf{1}(z_1>0)}{m}\r]=\frac{1}{2m}.
$$
Similarly,
$$
\E[\xi_C^2]=\E[\E[\xi_C^2|\mf{z}]]=\E\l[\frac{\mf{1}(z_1>0)}{m^2}\r]=\frac{1}{2m^2}.
$$
Therefore 
$$
\lim_{C\rw\infty}\lim_{d,N\rw\infty}\frac{\|H_{ii}\|_{\operatorname{F}}^2}{d}=\frac{2\gamma+1}{4m^2}.
$$
For the case of $G_{ij}$, by repeating all arguments above, we have
\bea
\eta_C=&\mf{1}(z_1>0)\mf{1}(z_2>0)\l[\sum_{k=1}^C r_kV_{k1}V_{k2}-\l(\sum_{k=1}^C r_k V_{k1}\r)\l(\sum_{k=1}^C r_k V_{k2}\r)\r]\\
=&\mf{1}(z_1>0)\mf{1}(z_2>0)\sum_{k=1}^C\sum_{l=1}^C r_k r_l V_{k1}(V_{k2}-V_{l2})\\
=&\frac{1}{2}\cdot\mf{1}(z_1>0)\mf{1}(z_2>0)\sum_{k=1}^C\sum_{l=1}^C r_k r_l (V_{k1}-V_{l1})(V_{k2}-V_{l2})\\
=&\frac{1}{2}\cdot\mf{1}(z_1>0)\mf{1}(z_2>0) \frac{\sum_{k=1}^C\sum_{l=1}^C \exp(\sigma(\mf{z})^\top v_k)\exp(\sigma(\mf{z})^\top v_l) (V_{k1}-V_{l1})(V_{k2}-V_{l2})}{\l[\sum_{k=1}^C \exp(\sigma(\mf{z})^\top v_k)\r]^2}.
\eea
Let 
$$
h(k,l)=\exp(\sigma(\mf{z})^\top v_k)\exp(\sigma(\mf{z})^\top v_l) (V_{k1}-V_{l1})(V_{k2}-V_{l2}).
$$
Then for $k\ne k'\ne l\ne l'$,
$$
\E[h(k,l)\ |\ \mf{z}]=0,\quad \E[h(k,l)^2\ |\ \mf{z}]<\infty,\quad
\E[h(k,l)h(k',l')\ |\ \mf{z}]=0,
$$
$$
\E[h(k,l)h(k,l')\ |\ \mf{z}]= \l(\frac{\sigma(z_1)^2\sigma(z_2)^2}{m^4}+\frac{\sigma(z_1)^2+\sigma(z_2)^2}{m^3}+\frac{1}{m^2}\r)\exp\l(\frac{3}{m}\|\sigma(\mf{z})\|^2\r).
$$

Then as $C\rw\infty$,
\bea
\E[\sqrt{C}\eta_C]&=\E[\E\sqrt{C}\eta_C|\mf{z}]]\\
&=\E\l[\frac{1}{2}\mf{1}(z_1>0)\mf{1}(z_2>0)\E\l[\frac{\frac{1}{C^{3/2}}\sum_{k=1}^C\sum_{l=1}^C h(k,l)}{\l[\frac{1}{C}\sum_{k=1}^C \exp(\sigma(\mf{z})^\top v_k)\r]^2}\bigg|\mf{z}\r]\r]\\
&=\E\l[\frac{1}{2}\mf{1}(z_1>0)\mf{1}(z_2>0)\frac{\frac{1}{C^{3/2}}\sum_{k=1}^C\sum_{l=1}^C\E\l[ h(k,l)\bigg|\mf{z}\r]}{\E\l[ \exp(\sigma(\mf{z})^\top v_k)\bigg|\mf{z}\r]^2}\r]+o(1)\\
&=o(1),
\eea

\bea
\E[C\eta_C^2]&=\E[\E[C\eta_C^2|\mf{z}]]\\
&=\E\l[\frac{1}{4}\mf{1}(z_1>0)\mf{1}(z_2>0)\E\l[\frac{\frac{1}{C^3}\l[\sum_{k=1}^C\sum_{l=1}^C h(k,l)\r]^2}{\l[\frac{1}{C}\sum_{k=1}^C \exp(\sigma(\mf{z})^\top v_k)\r]^4}\bigg|\mf{z}\r]\r]\\
&\rw \E\l[\frac{1}{4}\mf{1}(z_1>0)\mf{1}(z_2>0)\frac{4\E\l[ h(k,l)h(k,l')\bigg|\mf{z}\r]}{\E\l[ \exp(\sigma(\mf{z})^\top v_k)\bigg|\mf{z}\r]^4}\r]\\
&=\E\l[\mf{1}(z_1>0)\mf{1}(z_2>0)\l(\frac{\sigma(z_1)^2\sigma(z_2)^2}{m^4}+\frac{\sigma(z_1)^2+\sigma(z_2)^2}{m^3}+\frac{1}{m^2}\r)\exp\l(\frac{1}{m}\|\sigma(\mf{z})\|^2\r)\r]\\
&=\l(\E\l[\l(\frac{\sigma(z_1)^2}{m^2}+\frac{1}{m}\r)\exp\l(\frac{\sigma(z_1)^2}{m}\r)\mf{1}(z_1>0)\r]\r)^2 \l(\E\l[\exp\l(\frac{\sigma(z_1)^2}{m}\r)\r]\r)^{m-2}\\
&=\frac{(m-1)^2}{2^m(m-2)^3m}\l(\sqrt{\frac{m}{m-2}}+1\r)^{m-2}.
\eea
Then from \eqref{eqRHShij} we finish the proof.
\end{proof}

\subsubsection{Proof for the Hidden-layer Hessian with MSE Loss}
\label{appendix_hidden_weights_mse}

For short of notations we write
$K_{ii}=\frac{\partial^2\ell_{\text{MSE}}(W,V)}{\partial w_i  \partial w_i^\top},\ K_{ij}=\frac{\partial^2\ell_{\text{MSE}}(W,V)}{\partial w_i  \partial w_j^\top}$. Then
$$
K_{ii}=\l(\sum_{k=1}^C V_{ki}^2\r)L_{ii},\quad
K_{ij}=\l(\sum_{k=1}^C V_{ki}V_{kj}\r)L_{ij},
$$
where
$$
L_{ii}=\frac{1}{N}\sum_{n=1}^N \mf{1}(w_i^\top x_n>0)x_nx_n^\top,\quad L_{ij}=\frac{1}{N}\sum_{n=1}^N \mf{1}(w_i^\top x_n>0)\mf{1}(w_j^\top x_n>0) x_nx_n^\top.
$$

The following decoupling lemma is motivated by \citep{hanin2020products}.

\begin{lem}\label{decouplingHanin}
Under the assumptions in Theorem \ref{thm_nn}, we have:

\begin{enumerate}
    \item There exists random matrices $\hat{X}_N,\hat{\Lambda}_N$ in the same probability space, such that $\hat{X}_N\stackrel{d}{=}X_N$, $\hat{\Lambda}_N$ is a $N$-dimensional diagonal matrix with entries i.i.d. $\operatorname{ber}(\frac{1}{2})$ random variables independent of $\hat{X}_N$, and $L_{ii}\stackrel{d}{=}\frac{1}{N}\hat{X}_N\hat{\Lambda}_N\hat{X}_N'$.
    \item For $i\ne j$, there exists random matrices $\hat{X}_N,\hat{\Lambda}_N$ in the same probability space, such that $\hat{X}_N\stackrel{d}{=}X_N$, $\hat{\Lambda}_N$ is a $N$-dimensional diagonal matrix with entries i.i.d. $\operatorname{ber}(\frac{1}{4})$ random variables independent of $\hat{X}_N$, and $L_{ij}\stackrel{d}{=}\frac{1}{N}\hat{X}_N\hat{\Lambda}_N\hat{X}_N'$.
\end{enumerate}
\end{lem}

\begin{proof}
Let $\xi_1,\cdots,\xi_p,\eta_1,\cdots,\eta_N$  be i.i.d. Radamacher random variables. Let
$$
\tX_N=\diag(\xi)X_N\diag(\eta),\quad \tL_N=\diag(1(w_i'\tx_1>0),\cdots,1(w_i'\tx_N>0)).
$$
Here $\wt{x}_1,\cdots,\wt{x}_N$ are column vectors of $\wt{X}_N$.
Since entries of $X_N$ are i.i.d. centered, $\{\tX_N,\tL_N\}\stackrel{d}{=}\{X_N,\Lambda_N\}$. Hence
\bea
G_{ii}\stackrel{d}{=}\frac{1}{N}\tX_N\tL_N\tX_N&=\frac{1}{N}\diag(\xi)X_N\diag(\eta)\tL_N\diag(\eta)X_N\diag(\xi)\\
&=\frac{1}{N}\diag(\xi)X_N\tL_NX_N\diag(\xi).
\eea
Clearly $\diag(\xi)X_N\stackrel{d}{=}X_N$, then it suffices to show that $\tL_n$ is diagonal with i.i.d. $\operatorname{ber}(\frac{1}{2})$ entries independent of $\{\xi,X_N\}$. To see this, we have
$$
\tL_N(n,n)=1(w_i'\diag(\xi)x_n\eta_n>0),\quad 1\leq n\leq N.
$$
Then the required property follows from that $(\eta_n)_n$ are i.i.d. valuing in $\{1,-1\}$ with equal probability.

For $G_{ij} (i\ne j)$, the proof is the same, except that
$$
\tL_N(n,n)=1(w_i'\diag(\xi)x_n\eta_n>0)\cdot1( w_j'\diag(\xi)x_n\eta_n>0),\quad 1\leq n\leq N
$$
are i.i.d. $\operatorname{ber}(\frac{1}{4})$. 
\end{proof}

A probability measure $\mu$ is said to have the Marchenko-Pastur distribution with parameter $y>0$ and $\sigma^2>0$, denoted as $\text{MP}(y,\sigma^2)$, if
$$
\mu(dx)=(1-y^{-1})1(y>1)\delta_0(dx)+\frac
{\sqrt{(\lambda_+-x)(x-\lambda_-)}}{2\pi\sigma^2y x}1_{(\lambda_-,\lambda_+)}(x)dx.
$$
Here $\lambda_+=\sigma^2(1+\sqrt{y})^2$, $\lambda_-=\sigma^2(1-\sqrt{y})^2$.
The Stieltjes transform of $\text{MP}(y,\sigma^2)$ is given by
$$
s(z)=\frac{\sigma^2(1-y)-z+\sqrt{(z-\sigma^2-y\sigma^2)^2-4y\sigma^4}}{2yz\sigma^2},
$$
or equivalently by the equation (writing $s=s_{\mu}(z)$ for short)
\bea\label{MPs2}
yz\sigma^2 s^2+(z-\sigma^2(1-y))s+1=0.
\eea

\begin{prop}\label{NN_MSE_MP}
     As $d,N\rw\infty, \frac
    {d}{N}\rw \gamma\in (0,+\infty)$, we have
    \begin{enumerate}
        \item $\mu_{L_{ii}}$ converges weakly almost surely to $\text{MP}(2\gamma,\frac{1}{2})$.
        \item For $i\ne j$, $\mu_{L_{ij}}$ converges weakly almost surely to $\text{MP}(4\gamma,\frac{1}{4})$.
    \end{enumerate}
\end{prop}

\begin{proof}
    From Lemma \ref{decouplingHanin}, we have $L_{ii}\stackrel{d}{=}\frac{d}{N}T_N$, where $T_N=\frac{1}{d}\hat{X}_N\hat{\Lambda}_N\hat{X}_N'$.
From Proposition \ref{GMP}, $\mu_{T_N}$ converges weakly almost surely to a deterministic measure $\mu_{T}$. And $s_{\mu_{T}}(z)$ is specified by the equation
\bea\label{GMPii}
s(z)=\frac{1}{\frac{1}{\gamma}\int_{\mathbb{R}}\frac
    {t\nu_1(dt)}{1+ts(z)}-z},\quad \forall z\in\mathbb{C}^+,
\eea
where $\nu_1=\frac{1}{2}\delta_0+\frac{1}{2}\delta_1$. The equation can be simplified as 
$$
zs^2+\l(z-\frac
{1}{2\gamma}+1\r)s+1=0.
$$
This is exactly (\ref{MPs2}) with $y=2\gamma,\ \sigma^2=\frac{1}{2\gamma}$. Then $\mu_{T_N}\rw MP(2\gamma,\frac
{1}{2\gamma})$ and hence $\mu_{L_{ii}}\rw MP(2\gamma,\frac
{1}{2})$. The $L_{ij}$ case in the same procedure, replacing $\nu_1$ in (\ref{GMPii}) with $\nu_2=\frac{3}{4}\delta_0+\frac{1}{4}\delta_1$. 
\end{proof}

\begin{prop}
    Suppose that $m$ is fixed and as $d,N\rw\infty, \frac
    {d}{N}\rw \gamma\in (0,+\infty)$. Then for $i\ne j$, we have
    \bea
\lim_{C\rw\infty}\lim_{d,N\rw\infty}\frac{\E\l[\|H_{ii}\|_{\operatorname{F}}^2\r]}{C^2d}&=\frac{1+2\gamma}{4m^2},\\
\lim_{C\rw\infty}\lim_{d,N\rw\infty}\frac{\E\l[\|H_{ij}\|_{\operatorname{F}}^2\r]}{Cd}&=\frac{1+4\gamma}{16m^2}.
        \eea
\end{prop}

\begin{proof}
    From Proposition \ref{NN_MSE_MP} we have
    \bea\label{eqRHSuii}
    \frac{\|L_{ii}\|_{\operatorname{F}}^2}{d}=\int_{\mb{R}} x^2\mu_{L_{ii}}(dx)\rw \int_{\mb{R}} x^2\mu_{MP,2\gamma,\frac{1}{2}}(dx)=\frac{1+2\gamma}{4}\quad a.s.
    \eea
    For $i\ne j$, 
    \bea\label{eqRHSuij}
    \frac{\|L_{ij}\|_{\operatorname{F}}^2}{d}=\int_{\mb{R}} x^2\mu_{L_{ij}}(dx)\rw \int_{\mb{R}} x^2\mu_{MP,4\gamma,\frac{1}{4}}(dx)=\frac{1+4\gamma}{16}\quad a.s.
    \eea
    Since entries of $V$ are i.i.d. $\ml{N}(0,\frac{1}{m})$ independent of $\{L_{ii},L_{ij}\}$,
    $$
    \lim_{C\rw\infty}\lim_{d,N\rw\infty}\frac{\E\l[\|H_{ii}\|_{\operatorname{F}}^2\r]}{C^2d}=\lim_{C\rw\infty}\E\l[\l(\frac{\sum_{k=1}^C V_{ki}^2}{C}\r)^2\r]\lim_{d,N\rw\infty}\E\l[\frac{\|L_{ii}\|_{\operatorname{F}}^2}{d}\r]=\frac{1+2\gamma}{4m^2},
    $$
    $$
    \lim_{C\rw\infty}\lim_{d,N\rw\infty}\frac{\E\l[\|H_{ij}\|_{\operatorname{F}}^2\r]}{Cd}=\lim_{C\rw\infty}\E\l[\l(\frac{\sum_{k=1}^C V_{ki}V_{kj}}{\sqrt{C}}\r)^2\r]\lim_{d,N\rw\infty}\E\l[\frac{\|L_{ii}\|_{\operatorname{F}}^2}{d}\r]=\frac{1+4\gamma}{16m^2}.
    $$
\end{proof}

\subsubsection{Proof for the Output-layer Hessian with CE Loss}
\label{appendix_out_weights_ce}

Denote that
$$
 G_{ii}:=\frac{\partial^2\ell_{\text{CE}}(W,V)}{\partial v_i  \partial v_i^\top} =  \frac{1}{N}\sum_{n= 1}^N  p_{n,i} (1-p_{n,i}) \sigma(Wx_n) \sigma(Wx_n)^\top, 
$$
and for $i\ne j$,
$$
 G_{ij}:=\frac{\partial^2\ell_{\text{CE}}(W,V)}{\partial v_i  \partial v_j^\top} =  \frac{1}{N}\sum_{n= 1}^N  p_{n,i}p_{n,j} \sigma(Wx_n) \sigma(Wx_n)^\top.
$$

Let $Z$ be a $d\times N$ random matrix with entries i.i.d. $N(0,1)$ independent of $V$, and $z_1,\cdots,z_N$ are column vectors of $Z$. Define
$$
\wt{p}_{i,n}=\frac{\exp(\sigma(z_n)^\top v_i)}{\sum_{c=1}^C\exp(\sigma(z_n)^\top v_c)},\quad n\in [N],i\in [C],
$$
$$
\wt{G}_{ii} =  \frac{1}{N}\sum_{n= 1}^N  \wt{p}_{n,i} (1-\wt{p}_{n,i}) \sigma(z_n) \sigma(z_n)^\top, 
$$
$$
\wt{G}_{ij} =  \frac{1}{N}\sum_{n= 1}^N  \wt{p}_{n,i} \wt{p}_{n,j} \sigma(z_n) \sigma(z_n)^\top. 
$$
Following the proof of Lemma \ref{decoupling}\ref{decoupling_NN_CE_Hww} with the Lindeberg principle, one can show that for $k,l\in [m],\ k\ne l$,
$$
\lim_{d,N\rw\infty}\l(\E\l[G_{ii}(k,k)^2\r]-\E\l[\wt{G}_{ii}(k,k)^2\r]\r)=0,
$$
$$
\lim_{d,N\rw\infty}\l(\E\l[G_{ii}(k,l)^2\r]-\E\l[\wt{G}_{ii}(k,l)^2\r]\r)=0,
$$
$$
\lim_{d,N\rw\infty}\l(\E\l[G_{ij}(k,k)^2\r]-\E\l[\wt{G}_{ij}(k,k)^2\r]\r)=0,
$$
$$
\lim_{d,N\rw\infty}\l(\E\l[G_{ij}(k,l)^2\r]-\E\l[\wt{G}_{ij}(k,l)^2\r]\r)=0.
$$
Then since $m$ is fixed and
\bea
\E\l[\|G_{ii}\|_{\operatorname{F}}^2\r]&=m\E\l[G_{ii}(k,k)^2\r]+m(m-1)\E\l[G_{ii}(k,l)^2\r],\\
\E\l[\|G_{ij}\|_{\operatorname{F}}^2\r]&=m\E\l[G_{ij}(k,k)^2\r]+m(m-1)\E\l[G_{ij}(k,l)^2\r],
\eea
\bea\label{eqRHSqii}
\E\l[\|\wt{G}_{ii}\|_{\operatorname{F}}^2\r]&=m\E\l[\wt{G}_{ii}(k,k)^2\r]+m(m-1)\E\l[\wt{G}_{ii}(k,l)^2\r],
\eea
\bea\label{eqRHSqij}
\E\l[\|\wt{G}_{ij}\|_{\operatorname{F}}^2\r]&=m\E\l[\wt{G}_{ij}(k,k)^2\r]+m(m-1)\E\l[\wt{G}_{ij}(k,l)^2\r],
\eea
we have
\bea\label{eq60}
\lim_{d,N\rw\infty}\l(\E\l[\|G_{ii}\|_{\operatorname{F}}^2\r]-\E\l[\|\wt{G}_{ii}\|_{\operatorname{F}}^2\r]\r)&=0,\\
\lim_{d,N\rw\infty}\l(\E\l[\|G_{ij}\|_{\operatorname{F}}^2\r]-\E\l[\|\wt{G}_{ij}\|_{\operatorname{F}}^2\r]\r)&=0.\\
\eea
From
\bea
\wt{G}_{ii}(k,k)&=\frac{1}{N}\sum_{n= 1}^N  \wt{p}_{n,i} (1-\wt{p}_{n,i}) \sigma(z_{n,k})^2,\\
\wt{G}_{ii}(k,l)&=\frac{1}{N}\sum_{n= 1}^N  \wt{p}_{n,i} (1-\wt{p}_{n,i}) \sigma(z_{n,k}) \sigma(z_{n,l}),\\
\wt{G}_{ij}(k,k)&=\frac{1}{N}\sum_{n= 1}^N  \wt{p}_{n,i}\wt{d}_{n,j} \sigma(z_{n,k})^2,\\
\wt{G}_{ij}(k,l)&=\frac{1}{N}\sum_{n= 1}^N  \wt{p}_{n,i}\wt{p}_{n,j} \sigma(z_{n,k}) \sigma(z_{n,l}),\\
\eea
we have
\bea
\lim_{p,N\rw\infty}\E\l[\wt{G}_{ii}(k,k)^2\r]&=\E\l[\wt{p}_{1,i}\wt{p}_{2,i}(1-\wt{p}_{1,i})(1-\wt{p}_{2,i})\sigma(z_{1,k})^2\sigma(z_{2,k})^2\r],\\
\lim_{p,N\rw\infty}\E\l[\wt{G}_{ii}(k,l)^2\r]&=\E\l[\wt{p}_{1,i}\wt{p}_{2,i}(1-\wt{p}_{1,i})(1-\wt{p}_{2,i})\sigma(z_{1,k})\sigma(z_{1,l})\sigma(z_{2,k})\sigma(z_{2,l})\r],\\
\lim_{p,N\rw\infty}\E\l[\wt{G}_{ij}(k,k)^2\r]&=\E\l[\wt{p}_{1,i}\wt{p}_{2,i}\wt{p}_{1,j}\wt{p}_{2,j}\sigma(z_{1,k})^2\sigma(z_{2,k})^2\r],\\
\lim_{p,N\rw\infty}\E\l[\wt{G}_{ij}(k,l)^2\r]&=\E\l[\wt{p}_{1,i}\wt{p}_{2,i}\wt{p}_{1,j}\wt{p}_{2,j}\sigma(z_{1,k})\sigma(z_{1,l})\sigma(z_{2,k})\sigma(z_{2,l})\r].
\eea
Therefore
\bea\label{eq63}
&\lim_{C\rw\infty}\lim_{p,N\rw\infty}C^2\E\l[\wt{G}_{ii}(k,k)^2\r]\\
=&\lim_{C\rw\infty}\E\l[\E\l[\frac{\exp\big((\sigma(z_1)+\sigma(z_2))^\top v_i\big)}{\l(\frac{1}{C}\sum_{c=1}^C\exp(\sigma(z_1)^\top v_c)\r)\l(\frac{1}{C}\sum_{c=1}^C\exp(\sigma(z_2)^\top v_c)\r)}(1-\wt{p}_{1,i})(1-\wt{p}_{2,i})\sigma(z_{1,k})^2\sigma(z_{2,k})^2\ \bigg|\ z_1,z_2 \r]\r]\\
=&\E\l[\frac{\E\l[ \exp\big((\sigma(z_1)+\sigma(z_2))^\top v_i\big)\sigma(z_{1,k})^2\sigma(z_{2,k})^2 \bigg|\ z_1,z_2 \r]}{\l(\E\l[ \exp\big(\sigma(z_1)^\top v_i\big)\bigg|\ z_1 \r]\r)\l(\E\l[\exp\big(\sigma(z_2)^\top v_i\big)\bigg|\ z_2 \r]\r) }\r]\\
=&\E\l[ \frac{\exp\big(\frac{1}{2m}|\sigma(z_1)+\sigma(z_2)|^2\big)\sigma(z_{1,k})^2\sigma(z_{2,k})^2}{\exp\big(\frac{1}{2m}(|\sigma(z_1)|^2+|\sigma(z_2)|^2)\big)} \r]\\
=&\E\l[\exp\big(\frac{1}{m}\sigma(z_{1,1})\sigma(z_{1,2})\big)\sigma(z_{1,1})^2\sigma(z_{1,2})^2\r]\l( \E\l[\exp\big(\frac{1}{m}\sigma(z_{1,1})\sigma(z_{1,2})\big)\r] \r)^{m-1}\\
=&a_{12}b_1^{m-1},
\eea

\bea
&\lim_{C\rw\infty}\lim_{p,N\rw\infty}C^2\E\l[\wt{G}_{ii}(k,l)^2\r]\\
=&\lim_{C\rw\infty}\E\l[\E\l[\frac{\exp\big((\sigma(z_1)+\sigma(z_2))^\top v_i\big)(1-\wt{p}_{1,i})(1-\wt{p}_{2,i})}{\l(\frac{1}{C}\sum_{c=1}^C\exp(\sigma(z_1)^\top v_c)\r)\l(\frac{1}{C}\sum_{c=1}^C\exp(\sigma(z_2)^\top v_c)\r)} \sigma(z_{1,k})\sigma(z_{2,k})\sigma(z_{1,l})\sigma(z_{2,l})\ \bigg|\ z_1,z_2 \r]\r]\\
=&\E\l[\frac{\E\l[ \exp\big((\sigma(z_1)+\sigma(z_2))^\top v_i\big)\sigma(z_{1,k})\sigma(z_{2,k})\sigma(z_{1,l})\sigma(z_{2,l}) \bigg|\ z_1,z_2 \r]}{\l(\E\l[ \exp\big(\sigma(z_1)^\top v_i\big)\bigg|\ z_1 \r]\r)\l(\E\l[\exp\big(\sigma(z_2)^\top v_i\big)\bigg|\ z_2 \r]\r)}\r]\\
=&\E\l[ \frac{\exp\big(\frac{1}{2m}|\sigma(z_1)+\sigma(z_2)|^2\big)\sigma(z_{1,k})\sigma(z_{2,k})\sigma(z_{1,l})\sigma(z_{2,l})}{\exp\big(\frac{1}{2m}(|\sigma(z_1)|^2+|\sigma(z_2)|^2)\big)} \r]\\
=&\E\l[\exp\big(\frac{1}{m}\sigma(z_{1,1})\sigma(z_{1,2})\big)\sigma(z_{1,1})\sigma(z_{1,2})\r]^2\l( \E\l[\exp\big(\frac{1}{m}\sigma(z_{1,1})\sigma(z_{1,2})\big)\r] \r)^{m-2}\\
=&a_{11}^2b_1^{m-2},
\eea

\bea
&\lim_{C\rw\infty}\lim_{p,N\rw\infty}C^4\E\l[\wt{G}_{ij}(k,k)^2\r]\\
=&\lim_{C\rw\infty}\E\l[\E\l[\frac{\exp\big((\sigma(z_1)+\sigma(z_2))^\top (v_i+v_j)\big)}{\l(\frac{1}{C}\sum_{c=1}^C\exp(\sigma(z_1)^\top v_c)\r)^2\l(\frac{1}{C}\sum_{c=1}^C\exp(\sigma(z_2)^\top v_c)\r)^2}\sigma(z_{1,k})^2\sigma(z_{2,k})^2\ \bigg|\ z_1,z_2 \r]\r]\\
=&\E\l[\frac{\E\l[ \exp\big((\sigma(z_1)+\sigma(z_2))^\top (v_i+v_j)\big)\sigma(z_{1,k})^2\sigma(z_{2,k})^2 \bigg|\ z_1,z_2 \r]}{\l(\E\l[ \exp\big(\sigma(z_1)^\top v_i\big)\bigg|\ z_1 \r]\r)^2\l(\E\l[\exp\big(\sigma(z_2)^\top v_i\big)\bigg|\ z_2 \r]\r)^2}\r]\\
=&\E\l[ \frac{\exp\big(\frac{1}{m}|\sigma(z_1)+\sigma(z_2)|^2\big)\sigma(z_{1,k})^2\sigma(z_{2,k})^2}{\exp\big(\frac{1}{m}(|\sigma(z_1)|^2+|\sigma(z_2)|^2)\big)} \r]\\
=&\E\l[\exp\big(\frac{2}{m}\sigma(z_{1,1})\sigma(z_{1,2})\big)\sigma(z_{1,1})^2\sigma(z_{1,2})^2\r]\l( \E\l[\exp\big(\frac{2}{m}\sigma(z_{1,1})\sigma(z_{1,2})\big)\r] \r)^{m-1}\\
=&a_{22}b_2^{m-1},
\eea

\bea\label{eq66}
&\lim_{C\rw\infty}\lim_{p,N\rw\infty}C^4\E\l[\wt{G}_{ij}(k,k)^2\r]\\
=&\lim_{C\rw\infty}\E\l[\E\l[\frac{\exp\big((\sigma(z_1)+\sigma(z_2))^\top (v_i+v_j)\big)}{\l(\frac{1}{C}\sum_{c=1}^C\exp(\sigma(z_1)^\top v_c)\r)^2\l(\frac{1}{C}\sum_{c=1}^C\exp(\sigma(z_2)^\top v_c)\r)^2}\sigma(z_{1,k})\sigma(z_{2,k})\sigma(z_{1,l})\sigma(z_{2,l})\ \bigg|\ z_1,z_2 \r]\r]\\
=&\E\l[\frac{\E\l[ \exp\big((\sigma(z_1)+\sigma(z_2))^\top (v_i+v_j)\big)\sigma(z_{1,k})\sigma(z_{2,k})\sigma(z_{1,l})\sigma(z_{2,l}) \bigg|\ z_1,z_2 \r]}{\l(\E\l[ \exp\big(\sigma(z_1)^\top v_i\big)\bigg|\ z_1 \r]\r)^2\l(\E\l[\exp\big(\sigma(z_2)^\top v_i\big)\bigg|\ z_2 \r]\r)^2}\r]\\
=&\E\l[ \frac{\exp\big(\frac{1}{m}|\sigma(z_1)+\sigma(z_2)|^2\big)\sigma(z_{1,k})\sigma(z_{2,k})\sigma(z_{1,l})\sigma(z_{2,l})}{\exp\big(\frac{1}{m}(|\sigma(z_1)|^2+|\sigma(z_2)|^2)\big)} \r]\\
=&\E\l[\exp\big(\frac{2}{m}\sigma(z_{1,1})\sigma(z_{1,2})\big)\sigma(z_{1,1})\sigma(z_{1,2})\r]^2\l( \E\l[\exp\big(\frac{2}{m}\sigma(z_{1,1})\sigma(z_{1,2})\big)\r] \r)^{m-2}\\
=&a_{21}^2b_2^{m-2}.
\eea

Then from \eqref{eqRHSqii}-\eqref{eq60} and \eqref{eq63}-\eqref{eq66} we obtain \eqref{NN_CV_Hvv}. The whole proof is then completed.

\clearpage
\section{More Numerical Results}
\label{sec_experiments}

We now provide some more numerical evidence to support our theory.  We use the  the same Gaussian synthetic dataset as in Section \ref{sec_closer_look} (which follows Assumption \ref{assum_1}) and LeCun initialization (which follows Assumption \ref{assum_2}),  and try different $C$. More details can be seen in Appendix \ref{appendix_experimental_details}.

\paragraph{Case 1: linear models with MSE loss.} In Figure \ref{fig:hessian_linear_mse_different_c}, we present the Hessian of linear models under MSE loss. By the calculation of \eqref{eq_linear_model_mse_Hessian} in Section \ref{sec_preliminaries}, the Hessian is strictly block-diagonal.  The numerical results match the calculation.

\begin{figure}[h]
\vspace{-0.1cm}
    \centering
    \subfigure[\small $C = 10$]{\includegraphics[width=0.29\textwidth]{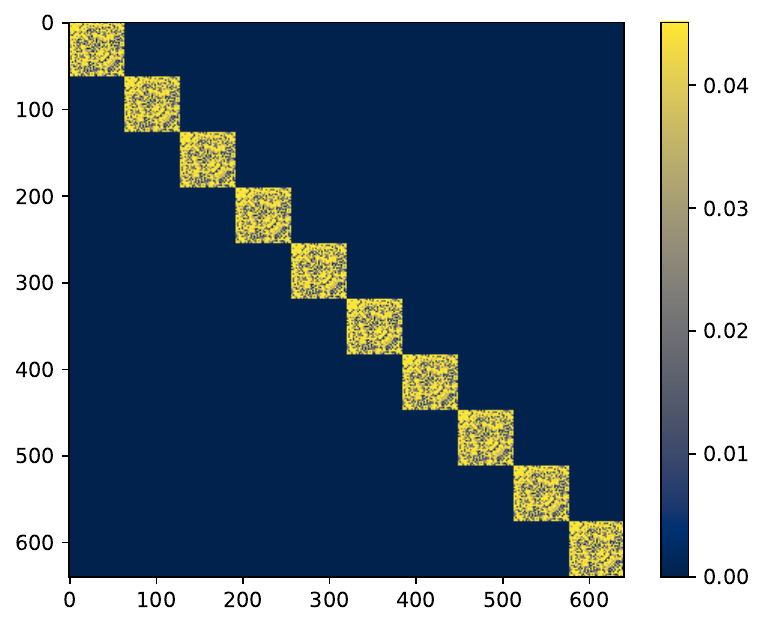}}
    \subfigure[\small $C = 100$]{\includegraphics[width=0.30\textwidth]{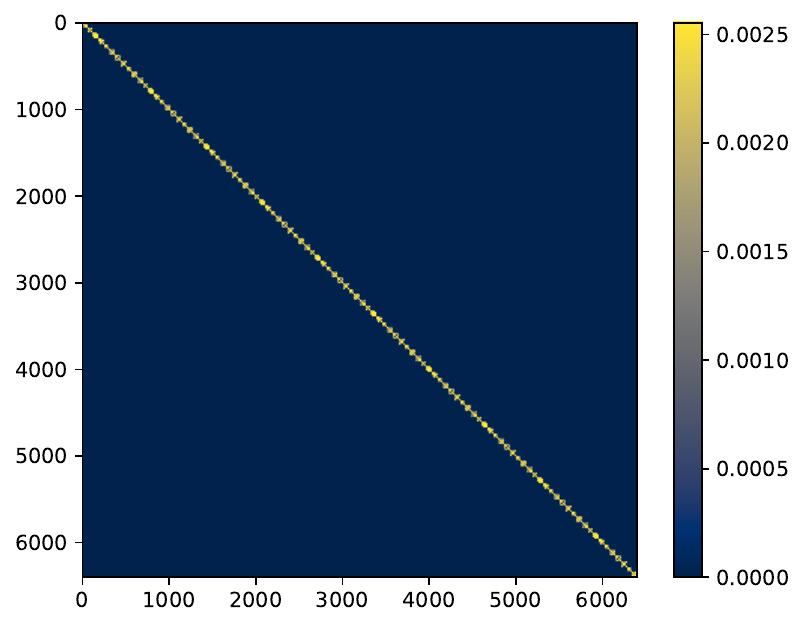}}
    \subfigure[\small $C = 1000$]{\includegraphics[width=0.30\textwidth]{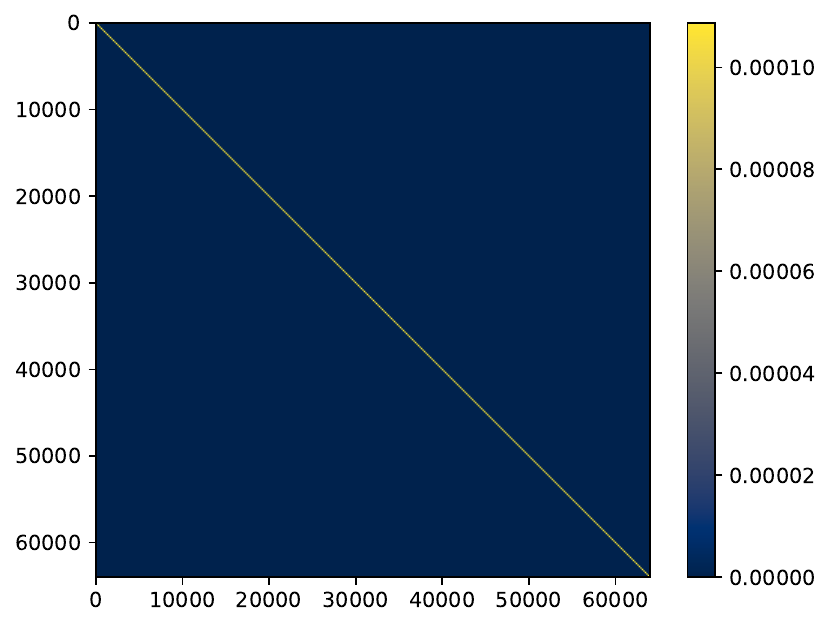}}
    \caption{\small  {\bf(a-c):}  The Hessian of {\bf Case 1: linear models with MSE loss}. We observe that the block-diagonal Hessian structure arises for all $C$. This is because the off-diagonal blocks are strictly zero in this case (see Eq. \eqref{eq_linear_model_mse_Hessian}).  }
  \label{fig:hessian_linear_mse_different_c}
\vspace{-0.3cm}
\end{figure}

\paragraph{Case 2: linear models with CE loss.} In Figure \ref{fig:hessian_linear_ce_different_c}, we present the Hessian of linear models under CE loss. The block-diagonal Hessian structure becomes \zhaoruirevise{clear} when $C$ increases, which matches our theoretical prediction in Theorem \ref{thm_linear_model}.

\begin{figure}[h]
\vspace{-0.3cm}
    \centering
    \subfigure[\small $C = 10$]{\includegraphics[width=0.29\textwidth]{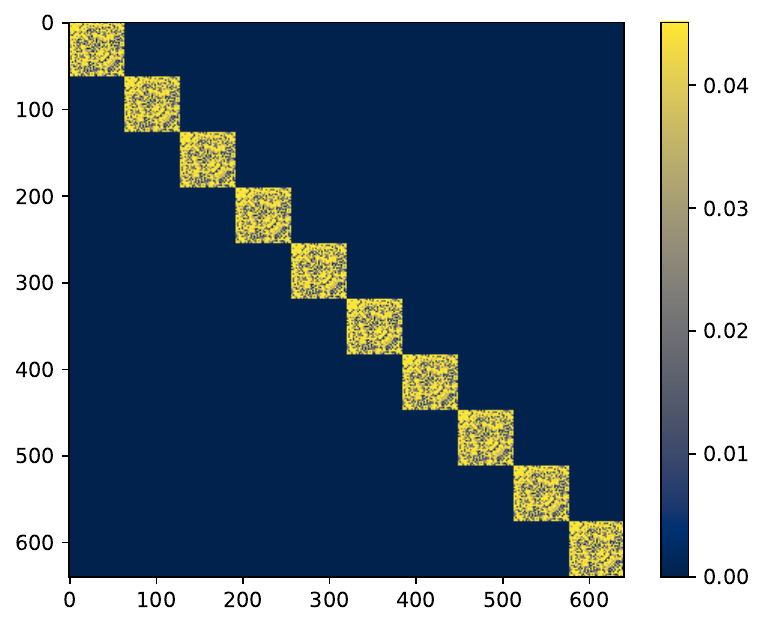}}
    \subfigure[\small $C = 100$]{\includegraphics[width=0.30\textwidth]{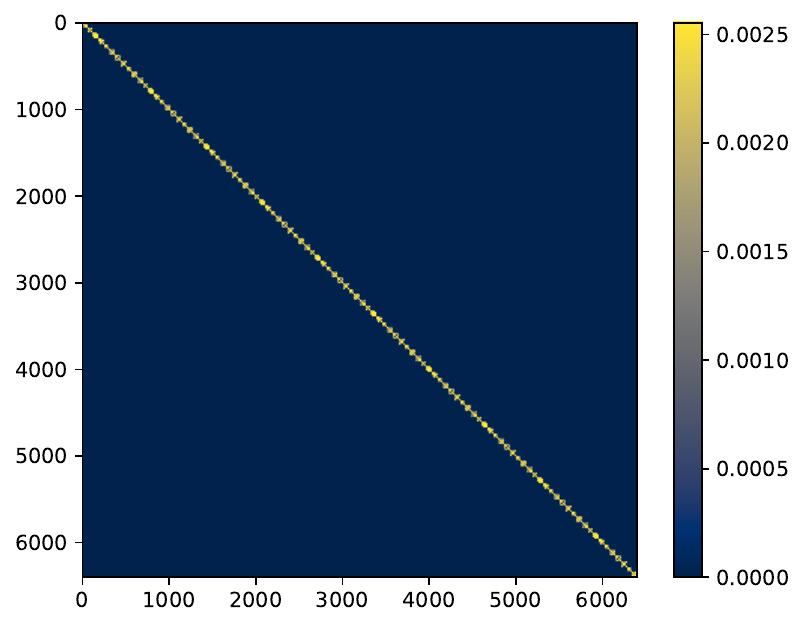}}
    \subfigure[\small $C = 1000$]{\includegraphics[width=0.30\textwidth]{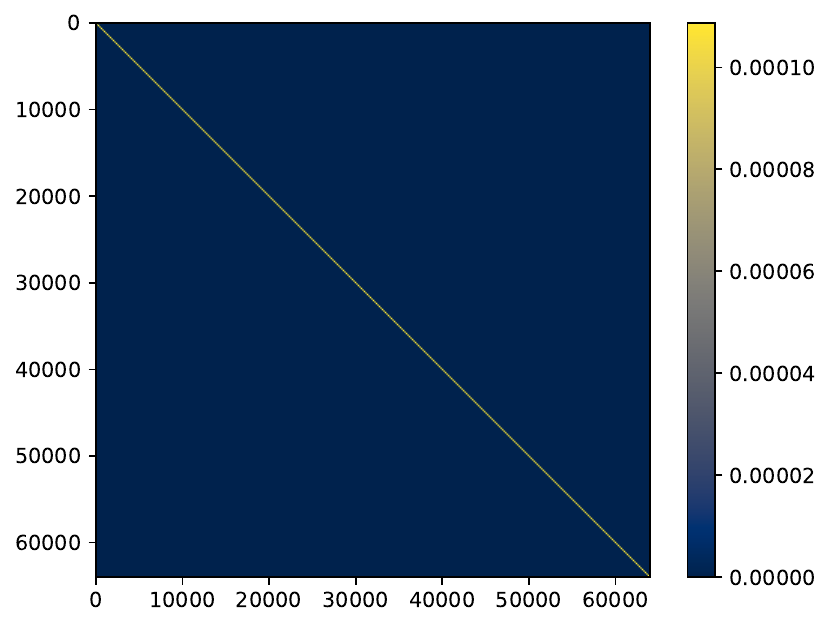}}
    \caption{\small  {\bf(a-c):}  The Hessian of {\bf Case 2: linear models with CE loss}.We observe that  the block-diagonal Hessian structure becomes \zhaoruirevise{clear} when $C$ increases. }
  \label{fig:hessian_linear_ce_different_c}
\vspace{-0.3cm}
\end{figure}

\paragraph{Case 3 and 4: 1-hidden-layer networks  with MSE and CE loss.}  In Figure \ref{fig:hessian_nn_mse_different_c} and \ref{fig:hessian_nn_ce_different_c}, we consider 1-hidden-layer network at random initialization. We present the hidden-layer Hessian $H_{ww}$ and output weights $H_{vv}$ to see if they match our theoretical prediction.  It can be seen that the block-diagonal structure becomes clearer as the number of classes $C$ increases, which matches the theoretical prediction.  These results hold for both MSE and CE  loss.

\begin{figure}[t]
 \vspace{-1.5cm} %
    \centering
    \subfigure[\small $C = 10$]{\includegraphics[width=0.30\textwidth]{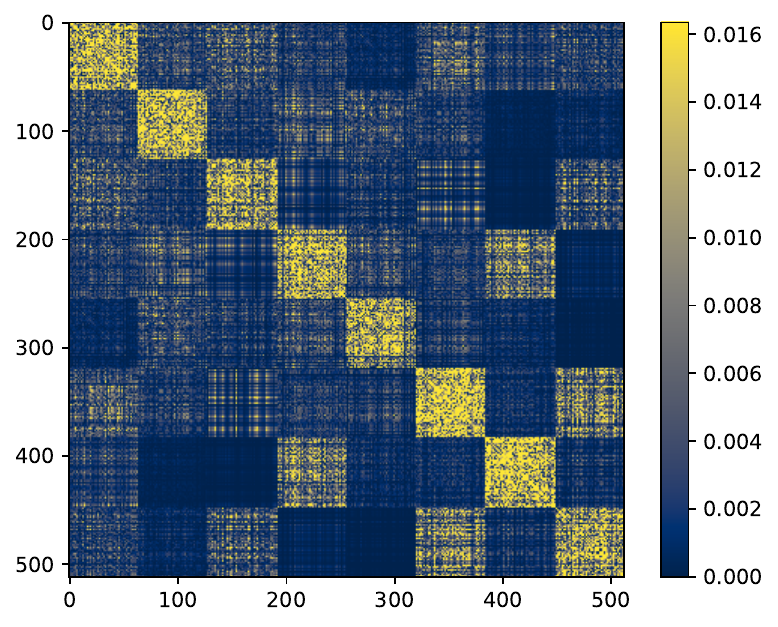}}
    \subfigure[\small $C = 100$]{\includegraphics[width=0.30\textwidth]{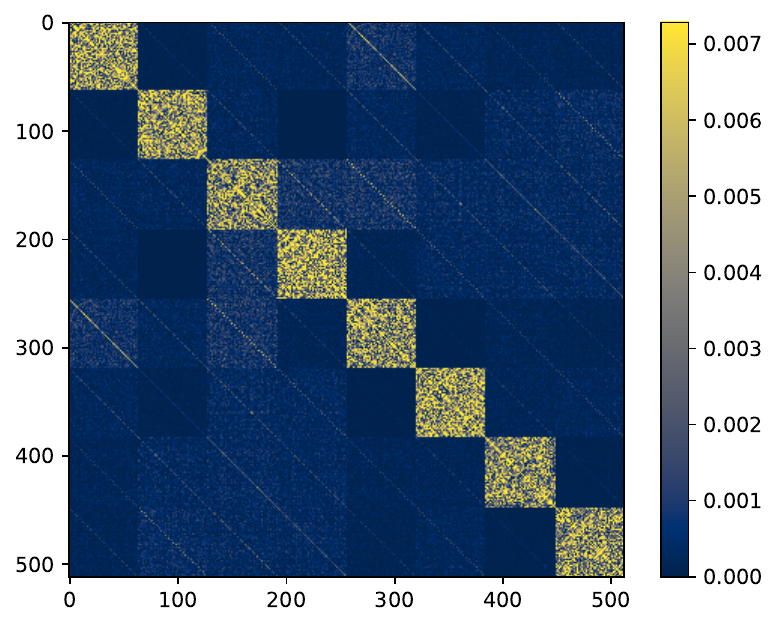}}
    \subfigure[\small $C = 1000$]{\includegraphics[width=0.30\textwidth]{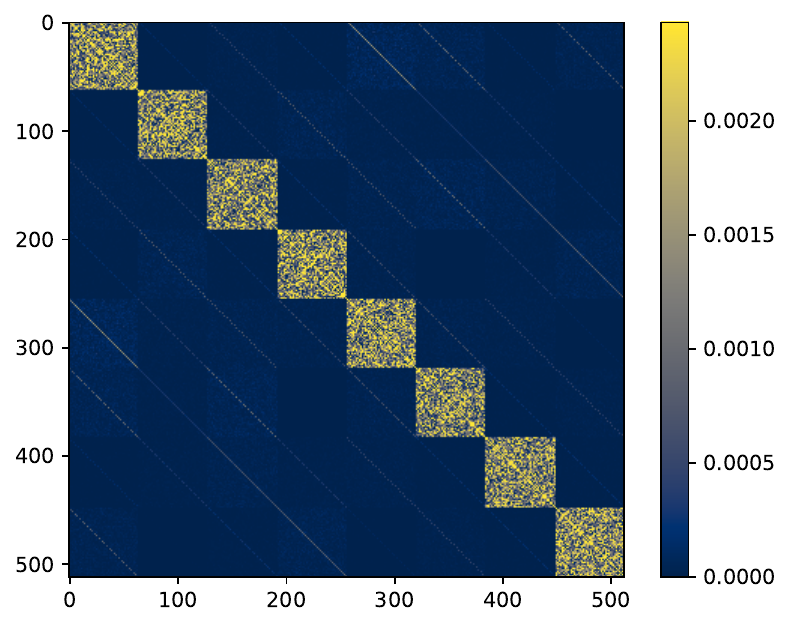}}
    \subfigure[\small $C = 10$]{\includegraphics[width=0.30\textwidth]{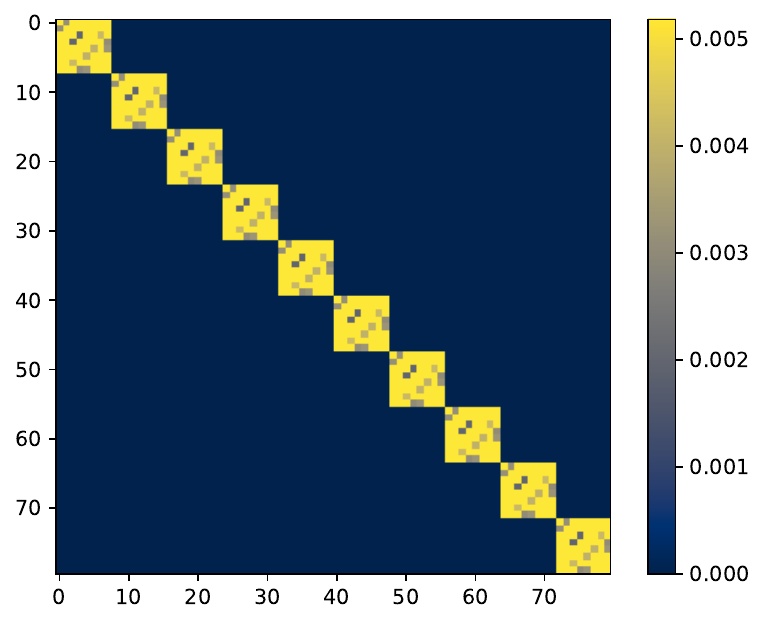}}
    \subfigure[\small $C = 100$]{\includegraphics[width=0.30\textwidth]{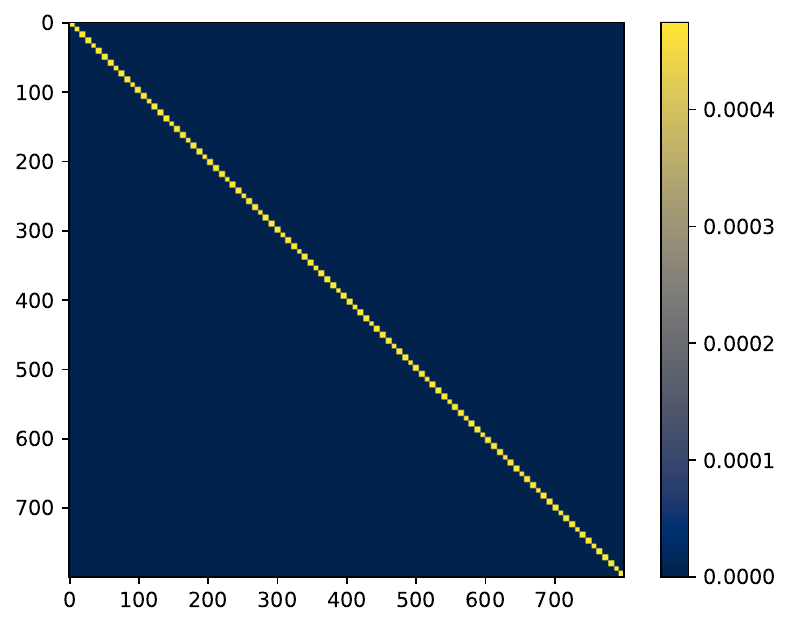}}
    \subfigure[\small $C = 1000$]{\includegraphics[width=0.30\textwidth]{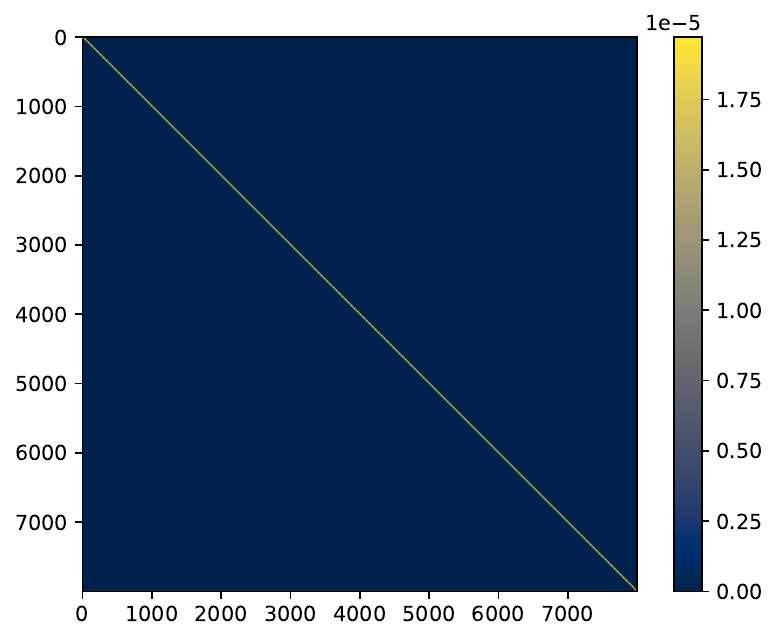}}
    \vspace{-0.3cm}
    \caption{\small   The Hessian   in {\bf Case 3: 1-hidden-layer network with MSE loss}. The network has 8 hidden neurons.  {\bf (a, b, c):} The hidden-layer Hessian  $H_{ww}$. {\bf (e, f, g):} The output-layer Hessian $H_{vv}$.  We observe that the block-diagonal Hessian structure in $H_{ww}$ becomes clearer as $C$ increases. $H_{vv}$ is always strictly block-diagonal, as expected by  Eq. \eqref{eq_nn_mse_Hessian_v}.  }
  \label{fig:hessian_nn_mse_different_c}
\end{figure}

\begin{figure}[h!]
    \centering
    \subfigure[\small $C = 10$]{\includegraphics[width=0.30\textwidth]{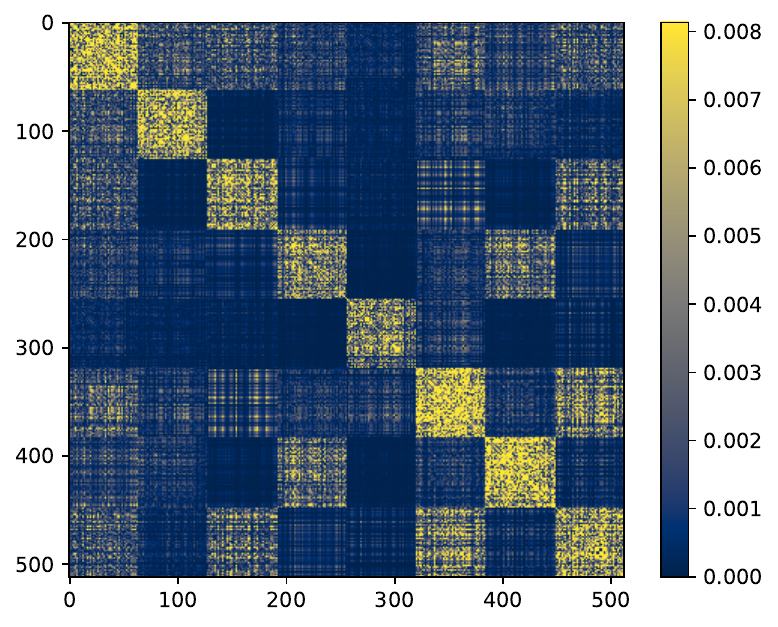}}
    \subfigure[\small $C = 100$]{\includegraphics[width=0.30\textwidth]{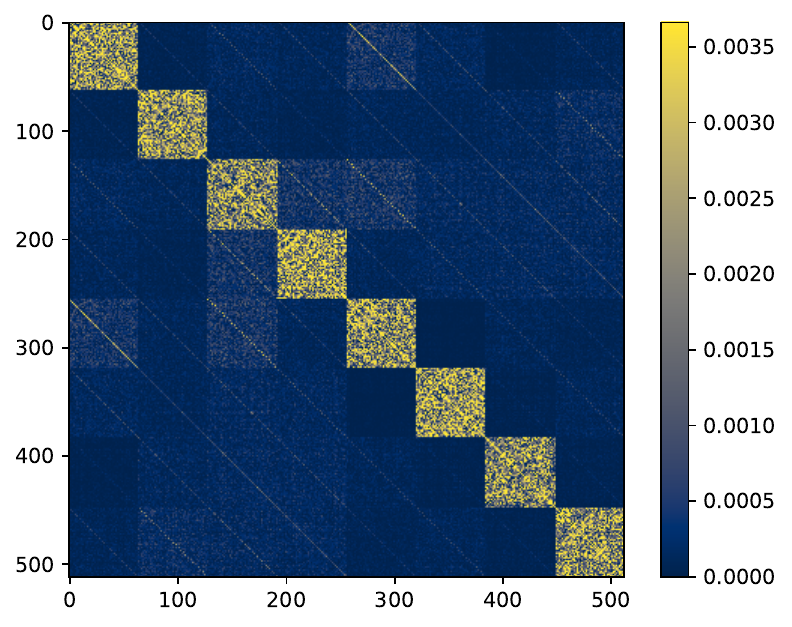}}
    \subfigure[\small $C = 1000$]{\includegraphics[width=0.30\textwidth]{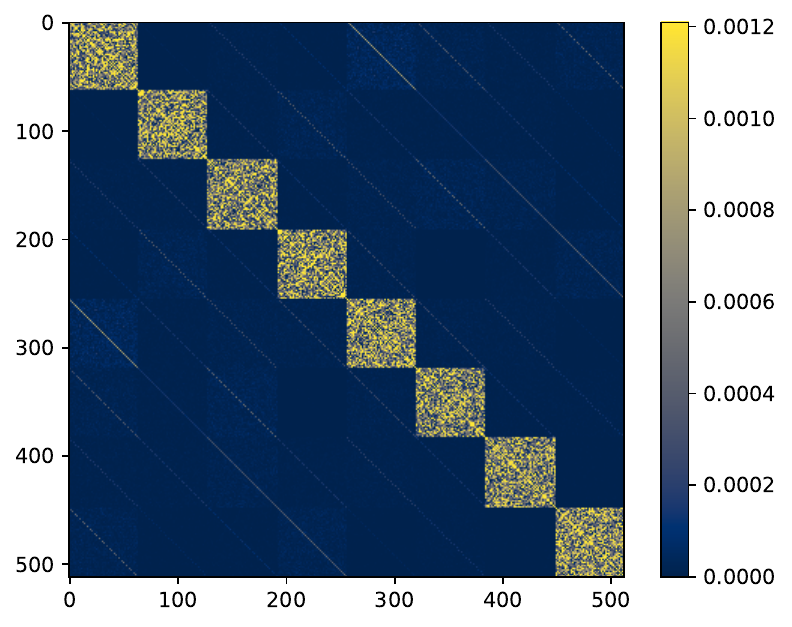}}
        \subfigure[\small $C = 10$]{\includegraphics[width=0.30\textwidth]{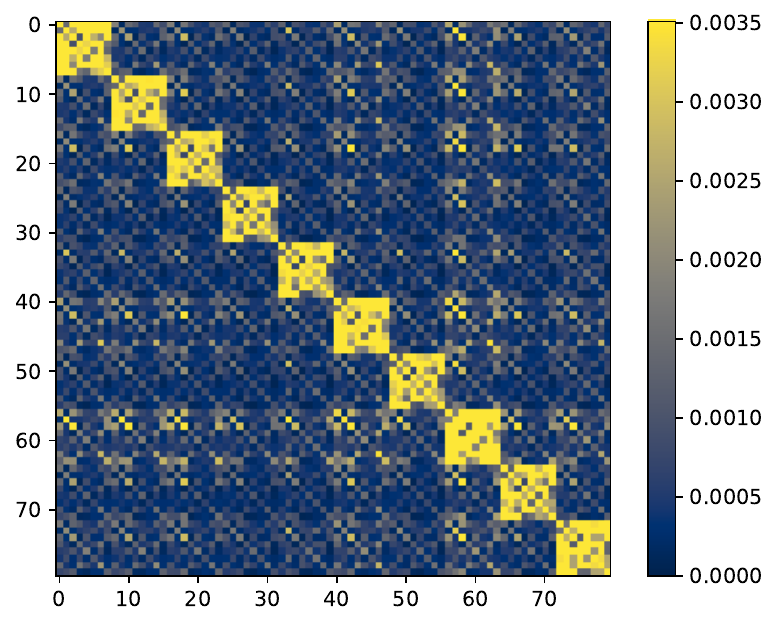}}
    \subfigure[\small $C = 100$]{\includegraphics[width=0.30\textwidth]{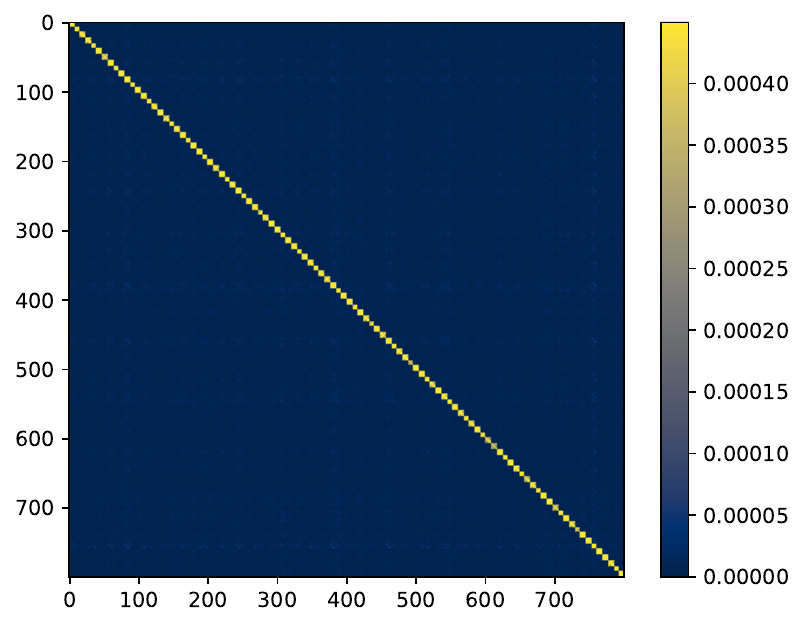}}
    \subfigure[\small $C = 1000$]{\includegraphics[width=0.30\textwidth]{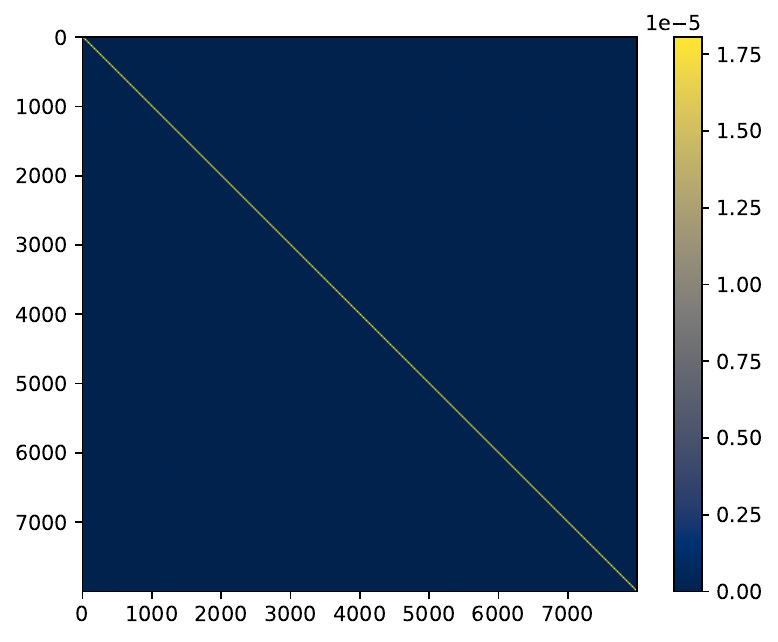}}
    \caption{\small  The Hessian   in {\bf Case 4: 1-hidden-layer network with CE loss}. The network has 8 hidden neurons. {\bf (a, b, c):} The hidden-layer Hessian $H_{ww}$. {\bf (e, f, g):} The output-layer Hessian $H_{vv}$. For both $H_{ww}$ and $H_{vv}$, we observe that the block-diagonal Hessian structure becomes clearer as $C$ increases. }
\label{fig:hessian_nn_ce_different_c}
\vspace{-0.3cm}
\end{figure}

\paragraph{On the Frobenius Norm of Hessian Blocks for Case 2.} We now investigate the following quantities, which appeared in Theorem \ref{thm_linear_model}:
 $${\small H_{11}^{\operatorname{CE}}:= \frac{C^2}{d}\bigg\|\frac{\partial^2\ell_{\text{CE}}(V)}{\partial v_1  \partial v_1^\top}\bigg\|_{\operatorname{F}}^2, \quad H_{12}^{\operatorname{CE}} := \frac{C^4}{d}\bigg\|\frac{\partial^2\ell_{\text{CE}}(V)}{\partial v_1  \partial v_2^\top}\bigg\|_{\operatorname{F}}^2, \quad
r := \bigg\|\frac{\partial^2\ell_{\text{CE}}(V)}{\partial v_1  \partial v_2^\top}\bigg\|_{\operatorname{F}}^2 \bigg/ \bigg\|\frac{\partial^2\ell_{\text{CE}}(V)}{\partial v_1  \partial v_1^\top}\bigg\|_{\operatorname{F}}^2.
}
$$
For each $C$, we simulate 1000  \zhaoruirevise{$H^{\text{CE}}_{11}$} and  $H^{\text{CE}}_{12}$ with LeCun initialization (Assumption \ref{assum_2}), and track their changes with $C$.  The results are shown in Figure \ref{fig:experiment_case2}. We make the following  observations.  These observations match our theoretical prediction.

\begin{figure}[t]
    \centering
    \subfigure[\small $C$ v.s. $H_{11}^{\operatorname{CE}}$]{\includegraphics[width=0.30\textwidth]{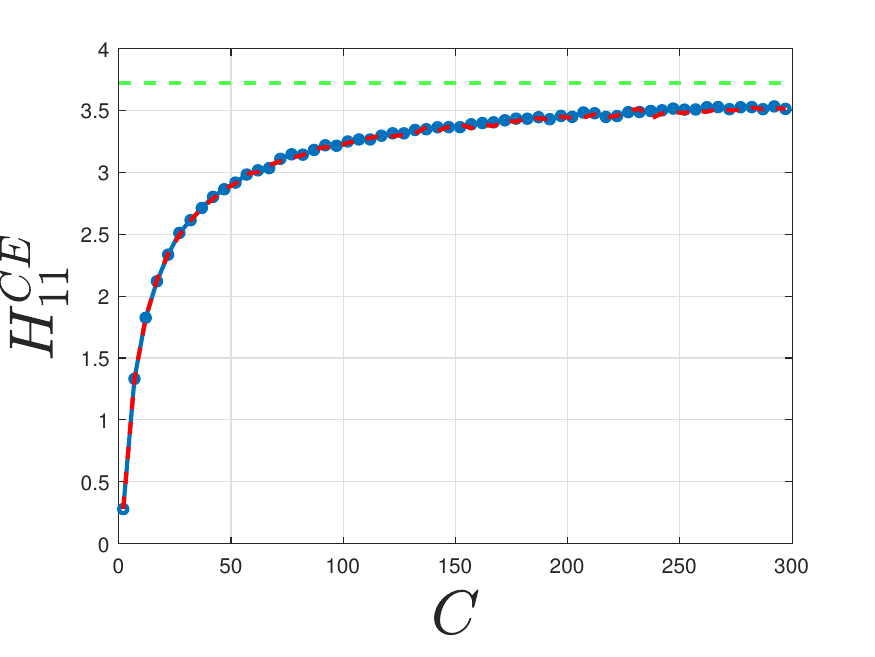}}
    \subfigure[\small $C$ v.s. $H_{12}^{\operatorname{CE}}$]{\includegraphics[width=0.30\textwidth]{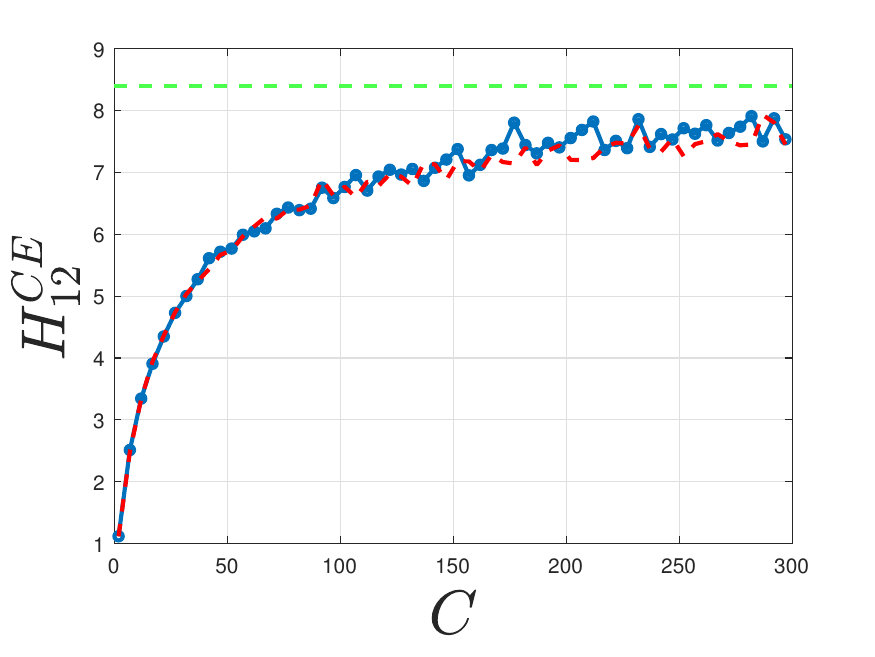}}
    \subfigure[\small $C$ v.s. $r$]{\includegraphics[width=0.30\textwidth]{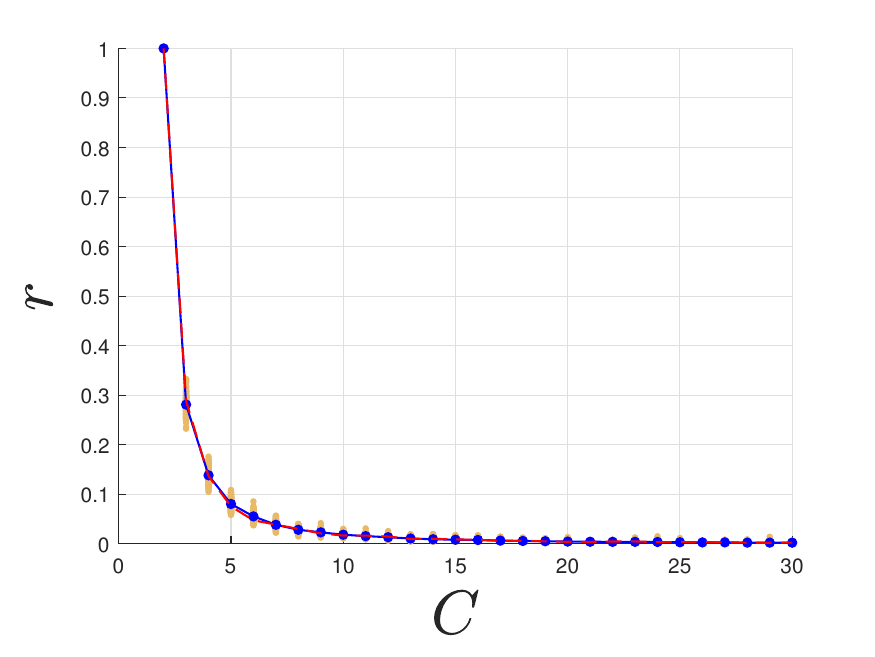}}
    \vspace{-0.2cm}
    \caption{\small  {\bf(a, b, c):} The evolution of $H_{11}^{\operatorname{CE}}$,  $H_{12}^{\operatorname{CE}}$, and $r$ as $C$ increases. For each $C$, the realizations of \zhaoruirevise{$H_{11}^{CE}$} and \zhaoruirevise{$H_{12}^{CE}$} concentrate around the red curves, which are their theoretical \zhaoruirevise{means} in \eqref{eq1thm3} and \eqref{eq2thm3} in Proposition \ref{thm3appendix} (shown later in Section \ref{appendix_thm1}). As $C\rightarrow \infty$,  $H_{11}^{\operatorname{CE}}$ and $H_{12}^{\operatorname{CE}}$  approach the green lines, which are their theoretical limits in Theorem \ref{thm_linear_model}. Further, $r$ vanishes to 0 as $C$ increases, and the decay rate matches Theorem \ref{thm_linear_model}. This means that off-diagonal blocks become relatively negligible as $C$ increases.  
    }
  \label{fig:experiment_case2}
\vspace{-0.3cm}
\end{figure}

\begin{itemize}
[topsep=1pt,parsep=1pt,partopsep=1pt, leftmargin=*]
    \item {\bf First}, for each $C$, the realizations of $H^{\text{CE}}_{11}$ and $H^{\text{CE}}_{12}$ concentrate around the red curves, which are their theoretical \zhaoruirevise{means} in \eqref{eq1thm3} and \eqref{eq2thm3} in Proposition \ref{thm3appendix} (shown later in Section \ref{appendix_thm1}). 
    \item {\bf Second, } as $C\rightarrow \infty$, we find that the  $H_{11}$ and $ H_{12}$  approach the green lines, which are their theoretical limits in Theorem \ref{thm_linear_model}. These results justify the results in Theorem \ref{thm_linear_model}.
    \item  {\bf Third}, as $C\rightarrow \infty$, we have $r \rightarrow 0
 $, and the decay rate matches our theoretical prediction. This means that off-diagonal blocks become relatively negligible as $C$ increases. 
\end{itemize}

\begin{figure}[h]
    \centering
    \subfigure[\small  $C$ v.s. $\tilde{H}_{11}^{\operatorname{CE}}$]{\includegraphics[width=0.30\textwidth]{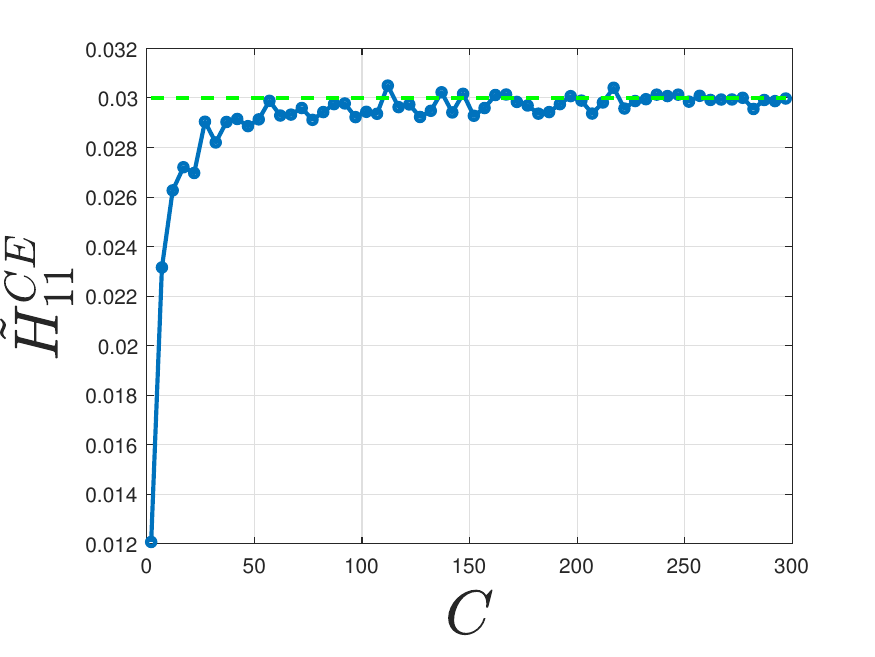}}
    \subfigure[\small  $C$ v.s. $\tilde{H}_{12}^{\operatorname{CE}}$]{\includegraphics[width=0.30\textwidth]{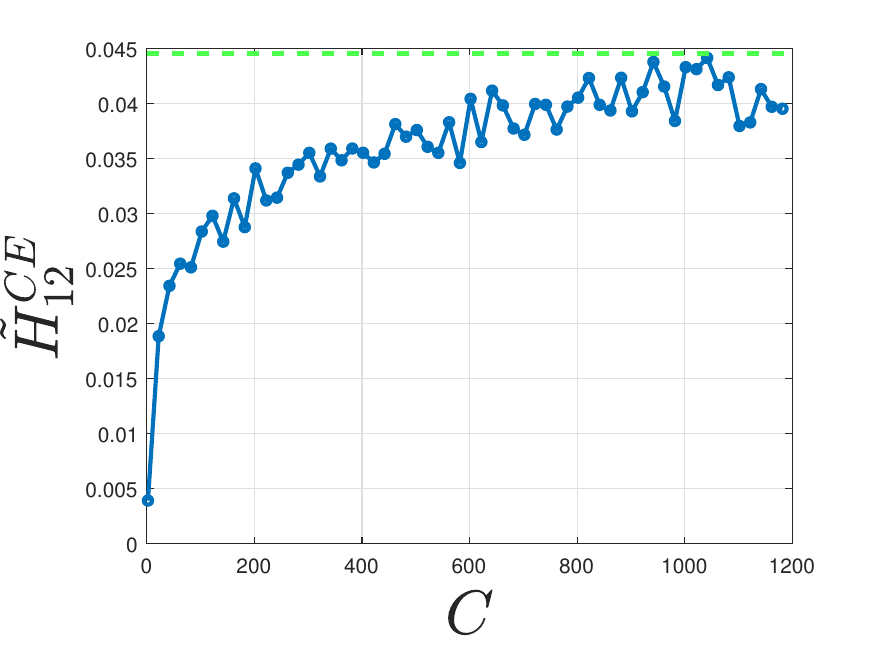}}
    \subfigure[\small  $C$ v.s. $\tilde{r}$]{\includegraphics[width=0.30\textwidth]{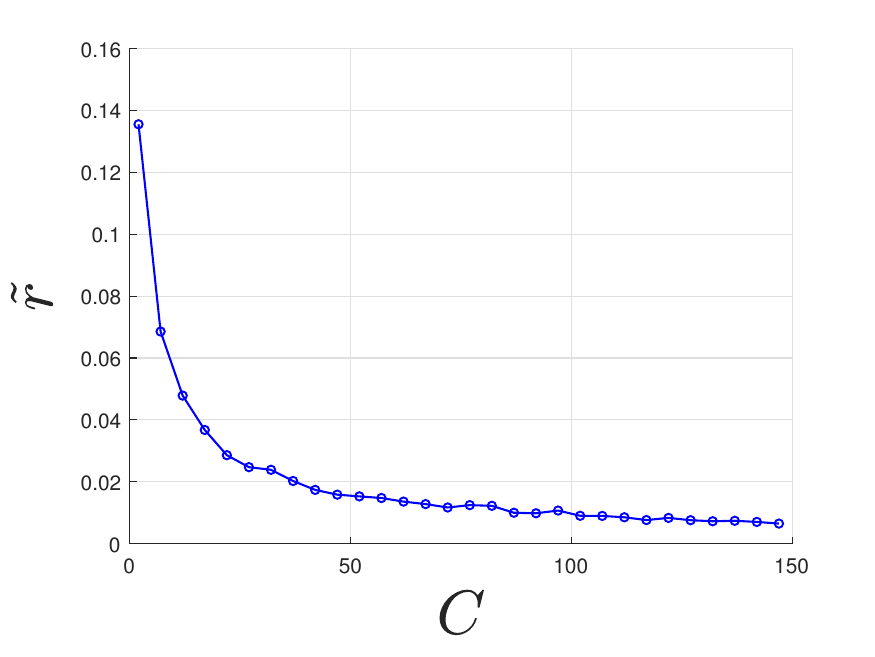}}
    \caption{\small  {\bf(a,b):} The evolution of $\tilde{H}_{11}^{\operatorname{CE}}$ and $\tilde{H}_{12}^{\operatorname{CE}}$ as $C$ increases.  We find that the  $\tilde{H}_{11}^{\operatorname{CE}}$ and $\tilde{H}_{12}^{\operatorname{CE}}$ approach the green lines, which are their theoretical limits in Theorem \ref{thm_nn}. {\bf (c):} $\tilde{r}$ decays to 0 with the same rate as we predicted in Theorem \ref{thm_nn}. }
  \label{fig:f_norm_nn}
\vspace{-0.3cm}
\end{figure}

\paragraph{On the Frobenius Norm of Hessian Blocks for Case 4.}  We now consider 1-hidden-layer networks with CE loss. We introduce the quantities as follows.

$$\tilde{H}_{11}^{\operatorname{CE}}:=\frac{1}{d}\bigg\|\frac{\partial^2\ell_{\text{CE}}(W,V)}{\partial w_1  \partial w_1^\top}\bigg\|_{\operatorname{F}}^2, \quad \tilde{H}_{12}^{\operatorname{CE}} = \frac{C}{d}\bigg\|\frac{\partial^2\ell_{\text{CE}}(W,V)}{\partial w_1  \partial w_2^\top}\bigg\|_{\operatorname{F}}^2, \quad \tilde{r} = \frac{\bigg\|\frac{\partial^2\ell_{\text{CE}}(W,V)}{\partial w_1  \partial w_2^\top}\bigg\|_{\operatorname{F}}^2}{\bigg\|\frac{\partial^2\ell_{\text{CE}}(W,V)}{\partial w_2  \partial w_2^\top}\bigg\|_{\operatorname{F}}^2} $$

 The results are shown in Figure \ref{fig:f_norm_nn}. Similarly as in Figure \ref{fig:experiment_case2},  we find that the  $\tilde{H}_{11}^{\operatorname{CE}}$ and $\tilde{H}_{12}^{\operatorname{CE}}$ approach the green lines, which are their theoretical limits in Theorem \ref{thm_nn}.  Further, $\tilde{r}$ decays to 0 with the same rate as we predicted in Theorem \ref{thm_nn}. These results support the results in Theorem \ref{thm_nn}.

\section{Conclusions}

In this work, we reveal two forces that
shape the near-block-diagonal Hessian structure of NNs: a “static force” rooted in the architecture design, and a “dynamic force” arisen
from training. We then provide a rigorous theoretical analysis of “static force” of linear and 1-hidden-layer NNs at random initialization. It is intriguing to extend our study beyond initialization and simple models. We provide more discussions of future directions in Appendix \ref{appendix_discussion}.

\clearpage
\section*{Acknowledgement}
Z. Dong would like to thank Prof. Zhenyu Liao, Prof. Shurong Zheng and Prof. Jianfeng Yao for organizing a reading group on RMT and Machine Learning in 2023-2024, where Pastur's technique is introduced and discussed. Y. Zhang would like to thank  Weronika Ormaniec, Sidak Pal Singh, Jeremy Cohen, Prof. Yinyu Ye, and Prof. Madeleine Udell for the valuable discussion.  J. Yao’s research is partially supported by the NSFC grant RFIS 12350710179. R. Sun’s research is partially supported by NSFC (No. 12326608); Guangdong Provincial Key Laboratory of Mathematical Foundations for Artificial Intelligence (2023B1212010001).

\appendix

\section*{Table of Contents for the Appendix}

\startcontents[sections]
\printcontents[sections]{l}{1}{\setcounter{tocdepth}{2}}

\section{More Discussions On Our Theory}
\label{appendix_discussion}

For completeness, we provide discuss more discussions on our theory, including some clarifications and future directions.

\begin{itemize}
[topsep=1pt,parsep=1pt,partopsep=1pt, leftmargin=*]
\item {\bf First,} our theory requires $N$ and $d$ proportionally grows to infinity, which is also known as  ``the proportional asymptotic regime''.  We believe this regime is meaningful. First, the proportional asymptotic regime in standard in random matrix theory (e.g., see \citep{wigner1958distribution,tao2012topics}). Second, in the proportional asymptotic regime, we obtained new insights into the Hessian structure (e.g., the effect of large $C$) and our insights matched a wide range of finite-dimensional experiments.  We will try to extend our analysis to other regimes such as non-asymptotic or over-parameterized regime in the future. 
\item {\bf Second,} our theory focuses on random initialization, and it does not cover the whole training process.   Interestingly, we numerically find that the block-diagonal Hessian structure remains throughout the training process (see  Figure \ref{fig:closer_look_mse} and \ref{fig:closer_look_ce}). This suggests that block-diagonal structure continuously influences the behavior of optimizers, {\it not just at initialization}.  It is possible to extend our theory to the whole training process, but it requires substantially new mathematical tools and we leave as a future direction.

\item {\bf Third,} our theory focuses on linear and 1-hidden-layer networks, and it currently does not cover deeper models. We believe our theory on linear and 1-hidden-layer networks is meaningful since it already provides new insights, e.g., the effect of large $C$.   It is possible to extend the results to deeper models by recursively applying our decoupling methods, but substantial effort is needed. For deeper models, we conjecture the block-diagonal structure will be primarily driven by the number of output neurons in each layer (a.k.a., the ``fan-out dimension''). It is also intriguing to explore other potential factors that will reshape the Hessian structure of deep models. We leave it as a future direction. 

\item {\bf Forth,} our theory focuses on block-wise Frobenius norm instead of block-wise spectrum. We focus on on Frobenius norm is it is more relevant to our current goal:  justifying the Hessian structure. One future direction is to theoretically characterize the block-wise spectrum and provide guidance for optimizer design. \citet{zhang2024transformers,zhang2024adam,wang2025sharpness} did some initial attempts in this direction, but these works focused on numerical exploration and did not establish rigorous theory on characterizing block-wise spectrum. Based on our theory so far, it is possible to theoretically analyze the block-wise spectrum by more fine-grained analysis of the Steiltjes transform of the limit eigenvalue distribution, which we leave as a future direction.

\end{itemize}

\clearpage
\section{More Numerical Results and Experimental Details}
\label{appendix_experiments}

Now we present the numerical results. We first re-state the experiments in \citep{collobert2004large} and then present some more of our numerical results.   All experimental details of our results are explained in Appendix \ref{appendix_experimental_details}.

\subsection{Results from \citep{collobert2004large}}
\label{appendix_collobert}

In Figure \ref{fig:collobert}, we restate Figure 7.3 and 7.5 from \citep{collobert2004large}. The authors reported block-diagonal Hessian structure for a 1-hidden-layer network with CE loss on a binary-classification dataset. Such structure disappeared when changing to MSE loss. 

We make two comments here. First, for MSE loss in Figure \ref{fig:collobert} (b), it is not surprising to see non-block-diagonal structure since our theory states that such structure arises when $C \rightarrow \infty$, while $C$ only equals to  2 here. Second, for CE loss in Figure \ref{fig:collobert} (a), the near-block-diagonal structure arises despite the small $C$. This is not covered in our theory  since we focus on large $C$. Nevertheless, it does not contradict our theory, either. 
We leave the exploration of binary classification with CE loss as a future direction.

\begin{figure}[h]
    \centering
    \subfigure[CE loss]{\includegraphics[width=0.40\textwidth]{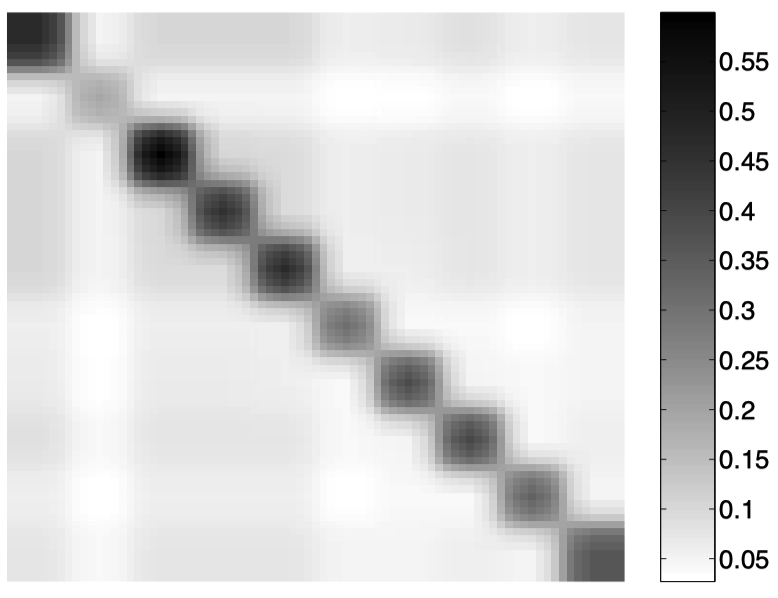}}
    \subfigure[MSE loss]{\includegraphics[width=0.40\textwidth]{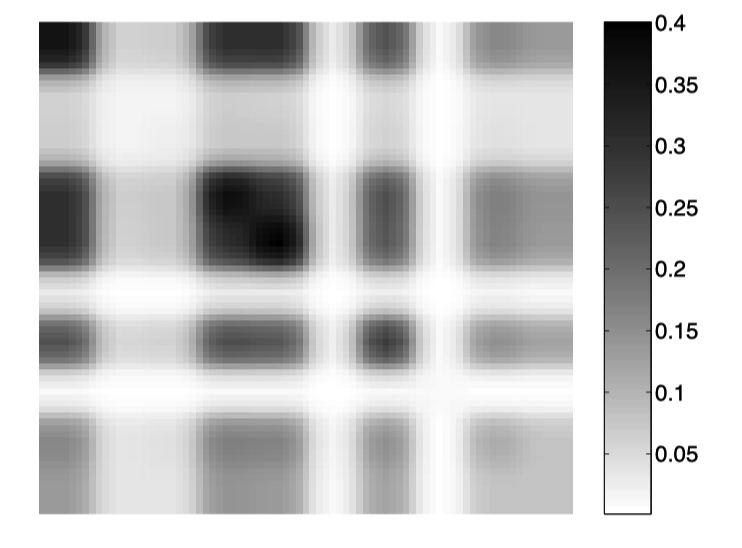}}
    \caption{\small  {\bf(a,b):}  Figure 7.3 and 7.5 from \citep{collobert2004large}. . The Hessian matrices of a 1-hidden-layer network with 10 hidden neurons on the Forest binary-classification dataset. For CE loss, the Hessian is computed after 1 iteration. For MSE loss, the Hessian is computed after 10 iterations. }
  \label{fig:collobert}
\vspace{-0.3cm}
\end{figure}

\subsection{More Ablation Studies}
\label{appendx_ablation}

We now conduct some more ablation studies on other Factors contributing to the Hessian structure. 
\citet{li2020gradient} argue that $K$-feature-clustered dataset will bring $K$-ranked Jacobian. We now investigate how this relates to the block-diagonal Hessian structure.

We construct the $K$-clustered dataset following the descriptions in \citep{li2020gradient}: {\it ``assume that the input $x_n \in \mathbb{R}^d$ come from $K$ clusters which are located on the unit Euclidean ball; assume our dataset consists of $C\leq K$ classes where each class can be composed of multiple clusters.''}. We attach the code for data generation below.

We present the results in Figure \ref{fig:ablation_mse} and
\ref{fig:ablation_ce}.
We report two findings here: 
(1) When \# classes $C = 2$ is small, the Hessian has no block-diagonal structure, regardless of \# clusters $K$. (2) When  $C = 500$ is large, the block-diagonal pattern appears regardless of  \# clusters $K$. This suggests that large $C$ plays a more critical role than $K$ in the Hessian structure.

\begin{figure}[t]
    \centering
    \subfigure[{\small $C = 2, K = 2$}]{\includegraphics[width=0.25\textwidth]{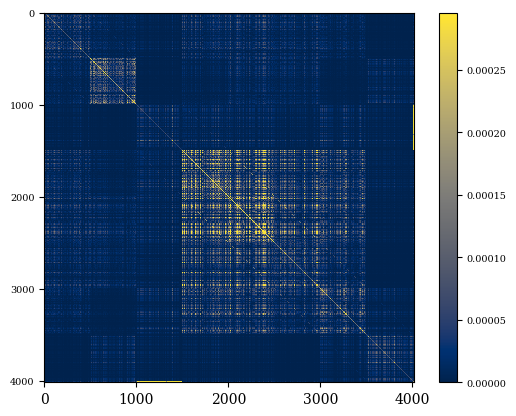}}
    \subfigure[{\small $C = 2, K = 100$}]{\includegraphics[width=0.24\textwidth]{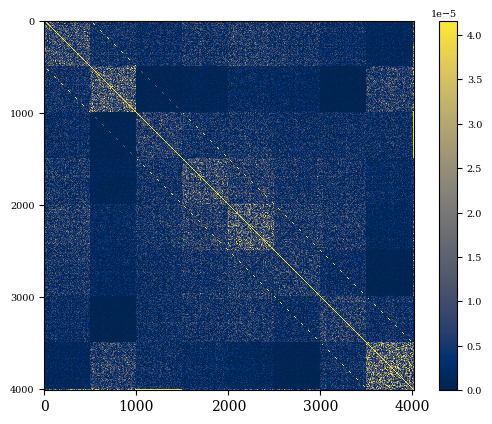}}
    \subfigure[{\small $C = 2, K = 250$}]{\includegraphics[width=0.24\textwidth]{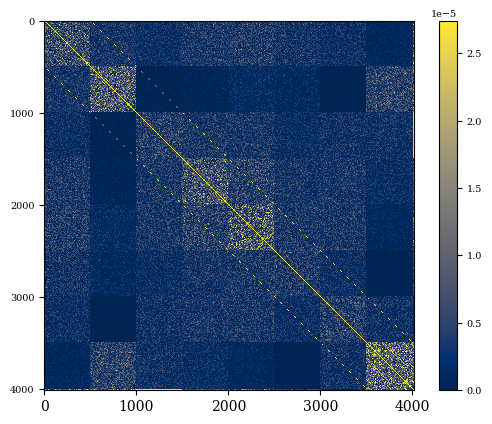}}
    \subfigure[{\small $C = 2, K = 500$}]{\includegraphics[width=0.24\textwidth]{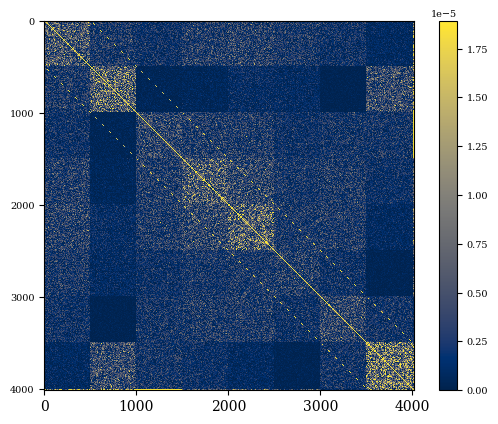}}
    \subfigure[{\small $C = 500, K = 500$}]{\includegraphics[width=0.24\textwidth]{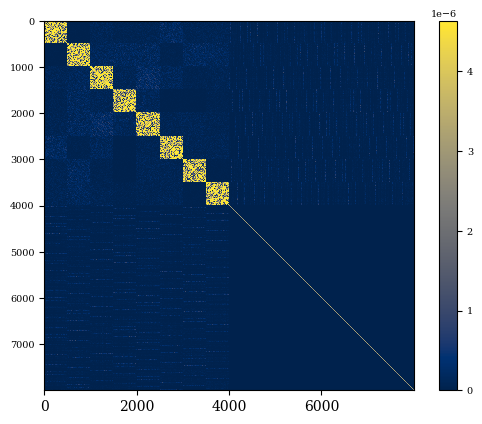}}
    \subfigure[{\small $C = 500, K = 750$}]{\includegraphics[width=0.24\textwidth]{images/0404-cluster-data-n_total-5000-nn-dim-500-width-8-cluster-500-class-500-adam-MSEvisiondegree50T0optimizer-adamlr-0.0001_fullhessian_T_0.png}}
    \subfigure[{\small $C = 500, K = 1000$}]{\includegraphics[width=0.24\textwidth]{images/0404-cluster-data-n_total-5000-nn-dim-500-width-8-cluster-500-class-500-adam-MSEvisiondegree50T0optimizer-adamlr-0.0001_fullhessian_T_0.png}}
    \subfigure[{\small $C = 500, K = 1250$}]{\includegraphics[width=0.24\textwidth]{images/0404-cluster-data-n_total-5000-nn-dim-500-width-8-cluster-500-class-500-adam-MSEvisiondegree50T0optimizer-adamlr-0.0001_fullhessian_T_0.png}}
    \caption{\small  Ablation studies for the effect of \# cluster $K$ on the Hessian structure. We construct $K$-clustered dataset following the setup in [Li et al. 19]: ``assume that the input $x_n \in \mathbb{R}^d$ come from $K$ clusters which are located on the unit Euclidean ball; assume our dataset consists of $C\leq K$ classes where each class can be composed of multiple clusters.'' We use MSE loss and random intialization. We find that: {\bf (a-d):} When $C = 2$ is small, the Hessian has no clear structure, regardless of \# clusters $K$. {\bf (e-h):} When  $C = 500$ is large, the block-diagonal pattern appears regardless of  \# clusters $K$. This suggests that large $C$ plays a more critical role than $K$ in the Hessian structure.  }
  \label{fig:ablation_mse}
\vspace{-0.3cm}
\end{figure}

\begin{figure}[h]
    \centering
    \subfigure[{\small $C = 2, K = 2$}]{\includegraphics[width=0.24\textwidth]{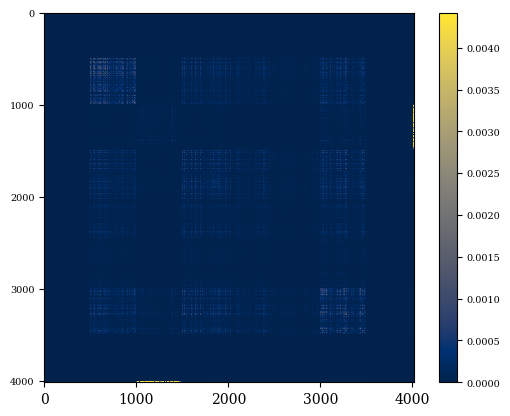}}
    \subfigure[{\small $C = 2, K = 100$}]{\includegraphics[width=0.24\textwidth]{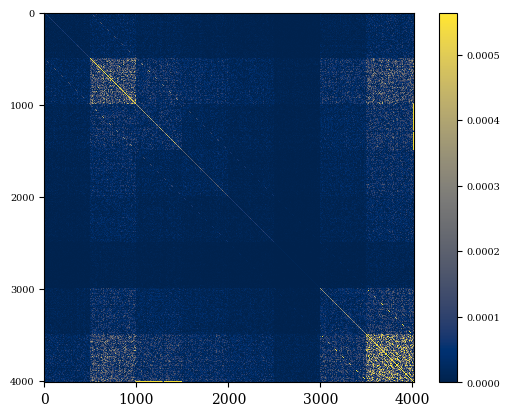}}
    \subfigure[{\small $C = 2, K = 250$}]{\includegraphics[width=0.24\textwidth]{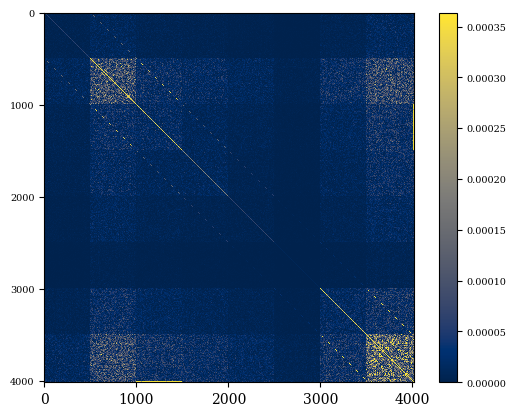}}
    \subfigure[{\small $C = 2, K = 500$}]{\includegraphics[width=0.24\textwidth]{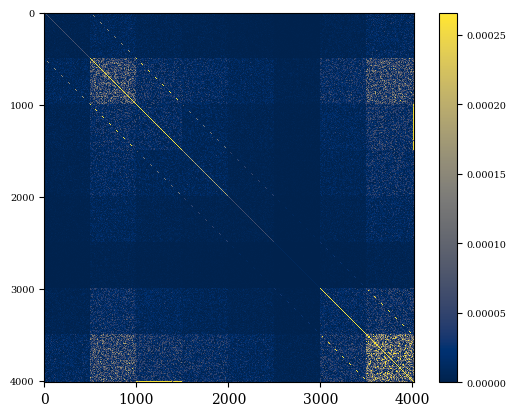}}
    \subfigure[{\small $C = 500, K = 500$}]{\includegraphics[width=0.24\textwidth]{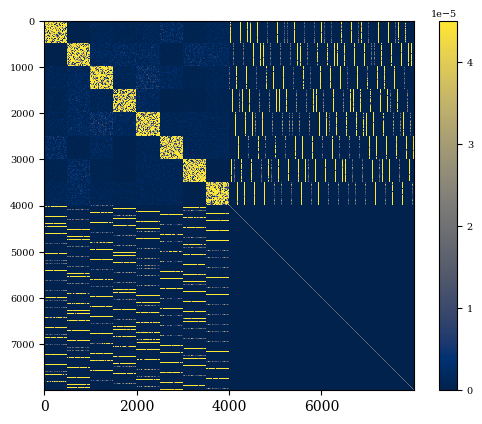}}
    \subfigure[{\small $C = 500, K = 750$}]{\includegraphics[width=0.24\textwidth]{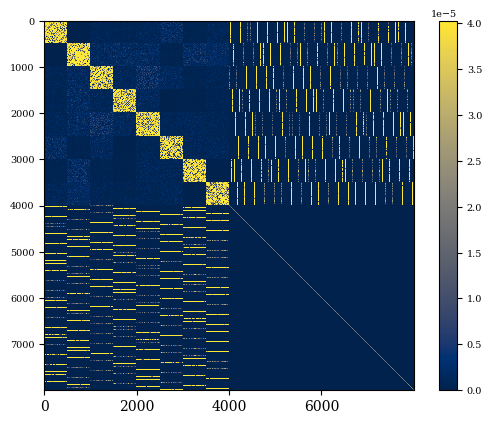}}
    \subfigure[{\small $C = 500, K = 1000$}]{\includegraphics[width=0.24\textwidth]{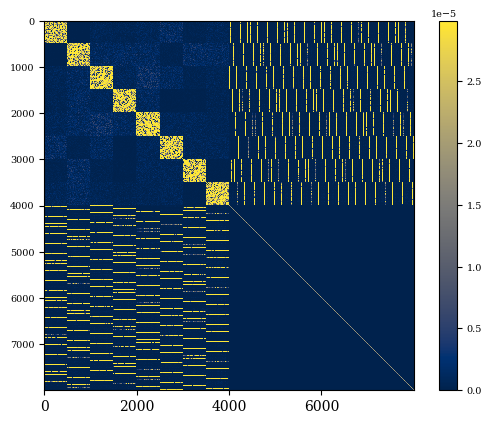}}
    \subfigure[{\small $C = 500, K = 1250$}]{\includegraphics[width=0.24\textwidth]{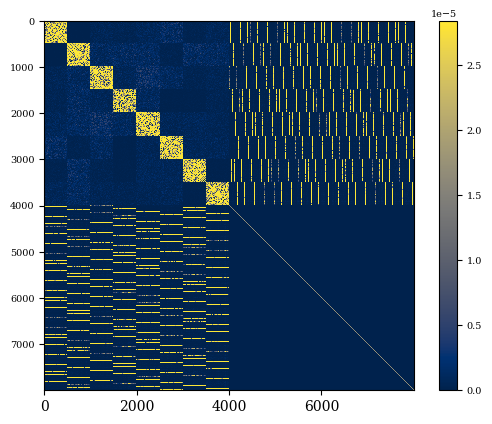}}
    \caption{\small Same ablation studies as in Figure \ref{fig:ablation_mse} except that we change to CE loss.  We find that: {\bf (a-d):} When $C = 2$ is small, the Hessian has no clear structure, regardless of \# clusters $K$. {\bf (e-h):} When  $C = 500$ is large, the block-diagonal pattern appears in $H_{ww}$ and $H_{vv}$ regardless of  \# clusters $K$. This suggests that large $C$ plays a more critical role than $K$ in the Hessian structure.  }
  \label{fig:ablation_ce}
\vspace{-0.2cm}
\end{figure}

\begin{python}[language=Python, caption=]

def generate_cluster_data(n_total, n_classes, n_clusters, input_dim):
    # Generate clustered synthetic data for specified dimensions
    # used for ablation study
    # n_total is the total number of samples
    # n_classes is the number of classes (smaller than n_clusters)
    # input_dim is the dimension of the data
    # raise error if n_cluster is larger than n_classes
    assert n_classes<= n_clusters, f"n_cluster = {n_classes} is not smaller than n_classes = {n_classes}"
    
    # n_samples_per_class is the number of samples per class
    X = []
    y = []
    n_cluster_per_class = n_clusters // n_classes
    n_samples_per_cluster = n_total // n_clusters
    cluster_idx = 0
    for class_idx in range(n_classes):

        for _ in range(n_cluster_per_class):
            cluster_idx += 1
            if input_dim == 2:
                center = np.array([np.cos(2 * np.pi * cluster_idx / n_clusters), np.sin(2 * np.pi * (cluster_idx) / n_clusters)]) * 5  # Class centers on a circle
            else:
                #extend the 2D case to higher dimension
                # Generate random points in higher dimensions and project onto hypersphere
                center = np.random.randn(input_dim)
                # Normalize to create a unit vector (point on unit hypersphere)
                center = center / np.linalg.norm(center)

            cluster_samples = np.random.randn(n_samples_per_cluster, input_dim) * 0.05 + center  # Add some noise
            X.append(cluster_samples)
            # assign label
            y.extend([class_idx]*n_samples_per_cluster)

    X = np.vstack(X)  # Combine all class samples
    y = np.array(y)    # Convert labels to a NumPy array
    return X, y

\end{python}

\subsection{Experimental Details}
\label{appendix_experimental_details}

Now we present the experimental details. 
All experiments are conducted on one NVIDIA V100 GPU.

\paragraph{Implementation details  for Figure \ref{fig:hessian}.} We calculate Hessian on a randomly selected 128 images from CIFAR-100. We  calculate Hessian  via two backpropagation passes \citep{pearlmutter1994fast}, our code is modified based on open-source Hessian-vector-product implementation \footnote{\url{https://github.com/zyushun/hessian-spectrum}}.  We consider a 1-hidden-layer network with ReLU activation, 8 hidden neurons, and 100 output neurons at random initialization. For all Hessian matrices reported in the paper, we report the absolute value of each Hessian entry.

\begin{figure}[h]
    \centering
                \subfigure[Loss curve under MSE loss]{\includegraphics[width=0.45\textwidth]{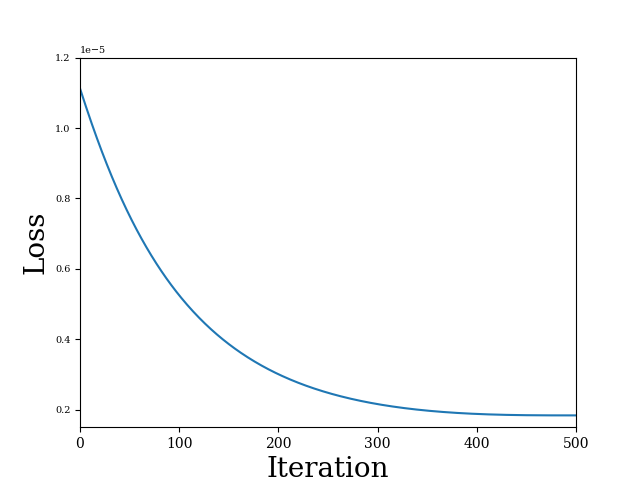}}
            \subfigure[Loss curve under CE loss]{\includegraphics[width=0.45\textwidth]{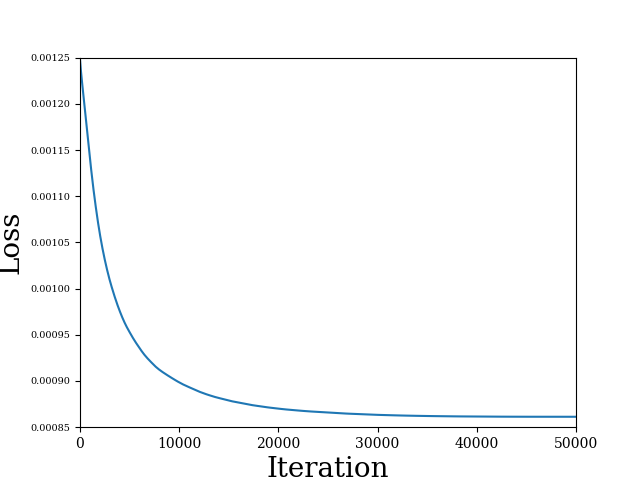}}
    \caption{\small The loss curve of the models trained in Section \ref{sec_closer_look}.  We train the models until convergence. }
  \label{fig:loss_curve_gaussian}
\vspace{-0.3cm}
\end{figure}

\paragraph{Implementation details for Section \ref{sec_closer_look} .} We first introduce the implementation details for the synthetic dataset used in both Section \ref{sec_closer_look} and   Appendix \ref{appendix_experiments}. We build the dataset following Assumption \ref{assum_1} and we assign the label randomly.  We attach the code for data generation here. In Section  \ref{sec_closer_look}, we use  \texttt{input\_dim} = 500,  \texttt{n\_classes} = 500, \texttt{n\_samples\_per\_class} = 10.

\begin{python}[language=Python, caption=]
def generate_gaussian_data(n_samples_per_class, n_classes, input_dim):
    X = []
    y = []
    for i in range(n_classes):
        class_samples = np.random.randn(n_samples_per_class, input_dim) 
        X.append(class_samples)
        y.extend([i] * n_samples_per_class) # not used

    X = np.vstack(X) 
    y = np.array(y)   
    return X, y

\end{python}
Now we  describe the model configurations in Section \ref{sec_closer_look}. We use 1-hidden-layer networks with  8 hidden neurons and ReLU activation. All the models are trained using Adam with \texttt{lr} = 1e-4 with cosine annealing schedule with  \texttt{lr\_min} =  0. We train the models until convergence. The loss curves are reported in Figure \ref{fig:loss_curve_gaussian}.

\paragraph{Implementation details for Section \ref{sec_experiments} and Appendix \ref{appendix_experiments}.} For these experiments, we use the same setups as in Section \ref{sec_closer_look}. 
We use \texttt{input\_dim} = 64, \texttt{n\_samples\_per\_class} = 1, and  the total number of samples $N = C$.

\paragraph{Implementation details for Figure \ref{fig:experiment_case2} and Figure \ref{fig:f_norm_nn}.} For Figure \ref{fig:experiment_case2}, we consider $N = 1000, d= 1000$. For each $C$,   we randomly sample 1000 realization of $H_{11}$, $H_{12}$ and report their Frobenius norms.  We use $10^6$ repetitions in Monte-Carlo integrals.
For Figure \ref{fig:f_norm_nn}, we use $N=300,\ d=300,$ and 200 repetitions for each $C$.

\clearpage

\bibliographystyle{abbrvnat}
\bibliography{reference.bib}

\begin{thebibliography}{86}
\providecommand{\natexlab}[1]{#1}
\providecommand{\url}[1]{\texttt{#1}}
\expandafter\ifx\csname urlstyle\endcsname\relax
  \providecommand{\doi}[1]{doi: #1}\else
  \providecommand{\doi}{doi: \begingroup \urlstyle{rm}\Url}\fi

\bibitem[Alain et~al.(2019)Alain, Roux, and Manzagol]{alain2019negative}
G.~Alain, N.~L. Roux, and P.-A. Manzagol.
\newblock Negative eigenvalues of the hessian in deep neural networks.
\newblock \emph{arXiv preprint arXiv:1902.02366}, 2019.

\bibitem[An et~al.(2025)An, Liu, Pan, Ma, Goldfarb, and Zhang]{an2025asgo}
K.~An, Y.~Liu, R.~Pan, S.~Ma, D.~Goldfarb, and T.~Zhang.
\newblock Asgo: Adaptive structured gradient optimization.
\newblock \emph{arXiv preprint arXiv:2503.20762}, 2025.

\bibitem[Arora et~al.(2022)Arora, Li, and Panigrahi]{arora2022understanding}
S.~Arora, Z.~Li, and A.~Panigrahi.
\newblock Understanding gradient descent on the edge of stability in deep learning.
\newblock In \emph{International Conference on Machine Learning}, pages 948--1024. PMLR, 2022.

\bibitem[Bai and Silverstein(2010)]{Baibook}
Z.~Bai and J.~W. Silverstein.
\newblock \emph{Spectral Analysis of Large Dimensional Random Matrices}.
\newblock Springer New York, 2010.

\bibitem[Bai and Zhou(2008)]{BaiZhou}
Z.~Bai and W.~Zhou.
\newblock Large sample covariance matrices without independence structures in columns.
\newblock \emph{Statistica Sinica}, 18\penalty0 (2):\penalty0 425--442, 2008.

\bibitem[Chatterjee(2006)]{Chatterjee2006}
S.~Chatterjee.
\newblock A generalization of the lindeberg principle.
\newblock \emph{Annals of Probability}, 34\penalty0 (6):\penalty0 2061 – 2076, 2006.

\bibitem[Chaudhari et~al.(2019)Chaudhari, Choromanska, Soatto, LeCun, Baldassi, Borgs, Chayes, Sagun, and Zecchina]{chaudhari2019entropy}
P.~Chaudhari, A.~Choromanska, S.~Soatto, Y.~LeCun, C.~Baldassi, C.~Borgs, J.~Chayes, L.~Sagun, and R.~Zecchina.
\newblock Entropy-sgd: Biasing gradient descent into wide valleys.
\newblock \emph{Journal of Statistical Mechanics: Theory and Experiment}, 2019\penalty0 (12):\penalty0 124018, 2019.

\bibitem[Cohen et~al.(2021)Cohen, Kaur, Li, Kolter, and Talwalkar]{cohen2021gradient}
J.~M. Cohen, S.~Kaur, Y.~Li, J.~Z. Kolter, and A.~Talwalkar.
\newblock Gradient descent on neural networks typically occurs at the edge of stability.
\newblock \emph{arXiv preprint arXiv:2103.00065}, 2021.

\bibitem[Cohen et~al.(2022)Cohen, Ghorbani, Krishnan, Agarwal, Medapati, Badura, Suo, Cardoze, Nado, Dahl, et~al.]{cohen2022adaptive}
J.~M. Cohen, B.~Ghorbani, S.~Krishnan, N.~Agarwal, S.~Medapati, M.~Badura, D.~Suo, D.~Cardoze, Z.~Nado, G.~E. Dahl, et~al.
\newblock Adaptive gradient methods at the edge of stability.
\newblock \emph{arXiv preprint arXiv:2207.14484}, 2022.

\bibitem[Collins and Hayase(2023)]{CollinsFreeness}
B.~Collins and T.~Hayase.
\newblock Asymptotic freeness of layerwise jacobians caused by invariance of multilayer perceptron: The haar orthogonal case.
\newblock \emph{Communications in Mathematical Physics}, 397\penalty0 (1):\penalty0 85 – 109, 2023.

\bibitem[Collobert(2004)]{collobert2004large}
R.~Collobert.
\newblock Large scale machine learning.
\newblock Technical report, Universit{\'e} de Paris VI, 2004.

\bibitem[Dangel et~al.(2020)Dangel, Harmeling, and Hennig]{dangel2020modular}
F.~Dangel, S.~Harmeling, and P.~Hennig.
\newblock Modular block-diagonal curvature approximations for feedforward architectures.
\newblock In \emph{International Conference on Artificial Intelligence and Statistics}, pages 799--808. PMLR, 2020.

\bibitem[Das et~al.(2024)Das, Agarwal, Sanghavi, and Dhillon]{das2024towards}
R.~Das, N.~Agarwal, S.~Sanghavi, and I.~S. Dhillon.
\newblock Towards quantifying the preconditioning effect of adam.
\newblock \emph{arXiv preprint arXiv:2402.07114}, 2024.

\bibitem[Dauphin et~al.(2014)Dauphin, Pascanu, Gulcehre, Cho, Ganguli, and Bengio]{dauphin2014identifying}
Y.~N. Dauphin, R.~Pascanu, C.~Gulcehre, K.~Cho, S.~Ganguli, and Y.~Bengio.
\newblock Identifying and attacking the saddle point problem in high-dimensional non-convex optimization.
\newblock \emph{Advances in neural information processing systems}, 27, 2014.

\bibitem[Desjardins et~al.(2015)Desjardins, Simonyan, Pascanu, et~al.]{desjardins2015natural}
G.~Desjardins, K.~Simonyan, R.~Pascanu, et~al.
\newblock Natural neural networks.
\newblock \emph{Advances in neural information processing systems}, 28, 2015.

\bibitem[Draxler et~al.(2018)Draxler, Veschgini, Salmhofer, and Hamprecht]{draxler2018essentially}
F.~Draxler, K.~Veschgini, M.~Salmhofer, and F.~Hamprecht.
\newblock Essentially no barriers in neural network energy landscape.
\newblock In \emph{International conference on machine learning}, pages 1309--1318. PMLR, 2018.

\bibitem[Fort and Ganguli(2019)]{fort2019emergent}
S.~Fort and S.~Ganguli.
\newblock Emergent properties of the local geometry of neural loss landscapes.
\newblock \emph{arXiv preprint arXiv:1910.05929}, 2019.

\bibitem[George et~al.(2018)George, Laurent, Bouthillier, Ballas, and Vincent]{george2018fast}
T.~George, C.~Laurent, X.~Bouthillier, N.~Ballas, and P.~Vincent.
\newblock Fast approximate natural gradient descent in a kronecker factored eigenbasis.
\newblock \emph{Advances in neural information processing systems}, 31, 2018.

\bibitem[Ghorbani et~al.(2019)Ghorbani, Krishnan, and Xiao]{ghorbani2019investigation}
B.~Ghorbani, S.~Krishnan, and Y.~Xiao.
\newblock An investigation into neural net optimization via hessian eigenvalue density.
\newblock In \emph{International Conference on Machine Learning}, pages 2232--2241. PMLR, 2019.

\bibitem[Goldfarb et~al.(2020)Goldfarb, Ren, and Bahamou]{goldfarb2020practical}
D.~Goldfarb, Y.~Ren, and A.~Bahamou.
\newblock Practical quasi-newton methods for training deep neural networks.
\newblock \emph{Advances in Neural Information Processing Systems}, 33:\penalty0 2386--2396, 2020.

\bibitem[G{\"o}tze et~al.(2015)G{\"o}tze, K{\"o}sters, and Tikhomirov]{gotze2015asymptotic}
F.~G{\"o}tze, H.~K{\"o}sters, and A.~Tikhomirov.
\newblock Asymptotic spectra of matrix-valued functions of independent random matrices and free probability.
\newblock \emph{Random Matrices: Theory and Applications}, 4\penalty0 (02):\penalty0 1550005, 2015.

\bibitem[Granziol et~al.(2019)Granziol, Garipov, Vetrov, Zohren, Roberts, and Wilson]{granziol2019towards}
D.~Granziol, T.~Garipov, D.~Vetrov, S.~Zohren, S.~Roberts, and A.~G. Wilson.
\newblock Towards understanding the true loss surface of deep neural networks using random matrix theory and iterative spectral methods.
\newblock 2019.

\bibitem[Granziol et~al.(2022)Granziol, Zohren, and Roberts]{granziol2022learning}
D.~Granziol, S.~Zohren, and S.~Roberts.
\newblock Learning rates as a function of batch size: A random matrix theory approach to neural network training.
\newblock \emph{Journal of Machine Learning Research}, 23\penalty0 (173):\penalty0 1--65, 2022.

\bibitem[Gupta et~al.(2018)Gupta, Koren, and Singer]{gupta2018shampoo}
V.~Gupta, T.~Koren, and Y.~Singer.
\newblock Shampoo: Preconditioned stochastic tensor optimization.
\newblock In \emph{International Conference on Machine Learning}, pages 1842--1850. PMLR, 2018.

\bibitem[Gur-Ari et~al.(2018)Gur-Ari, Roberts, and Dyer]{gur2018gradient}
G.~Gur-Ari, D.~A. Roberts, and E.~Dyer.
\newblock Gradient descent happens in a tiny subspace.
\newblock \emph{arXiv preprint arXiv:1812.04754}, 2018.

\bibitem[Hanin and Nica(2020)]{hanin2020products}
B.~Hanin and M.~Nica.
\newblock Products of many large random matrices and gradients in deep neural networks.
\newblock \emph{Communications in Mathematical Physics}, 376\penalty0 (1):\penalty0 287--322, 2020.

\bibitem[He et~al.(2019)He, Huang, and Yuan]{he2019asymmetric}
H.~He, G.~Huang, and Y.~Yuan.
\newblock Asymmetric valleys: Beyond sharp and flat local minima.
\newblock \emph{Advances in neural information processing systems}, 32, 2019.

\bibitem[Jastrzebski et~al.(2020)Jastrzebski, Szymczak, Fort, Arpit, Tabor, Cho, and Geras]{jastrzebski2020break}
S.~Jastrzebski, M.~Szymczak, S.~Fort, D.~Arpit, J.~Tabor, K.~Cho, and K.~Geras.
\newblock The break-even point on optimization trajectories of deep neural networks.
\newblock \emph{arXiv preprint arXiv:2002.09572}, 2020.

\bibitem[Jastrzkebski et~al.(2018)Jastrzkebski, Kenton, Ballas, Fischer, Bengio, and Storkey]{jastrzkebski2018relation}
S.~Jastrzkebski, Z.~Kenton, N.~Ballas, A.~Fischer, Y.~Bengio, and A.~Storkey.
\newblock On the relation between the sharpest directions of dnn loss and the sgd step length.
\newblock \emph{arXiv preprint arXiv:1807.05031}, 2018.

\bibitem[Jiang et~al.(2019)Jiang, Neyshabur, Mobahi, Krishnan, and Bengio]{jiang2019fantastic}
Y.~Jiang, B.~Neyshabur, H.~Mobahi, D.~Krishnan, and S.~Bengio.
\newblock Fantastic generalization measures and where to find them.
\newblock \emph{arXiv preprint arXiv:1912.02178}, 2019.

\bibitem[Jordan et~al.(2024)Jordan, Jin, Boza, You, Cesista, Newhouse, and Bernstein]{jordan2024muon}
K.~Jordan, Y.~Jin, V.~Boza, J.~You, F.~Cesista, L.~Newhouse, and J.~Bernstein.
\newblock Muon: An optimizer for hidden layers in neural networks, 2024.
\newblock URL \url{https://kellerjordan.github.io/posts/muon/}.

\bibitem[Keskar et~al.(2016)Keskar, Mudigere, Nocedal, Smelyanskiy, and Tang]{keskar2016large}
N.~S. Keskar, D.~Mudigere, J.~Nocedal, M.~Smelyanskiy, and P.~T.~P. Tang.
\newblock On large-batch training for deep learning: Generalization gap and sharp minima.
\newblock \emph{arXiv preprint arXiv:1609.04836}, 2016.

\bibitem[Kingma and Ba(2014)]{kingma2014adam}
D.~P. Kingma and J.~Ba.
\newblock Adam: A method for stochastic optimization.
\newblock \emph{arXiv preprint arXiv:1412.6980}, 2014.

\bibitem[Kunstner et~al.(2024)Kunstner, Yadav, Milligan, Schmidt, and Bietti]{kunstner2024heavy}
F.~Kunstner, R.~Yadav, A.~Milligan, M.~Schmidt, and A.~Bietti.
\newblock Heavy-tailed class imbalance and why adam outperforms gradient descent on language models.
\newblock \emph{arXiv preprint arXiv:2402.19449}, 2024.

\bibitem[Lavezzi et~al.(2022)Lavezzi, Guye, and Ciarci{\`a}]{lavezzi2022nonlinear}
G.~Lavezzi, K.~Guye, and M.~Ciarci{\`a}.
\newblock Nonlinear programming solvers for unconstrained and constrained optimization problems: a benchmark analysis.
\newblock \emph{arXiv preprint arXiv:2204.05297}, 2022.

\bibitem[LeCun et~al.(2002)LeCun, Bottou, Orr, and M{\"u}ller]{lecun2002efficient}
Y.~LeCun, L.~Bottou, G.~B. Orr, and K.-R. M{\"u}ller.
\newblock Efficient backprop.
\newblock In \emph{Neural networks: Tricks of the trade}, pages 9--50. Springer, 2002.

\bibitem[Li et~al.(2020{\natexlab{a}})Li, Soltanolkotabi, and Oymak]{li2020gradient}
M.~Li, M.~Soltanolkotabi, and S.~Oymak.
\newblock Gradient descent with early stopping is provably robust to label noise for overparameterized neural networks.
\newblock In \emph{International conference on artificial intelligence and statistics}, pages 4313--4324. PMLR, 2020{\natexlab{a}}.

\bibitem[Li et~al.(2020{\natexlab{b}})Li, Gu, Zhou, Chen, and Banerjee]{li2020hessian}
X.~Li, Q.~Gu, Y.~Zhou, T.~Chen, and A.~Banerjee.
\newblock Hessian based analysis of sgd for deep nets: Dynamics and generalization.
\newblock In \emph{Proceedings of the 2020 SIAM International Conference on Data Mining}, pages 190--198. SIAM, 2020{\natexlab{b}}.

\bibitem[Liao and Mahoney(2021)]{liao2021hessian}
Z.~Liao and M.~W. Mahoney.
\newblock Hessian eigenspectra of more realistic nonlinear models.
\newblock \emph{Advances in Neural Information Processing Systems}, 34:\penalty0 20104--20117, 2021.

\bibitem[Lindeberg(1922)]{Lindeberg1922}
J.~Lindeberg.
\newblock Eine neue herleitung des exponentialgesetzes in der wahrscheinlichkeitsrechnung.
\newblock \emph{Mathematische Zeitschrift}, 15\penalty0 (1):\penalty0 211 – 225, 1922.

\bibitem[Liu et~al.(2024)Liu, Feng, Xue, Wang, Wu, Lu, Zhao, Deng, Zhang, Ruan, et~al.]{liu2024deepseek}
A.~Liu, B.~Feng, B.~Xue, B.~Wang, B.~Wu, C.~Lu, C.~Zhao, C.~Deng, C.~Zhang, C.~Ruan, et~al.
\newblock Deepseek-v3 technical report.
\newblock \emph{arXiv preprint arXiv:2412.19437}, 2024.

\bibitem[Liu et~al.(2023)Liu, Li, Hall, Liang, and Ma]{liu2023sophia}
H.~Liu, Z.~Li, D.~Hall, P.~Liang, and T.~Ma.
\newblock Sophia: A scalable stochastic second-order optimizer for language model pre-training.
\newblock \emph{arXiv preprint arXiv:2305.14342}, 2023.

\bibitem[Liu et~al.(2025)Liu, Su, Yao, Jiang, Lai, Du, Qin, Xu, Lu, Yan, et~al.]{liu2025muon}
J.~Liu, J.~Su, X.~Yao, Z.~Jiang, G.~Lai, Y.~Du, Y.~Qin, W.~Xu, E.~Lu, J.~Yan, et~al.
\newblock Muon is scalable for llm training.
\newblock \emph{arXiv preprint arXiv:2502.16982}, 2025.

\bibitem[Lyu et~al.(2022)Lyu, Li, and Arora]{lyu2022understanding}
K.~Lyu, Z.~Li, and S.~Arora.
\newblock Understanding the generalization benefit of normalization layers: Sharpness reduction.
\newblock \emph{Advances in Neural Information Processing Systems}, 35:\penalty0 34689--34708, 2022.

\bibitem[Maes et~al.(2024)Maes, Zhang, Jolicoeur-Martineau, Mitliagkas, Scieur, Lacoste-Julien, and Guille-Escuret]{maes2024understanding}
L.~Maes, T.~H. Zhang, A.~Jolicoeur-Martineau, I.~Mitliagkas, D.~Scieur, S.~Lacoste-Julien, and C.~Guille-Escuret.
\newblock Understanding adam requires better rotation dependent assumptions.
\newblock \emph{arXiv preprint arXiv:2410.19964}, 2024.

\bibitem[Malinovskii et~al.(2024)Malinovskii, Panferov, Ilin, Guo, Richt{\'a}rik, and Alistarh]{malinovskii2024pushing}
V.~Malinovskii, A.~Panferov, I.~Ilin, H.~Guo, P.~Richt{\'a}rik, and D.~Alistarh.
\newblock Pushing the limits of large language model quantization via the linearity theorem.
\newblock \emph{arXiv preprint arXiv:2411.17525}, 2024.

\bibitem[Martens and Grosse(2015)]{martens2015optimizing}
J.~Martens and R.~Grosse.
\newblock Optimizing neural networks with kronecker-factored approximate curvature.
\newblock In \emph{International conference on machine learning}, pages 2408--2417. PMLR, 2015.

\bibitem[Marčenko and Pastur(1967)]{Pastur67}
V.~A. Marčenko and L.~A. Pastur.
\newblock Distribution of eigenvalues for some set of random matrices.
\newblock \emph{Mathematics of the USSR-Sbornik}, 1\penalty0 (4):\penalty0 457, 1967.

\bibitem[Mingo and Speicher(2017)]{Mingobook}
J.~Mingo and R.~Speicher.
\newblock \emph{Free Probability and Random Matrices}.
\newblock Springer, 2017.

\bibitem[{Moonshot AI}(2025)]{moonshot2025k2}
{Moonshot AI}.
\newblock Kimi k2: Open agentic intelligence, July 2025.
\newblock URL \url{https://moonshotai.github.io/Kimi-K2/}.

\bibitem[Ormaniec et~al.(2024)Ormaniec, Dangel, and Singh]{ormaniec2024does}
W.~Ormaniec, F.~Dangel, and S.~P. Singh.
\newblock What does it mean to be a transformer? insights from a theoretical hessian analysis.
\newblock \emph{arXiv preprint arXiv:2410.10986}, 2024.

\bibitem[Papyan(2018)]{papyan2018full}
V.~Papyan.
\newblock The full spectrum of deepnet hessians at scale: Dynamics with sgd training and sample size.
\newblock \emph{arXiv preprint arXiv:1811.07062}, 2018.

\bibitem[Papyan(2019)]{papyan2019measurements}
V.~Papyan.
\newblock Measurements of three-level hierarchical structure in the outliers in the spectrum of deepnet hessians.
\newblock \emph{arXiv preprint arXiv:1901.08244}, 2019.

\bibitem[Papyan(2020)]{papyan2020traces}
V.~Papyan.
\newblock Traces of class/cross-class structure pervade deep learning spectra.
\newblock \emph{The Journal of Machine Learning Research}, 21\penalty0 (1):\penalty0 10197--10260, 2020.

\bibitem[Park and Kim(2022)]{park2022vision}
N.~Park and S.~Kim.
\newblock How do vision transformers work?
\newblock \emph{arXiv preprint arXiv:2202.06709}, 2022.

\bibitem[Pastur(2020)]{pastur2020random}
L.~Pastur.
\newblock On random matrices arising in deep neural networks: Gaussian case.
\newblock \emph{Pure and Applied Functional Analysis}, 5\penalty0 (6):\penalty0 1395 – 1424, 2020.

\bibitem[Pastur(2022)]{pastur2022haarorthogonal}
L.~Pastur.
\newblock Eigenvalue distribution of large random matrices arising in deep neural networks: Orthogonal case.
\newblock \emph{Journal of Mathematical Physics}, 63\penalty0 (6), 2022.

\bibitem[Pastur and Slavin(2023)]{pastur2023iid}
L.~Pastur and V.~Slavin.
\newblock On random matrices arising in deep neural networks: General i.i.d. case.
\newblock \emph{Random Matrices: Theory and Application}, 12\penalty0 (1), 2023.

\bibitem[Pearlmutter(1994)]{pearlmutter1994fast}
B.~A. Pearlmutter.
\newblock Fast exact multiplication by the hessian.
\newblock \emph{Neural computation}, 6\penalty0 (1):\penalty0 147--160, 1994.

\bibitem[Pennington and Bahri(2017)]{pennington2017geometry}
J.~Pennington and Y.~Bahri.
\newblock Geometry of neural network loss surfaces via random matrix theory.
\newblock In \emph{International conference on machine learning}, pages 2798--2806. PMLR, 2017.

\bibitem[Qu et~al.(2022)Qu, Gao, Hinder, Ye, and Zhou]{qu2022optimal}
Z.~Qu, W.~Gao, O.~Hinder, Y.~Ye, and Z.~Zhou.
\newblock Optimal diagonal preconditioning: Theory and practice.
\newblock \emph{arXiv preprint arXiv:2209.00809}, 2022.

\bibitem[Roux et~al.(2007)Roux, Manzagol, and Bengio]{roux2007topmoumoute}
N.~Roux, P.-A. Manzagol, and Y.~Bengio.
\newblock Topmoumoute online natural gradient algorithm.
\newblock \emph{Advances in neural information processing systems}, 20, 2007.

\bibitem[Sagun et~al.(2016)Sagun, Bottou, and LeCun]{sagun2016eigenvalues}
L.~Sagun, L.~Bottou, and Y.~LeCun.
\newblock Eigenvalues of the hessian in deep learning: Singularity and beyond.
\newblock \emph{arXiv preprint arXiv:1611.07476}, 2016.

\bibitem[Sagun et~al.(2017)Sagun, Evci, Guney, Dauphin, and Bottou]{sagun2017empirical}
L.~Sagun, U.~Evci, V.~U. Guney, Y.~Dauphin, and L.~Bottou.
\newblock Empirical analysis of the hessian of over-parametrized neural networks.
\newblock \emph{arXiv preprint arXiv:1706.04454}, 2017.

\bibitem[Sankar et~al.(2021)Sankar, Khasbage, Vigneswaran, and Balasubramanian]{sankar2021deeper}
A.~R. Sankar, Y.~Khasbage, R.~Vigneswaran, and V.~N. Balasubramanian.
\newblock A deeper look at the hessian eigenspectrum of deep neural networks and its applications to regularization.
\newblock In \emph{Proceedings of the AAAI Conference on Artificial Intelligence}, volume~35, pages 9481--9488, 2021.

\bibitem[Singh et~al.(2021)Singh, Bachmann, and Hofmann]{singh2021analytic}
S.~P. Singh, G.~Bachmann, and T.~Hofmann.
\newblock Analytic insights into structure and rank of neural network hessian maps.
\newblock \emph{Advances in Neural Information Processing Systems}, 34:\penalty0 23914--23927, 2021.

\bibitem[Sun(2019)]{sun2019optimization}
R.~Sun.
\newblock Optimization for deep learning: theory and algorithms.
\newblock \emph{arXiv preprint arXiv:1912.08957}, 2019.

\bibitem[Sun and Ye(2021)]{sun2021worst}
R.~Sun and Y.~Ye.
\newblock Worst-case complexity of cyclic coordinate descent: O (n\^{} 2) o (n 2) gap with randomized version.
\newblock \emph{Mathematical Programming}, 185:\penalty0 487--520, 2021.

\bibitem[Talagrand(2003)]{Talagrandbook}
M.~Talagrand.
\newblock \emph{Spin Glasses: A Challenge for Mathematicians. Cavity and Mean Field Models}.
\newblock Springer, 2003.

\bibitem[Tao(2012)]{tao2012topics}
T.~Tao.
\newblock \emph{Topics in random matrix theory}, volume 132.
\newblock American Mathematical Soc., 2012.

\bibitem[Touvron et~al.(2023)Touvron, Martin, Stone, Albert, Almahairi, Babaei, Bashlykov, Batra, Bhargava, Bhosale, et~al.]{touvron2023llama}
H.~Touvron, L.~Martin, K.~Stone, P.~Albert, A.~Almahairi, Y.~Babaei, N.~Bashlykov, S.~Batra, P.~Bhargava, S.~Bhosale, et~al.
\newblock Llama 2: Open foundation and fine-tuned chat models.
\newblock \emph{arXiv preprint arXiv:2307.09288}, 2023.

\bibitem[Vyas et~al.(2024)Vyas, Morwani, Zhao, Kwun, Shapira, Brandfonbrener, Janson, and Kakade]{vyas2024soap}
N.~Vyas, D.~Morwani, R.~Zhao, M.~Kwun, I.~Shapira, D.~Brandfonbrener, L.~Janson, and S.~Kakade.
\newblock Soap: Improving and stabilizing shampoo using adam.
\newblock \emph{arXiv preprint arXiv:2409.11321}, 2024.

\bibitem[Wang et~al.(2025)Wang, Wang, Zhou, Yan, Wu, et~al.]{wang2025sharpness}
J.~Wang, M.~Wang, Z.~Zhou, J.~Yan, L.~Wu, et~al.
\newblock The sharpness disparity principle in transformers for accelerating language model pre-training.
\newblock \emph{arXiv preprint arXiv:2502.19002}, 2025.

\bibitem[Wang et~al.(2022)Wang, Li, and Li]{wang2022analyzing}
Z.~Wang, Z.~Li, and J.~Li.
\newblock Analyzing sharpness along gd trajectory: Progressive sharpening and edge of stability.
\newblock \emph{Advances in Neural Information Processing Systems}, 35:\penalty0 9983--9994, 2022.

\bibitem[Wei and Schwab(2019)]{wei2019noise}
M.~Wei and D.~J. Schwab.
\newblock How noise affects the hessian spectrum in overparameterized neural networks.
\newblock \emph{arXiv preprint arXiv:1910.00195}, 2019.

\bibitem[Wigner(1958)]{wigner1958distribution}
E.~P. Wigner.
\newblock On the distribution of the roots of certain symmetric matrices.
\newblock \emph{Annals of Mathematics}, 67\penalty0 (2):\penalty0 325--327, 1958.

\bibitem[Wu et~al.(2017)Wu, Zhu, et~al.]{wu2017towards}
L.~Wu, Z.~Zhu, et~al.
\newblock Towards understanding generalization of deep learning: Perspective of loss landscapes.
\newblock \emph{arXiv preprint arXiv:1706.10239}, 2017.

\bibitem[Wu et~al.(2020)Wu, Zhu, Wu, Wang, and Ge]{wu2020dissecting}
Y.~Wu, X.~Zhu, C.~Wu, A.~Wang, and R.~Ge.
\newblock Dissecting hessian: Understanding common structure of hessian in neural networks.
\newblock \emph{arXiv preprint arXiv:2010.04261}, 2020.

\bibitem[Yao et~al.(2015)Yao, Zheng, and Bai]{YaoBook}
J.~Yao, S.~Zheng, and Z.~Bai.
\newblock \emph{Large Sample Covariance Matrices and High-Dimensional Data Analysis}.
\newblock Cambridge University Press, 2015.

\bibitem[Yao et~al.(2018)Yao, Gholami, Lei, Keutzer, and Mahoney]{yao2018hessian}
Z.~Yao, A.~Gholami, Q.~Lei, K.~Keutzer, and M.~W. Mahoney.
\newblock Hessian-based analysis of large batch training and robustness to adversaries.
\newblock \emph{Advances in Neural Information Processing Systems}, 31, 2018.

\bibitem[Yao et~al.(2020)Yao, Gholami, Keutzer, and Mahoney]{yao2020pyhessian}
Z.~Yao, A.~Gholami, K.~Keutzer, and M.~W. Mahoney.
\newblock Pyhessian: Neural networks through the lens of the hessian.
\newblock In \emph{2020 IEEE international conference on big data (Big data)}, pages 581--590. IEEE, 2020.

\bibitem[Zeng et~al.(2025)Zeng, Lv, Zheng, Hou, Chen, Xie, Wang, Yin, Zeng, Zhang, et~al.]{zeng2025glm}
A.~Zeng, X.~Lv, Q.~Zheng, Z.~Hou, B.~Chen, C.~Xie, C.~Wang, D.~Yin, H.~Zeng, J.~Zhang, et~al.
\newblock Glm-4.5: Agentic, reasoning, and coding (arc) foundation models.
\newblock \emph{arXiv preprint arXiv:2508.06471}, 2025.

\bibitem[Zhang et~al.(2019)Zhang, Li, Nado, Martens, Sachdeva, Dahl, Shallue, and Grosse]{zhang2019algorithmic}
G.~Zhang, L.~Li, Z.~Nado, J.~Martens, S.~Sachdeva, G.~Dahl, C.~Shallue, and R.~B. Grosse.
\newblock Which algorithmic choices matter at which batch sizes? insights from a noisy quadratic model.
\newblock \emph{Advances in neural information processing systems}, 32, 2019.

\bibitem[Zhang et~al.(2017)Zhang, Xiong, Bradbury, and Socher]{zhang2017block}
H.~Zhang, C.~Xiong, J.~Bradbury, and R.~Socher.
\newblock Block-diagonal hessian-free optimization for training neural networks.
\newblock \emph{arXiv preprint arXiv:1712.07296}, 2017.

\bibitem[Zhang et~al.(2024{\natexlab{a}})Zhang, Chen, Ding, Li, Sun, and Luo]{zhang2024transformers}
Y.~Zhang, C.~Chen, T.~Ding, Z.~Li, R.~Sun, and Z.-Q. Luo.
\newblock Why transformers need adam: A hessian perspective.
\newblock \emph{arXiv preprint arXiv:2402.16788}, 2024{\natexlab{a}}.

\bibitem[Zhang et~al.(2024{\natexlab{b}})Zhang, Chen, Li, Ding, Wu, Kingma, Ye, Luo, and Sun]{zhang2024adam}
Y.~Zhang, C.~Chen, Z.~Li, T.~Ding, C.~Wu, D.~P. Kingma, Y.~Ye, Z.-Q. Luo, and R.~Sun.
\newblock Adam-mini: Use fewer learning rates to gain more.
\newblock \emph{arXiv preprint arXiv:2406.16793}, 2024{\natexlab{b}}.

\end{thebibliography}

\end{document}